\documentclass[11pt]{article}
\usepackage[utf8]{inputenc}
\usepackage[T1]{fontenc}
\usepackage{mlmodern}
\usepackage[dvipsnames]{xcolor}
\usepackage[colorlinks=true, linkcolor=Mahogany, citecolor=ForestGreen]{hyperref}
\usepackage{url}
\usepackage{booktabs}
\usepackage{amsfonts}
\usepackage{amsmath}
\usepackage{amssymb}
\usepackage{amsthm}
\usepackage{nicefrac}
\usepackage{microtype}
\usepackage{newtxtext}

\usepackage[square, sort&compress, numbers]{natbib}
\usepackage{parskip}
\usepackage[margin=1in]{geometry}
\usepackage{changepage}

\usepackage{todonotes}
\usepackage{dsfont}
\usepackage{marginnote}
\usepackage{tikz, pgf}
\usepackage{mathtools}
\newcommand*{\coloneq}{\vcentcolon\mathrel{\mkern-1.2mu}=} 
\usepackage{algorithm}
\usepackage{algpseudocode}
\usepackage{enumitem}
\usepackage{tcolorbox}
\usepackage{mathrsfs}
\usetikzlibrary{positioning}

\DeclarePairedDelimiter\ceil{\lceil}{\rceil}

\newtheorem{theorem}{Theorem}
\newtheorem{proposition}{Proposition}
\newtheorem{lemma}[proposition]{Lemma}
\newtheorem{corollary}{Corollary}
\newtheorem{definition}{Definition}

\newtheorem{remark}{Remark}

\definecolor{forestgreen}{RGB}{34,139,34}

\newcommand{\E}[2][]{\mathbb{E}_{#1}\left[#2\right]}

\newcommand{\norm}[1]{\left\lVert #1 \right\rVert}
\newcommand{\abs}[1]{\left| #1 \right|}
\newcommand{\opnorm}[1]{\left\lVert #1 \right\rVert_{\mathrm{op}}}

\newcommand{\Pbb}{\mathbb{P}}
\newcommand{\Ebb}{\mathbb{E}}
\newcommand{\Rbb}{\mathbb{R}}
\newcommand{\Nbb}{\mathbb{N}}
\newcommand{\Vcal}{\mathcal{V}}
\newcommand{\Ecal}{\mathcal{E}}

\newcommand{\Mcal}{\mathcal{M}}
\newcommand{\Gcal}{\mathcal{G}}
\newcommand{\Acal}{\mathcal{A}}

\newcommand{\diag}{\operatorname{Diag}}

\newcommand{\Var}{\operatorname{Var}}

\newcommand{\tssym}{\Psi}
\newcommand{\ts}{\tssym_{i,A}(j)}
\newcommand{\tsat}[3]{\tssym_{#1,#2}(#3)}
\newcommand{\undeftok}{\textsc{undef}}

\newcommand{\dunes}{\textsc{DuNeS}}
\newcommand{\algoone}{\textsc{LS-DuNeS}}
\newcommand{\algotwo}{\textsc{AU-DuNeS}}

\title{{Mixing-Free and Signal-Optimal Learning of Gaussian Graphical Models from Glauber Dynamics}}

\author{
{
\begin{tabular}{c}
Vignesh Tirukkonda \qquad Gautam Dasarathy \\
Arizona State University \\
{\footnotesize\texttt{\{vtirukko, gautamd\}@asu.edu}}
\end{tabular}
}
}

\date{July 19, 2026}

\begin{document}
\maketitle

\begin{abstract}
Gaussian graphical model selection is usually studied under independent sampling, but in many applications the data arise as a single trajectory of a dependent stochastic process. We study exact recovery of the graph from one trajectory of random-scan Gaussian Glauber dynamics. Existing techniques for this problem either inherit the mixing time of the chain, which can be super-polynomial in the dimension $p$ without strong assumptions, or are suboptimal in the minimum normalized edge strength $\kappa$. We propose two algorithms that are mixing-free and attain the $\kappa^{-2}$ dependence of the information-theoretic lower bounds. Both instantiate a shared dueling-neighborhood search meta-algorithm with a local statistic built directly from the update sequence. For every fixed precision matrix and deterministic initialization, the first algorithm fits a least-squares regression at the updates of each node and has pointwise recovery horizon $\widetilde O(pd^{2}/\kappa^{2})$, where $d$ is the maximum degree. Its horizon depends logarithmically on a local conditioning quantity and on the initialization potential. The second algorithm is based on counting occurences of a specific update pattern and requires $\widetilde O(pd^{4}/\kappa^{2})$ updates, with no dependence on any condition number. The central technical challenge is that both statistics are built from dependent, non-stationary observations. Our analysis tackles this by demonstrating how to extract fresh Gaussian innovations from the update sequence, which yields mixing-free control of appropriate quantities. Neither the algorithms nor their analyses invoke stationarity, a spectral gap, or mixing conditions.
\end{abstract}

\section{Introduction}\label{sec:introduction}

The Gaussian graphical model (GGM) is an important and widely used formalism for representing conditional dependence structure in high-dimensional Gaussian data. It has found applications in diverse areas such as bioinformatics, finance, and psychology~\citep{edwards2010forests, zhan2021graphical, epskamp2016gaussian}. A central statistical challenge associated with these models is \emph{structure learning}: recovering the conditional-independence graph, or equivalently the sparsity pattern of the precision matrix of the Gaussian distribution. This problem of learning the graph structure from node data has been studied extensively when the data is sampled independently (and identically) from the joint distribution; see, for example,~\citet{drton2016structure}.

In this paper we study a variant of the structure learning problem, where the data we observe comes from a single trajectory of the Gaussian Glauber dynamics also known as the random-scan Gaussian Gibbs sampler on the underlying distribution. That is, each time step supplies a local conditional update to a randomly chosen node (or coordinate). Indeed, since all other nodes remain the same, these observations are highly dependent. {This setting is motivated by the many practical scenarios in which i.i.d.\ samples are unavailable and the data are more naturally modeled as a stochastic process evolving over time~\citep{Dennis2022Markov, Lara2019GlauberEpidemicModel}. Glauber dynamics in particular has been used to model the spread of information and diseases across networks~\citep{Montanari2010InformationSpread, Lara2019GlauberEpidemicModel} and agent decisions in coordination games~\citep{Auletta2017}.} {A recent line of work has begun to address this problem}; see Section~\ref{sec:related} for more on this.

In the classical i.i.d. setting, the information theoretically optimal~\citep{wang2010information} sample complexity $\mathcal{O}(\log p/\kappa^{2})$ for structure learning has been shown to be achievable~\citep{misra2020information}; here $p$ is the number of vertices in the graph (or the ambient dimension) and $\kappa$ is the normalized minimum edge strength (defined in Section~\ref{sec:ggm}). Indeed, lower bounds with similar scaling hold even in the Gaussian Glauber dynamics setting~\citep{TRD25, SWMM26}. This raises the question -- is this sample complexity achievable by an algorithm for structure learning from a single Glauber trajectory?

\citet{TRD25} propose two procedures that make partial progress on this question. The first, BTR-GL, is mixing-based: it waits for the trajectory to mix and then reduces it to near-independent samples. It attains the corresponding dependence on the ``signal strength'' $\kappa$, but its sample complexity depends on the mixing time of the underlying Markov chain, and keeping that mixing time small requires strong assumptions on the dynamics, such as a Dobrushin-type contraction condition and bounded diagonal entries of the precision matrix. The second procedure of~\citet{TRD25}, LET-GL, is local and does not explicitly invoke mixing. Its guarantee, however, rests on similar Dobrushin-type assumptions, under which the chain will in fact have mixed by the time structure learning succeeds; moreover, its sample complexity scales as $\kappa^{-5}$. \citet{SWMM26} also propose two procedures. Their mixing-based procedure comes with a guarantee scaling as $\kappa^{-2}\left(t_{\mathrm{mix}}+d^{2}\kappa^{-2}\right)$, where $t_{\mathrm{mix}}$ is the mixing time, and so retains a $\kappa^{-4}$ term regardless of how fast the chain mixes. {Since their result depends on the mixing time, keeping the sample complexity competitive would require the strong assumptions discussed above}. Their local procedure is mixing-free and its guarantee does not require Dobrushin-type conditions. Its sample complexity, however, scales as $\kappa^{-5}$ like LET-GL. In short, no existing procedure is simultaneously free of mixing assumptions and optimal in its dependence on $\kappa$.

In this work, we propose two algorithms, \algoone{} and \algotwo{}, that (a) {\bf do not require fast mixing} and (b) retain the {\bf information-theoretically optimal dependence on the signal strength} $\kappa$\footnote{Throughout, optimality in $\kappa$ refers to the polynomial rate: our guarantees scale as $\kappa^{-2}$, matching the lower bounds discussed in Section~\ref{sec:ggm}. The logarithmic factors in our guarantees contain an additive $\log(1/\kappa)$ that the lower bounds do not have; this term is dominated by $\log p$ whenever $1/\kappa$ is polynomial in $p$. See Section~\ref{sec:discussion}.}. Our starting point is a meta-algorithm, directly inspired by~\cite{misra2020information}, which we call \dunes{} (for \underline{Du}eling \underline{Ne}ighborhood \underline{S}earch).  For a fixed node $i$, \dunes{} compares a candidate set $C$ with dueling sets $D$ and invites the elements of the latter to challenge the candidacy of $C$ via a carefully designed statistic. If any of the challengers succeed, $C$ is discarded. And, if there is no dueling set $D$ that is able to successfully challenge $C$, then $C$ is declared to be a successful candidate that contains the true neighborhood of $i$. A simple post-processing of $C$ then reveals the true neighborhood.  

The first instantiation, \algoone{}, performs a local regression for each node to create this test statistic. Although the covariate rows in these regressions are highly dependent, we know that each row has a fresh Gaussian innovation, and our analysis uses this structure carefully. The resulting procedure has information-theoretically optimal dependence on the minimum edge strength $\kappa$. For each fixed precision matrix $\Theta$ and deterministic initialization $x^{(0)}$, its pointwise recovery horizon depends only logarithmically on the \emph{degree-constrained spectral radius} $\rho_{2d}(\Theta)$ (Definition~\ref{def:rho2d}) and on the initialization potential $\norm{x^{(0)}}_\Theta^2$. We show in Proposition~\ref{prop:large-rho-slow-relaxation} that large values of $\rho_{2d}(\Theta)$ force a slow uniform mixing time; \algoone{}, however, can still be efficient.

The second instantiation, \algotwo{}, is based on looking for alternating update patterns of the form $i,i,j,i$ (as in~\cite{TRD25, SWMM26}); see Section~\ref{sec:algotwo} for the intuition behind this. Earlier work using pattern-based statistics, including work on Ising models~\citep{bresler2014learning, gaitonde2025bypassing, gaitonde2025better}, face a basic difficulty: the pattern is informative only when the neighborhood of $i$ is quiet, i.e,  none of the neighbors of node $i$ are updated during the pattern updates. However, the neighborhood is unknown, which means this event is undetectable. To handle this difficulty, these methods divide the trajectory into fixed-length windows and count the number of windows where this pattern occurs. They then estimate how often neighbor updates could contaminate the test statistic and apply a corresponding correction. This creates a tension: the larger the window is, the more likely one is to observe these $iiji$ patterns, but also, the more likely it is for contaminations to occur. This effectively severely dilutes the signal and leads to poor dependence on the minimum edge strength $\kappa${; the resulting sample complexities scale as $\kappa^{-5}$}.

The candidate-set strategy of \dunes{} removes this observability problem. Since we iterate through explicitly declared candidate neighborhoods, we no longer have to work with fixed length windows or correct the test statistic and can therefore retain the informative blocks without losing signal to contamination. Our local test statistic construction, based on the normalized residual sum of squares, avoids a further source of signal loss: the statistics in past work require individual updates to have sufficiently large magnitude and, in some cases, all samples to be simultaneously bounded above. Being scale-invariant, ours requires neither, which permits a sharper analysis through self-normalized martingale bounds. The resulting algorithm is information-theoretically optimal in $\kappa$ while not requiring any extra assumptions such as  bounded condition number, irrepresentability conditions, or Dobrushin-type assumptions.

We summarize the main contributions of this paper as follows.
\begin{itemize}
    \item We introduce the \dunes{} meta algorithm, a natural abstraction of DICE~\cite{misra2020information}, that is guaranteed to recover the exact graph provided its local test statistic satisfies a uniform separation property.
    \item We develop \algoone{}, which instantiates the template with a regression-based statistic. For each node, it reads the trajectory at the updates of that node, fits an ordinary least-squares regression on the tested coordinates, and uses appropriately normalized regression coefficients to perform the test.
    \item We develop \algotwo{}, which instantiates the template with a block statistic built from the alternating update pattern $i,i,j,i$ restricted to the candidate and dueling sets. Since the sets are explicit, the test statistic does not have to correct for possible contamination from neighbor updates required by earlier update-pattern-based methods.
    \item We prove finite-sample recovery guarantees for both procedures. Writing $\widetilde O$ for a bound that suppresses factors logarithmic in $p$, $d$, $1/\kappa$, $1/\delta$, $\rho_{2d}(\Theta)$, and the initialization potential:
    \begin{itemize}
        \item \algoone{} has pointwise recovery horizon
        $N=\widetilde O\!\left(p\,d^{2}/\kappa^{2}\right)$ for every
        $(\Theta,x^{(0)})$. Its guarantee depends on
        $\rho_{2d}(\Theta)$ and
        $\norm{x^{(0)}}_\Theta^2$ only logarithmically: the horizon is
        essentially linear in $p$ when $\rho_{2d}(\Theta)$ is
        polynomial, and remains polynomial even when
        $\rho_{2d}(\Theta)$ grows exponentially in $p$, a regime in
        which mixing-based horizons are exponential.
        \item \algotwo{} recovers the graph from 
        $N=\widetilde O\!\left(p\,d^{4}/\kappa^{2}\right)$ updates,
        uniformly over every precision matrix in the model
        class and every deterministic initialization, with no
        dependence on $\rho_{2d}(\Theta)$, the initialization
        potential, or any other condition number.
    \end{itemize}
    Both guarantees are mixing-free and attain the $\kappa^{-2}$ dependence of the information-theoretic lower bound. The analyses behind these guarantees share one driving principle which directly builds on the structure of the Glauber dynamics: whenever a tested coordinate has just been refreshed, the corresponding observation decomposes into a predictable part and a fresh Gaussian innovation. \algoone{} finds this fresh randomness in the regression rows that immediately follow an update of a tested coordinate, and \algotwo{} in the refresh of the probed coordinate within each occurrence of the pattern. The remaining technical difficulties are specific to each statistic. For \algoone{}, we must control the spectrum of a regression design built from dependent, non-stationary observations. We use a Lyapunov-style drift inequality to bound the spectrum from above, and the fresh innovations to bound it from below. For \algotwo{}, the observations drawn from successive occurrences of the pattern are themselves dependent. A key difficulty here is to construct an appropriate filtration under which we can still invoke self-normalized concentration, and we control the resulting unbounded statistic by a peeling argument.
\end{itemize}

\paragraph{Paper organization.}
Section~\ref{sec:related} reviews related work on structure learning. Section~\ref{sec:setup} introduces the GGM, the Glauber dynamics, and the recovery problem. Section~\ref{sec:skeleton} presents the \dunes{} template and its deterministic recovery guarantee. Section~\ref{sec:algoone} develops the regression-based algorithm \algoone{}. Section~\ref{sec:algotwo} develops the alternating-update based algorithm~\algotwo{}. Section~\ref{sec:discussion} compares the two procedures. The proofs are collected in the appendices.

\section{Related Work}\label{sec:related}

\paragraph{Structure learning from i.i.d.\ Gaussian samples.}
Structure learning in GGMs from independent samples has been studied extensively. Neighborhood-regression methods recover the graph one node at a time by sparse linear regression of each coordinate on the rest~\citep{meinshausen2006high}, while penalized maximum-likelihood methods such as the graphical lasso fit a sparse precision matrix jointly~\citep{friedman2008sparse, Yuan2007Gaussian, banerjee2007model}. Guarantees for both families typically require incoherence or irrepresentability conditions on $\Theta$~\citep{ravikumar2011high}, conditions that the information-theoretic treatment of the problem avoids. \citet{wang2010information} established lower bounds together with matching combinatorial estimators, and \citet{misra2020information} gave DICE, an iterative support-testing estimator whose sample complexity matches the information-theoretic lower bound up to a factor of the maximum degree $d$, under a non-degeneracy condition alone. The dueling-sets skeleton of \dunes{} inherits the candidate/cleanup architecture of DICE; the substantive difference is that our local statistics are built from Glauber update events rather than from independent samples.

\paragraph{GGM structure learning from dynamics.}
For discrete graphical models, \citet{bresler2014learning} first showed that Ising structure can be learned from a single Glauber trajectory. Following this, there has been a line of work improving sample complexity, generalizing to other dynamics, and weakening the observation model~\citep{gaitonde2023unifiedapproachlearningising, gaitonde2025bypassing, gaitonde2025better, dutt2021exponentialreductionsamplecomplexity}. For Gaussian graphical models, \citet{tirukkonda2025structure} initiated structure learning from a single Glauber trajectory. Subsequently, \citet{TRD25} proposed two complementary estimators: a local edge-testing estimator (LET-GL), which does not explicitly invoke mixing but whose guarantee requires Dobrushin-type assumptions, and a mixing-based reduction (BTR-GL), which uses a Dobrushin contraction condition $r<1$ to thin the trajectory into a near-independent sample and then invokes any i.i.d.\ Gaussian learner as a black box. Independently, \citet{SWMM26} also study learning GGMs from a Glauber trajectory without a mixing assumption. Relative to this line of work, the present paper removes a tradeoff. The local procedures do not wait for the chain to mix, but they are suboptimal in the signal strength: LET-GL scales as $\kappa^{-5}$ and its guarantee requires Dobrushin-type assumptions, while the local procedure of~\citet{SWMM26} avoids such assumptions but likewise scales as $\kappa^{-5}$. The optimal $\kappa^{-2}$ dependence was previously available only through the mixing-based reduction BTR-GL of~\citet{TRD25}, and only when the chain mixes fast, which is in turn ensured by the Dobrushin contraction condition $r<1$. Our algorithms operate at the level of a candidate/dueling-set support test rather than per-edge testing, inheriting the combinatorial structure of the DICE estimator while working directly on the trajectory, and they attain the $\kappa^{-2}$ dependence without a mixing assumption.

\paragraph{Mixing-based reductions.}
The most direct way to use a trajectory is to wait for the chain to mix and then subsample it into an approximately independent collection, the classical thinning heuristic of MCMC practice~\citep{maceachern1994subsampling, link2012thinning}, whose statistical efficiency has itself been scrutinized~\citep{owen2017statistically, riabiz2022optimal}. Making such a reduction rigorous requires quantitative convergence rates, which for the Gaussian Gibbs sampler are available under Dobrushin-type or log-concavity conditions~\citep{wang2014convergence, wang2017convergence, ascolani2024entropy, wadia2024gibbs} but are either not high-dimensional or not under the total variation metric, or are not entirely assumption-free. Guarantees obtained this way inherit the invoked convergence rate in their sample complexity: the BTR-GL bound of~\citet{TRD25}, for instance, carries a multiplicative $(1-r)^{-1}$ factor. The mixing-based procedure of~\citet{SWMM26} is subject to the same inheritance. Our analysis makes no such appeal. When $\rho_{2d}(\Theta)$ is super-polynomial in $p$, Proposition~\ref{prop:large-rho-slow-relaxation} shows that a mixing horizon valid uniformly over a sufficiently large bounded-potential class is also super-polynomial.

\paragraph{{Estimation from a single trajectory without mixing.}}
{Beyond graphical models, least-squares estimation on a single dependent trajectory without mixing arguments has been developed in linear system identification. For instance, \citet{simchowitz2018learning} show that ordinary least squares identifies a marginally stable linear system at a near-minimax rate, replacing mixing with a block martingale small-ball condition that lower bounds the empirical Gram matrix; see \citet{tsiamis2023statistical} for an overview of this line of work. Our design-control argument for \algoone{} is of a similar spirit: the regenerated rows of Section~\ref{sec:lower-design-control} plays the anti-concentration supplier role that the small-ball condition does there. The two settings differ in several ways including the update structure (single-site updates read at random times), the goal (exact model selection rather than parameter estimation), and initialization.}

\section{Setup and Notation}\label{sec:setup}

We begin with some notation we use throughout the manuscript. 

\textbf{Notation.}  We write \(\mathcal{G} = (\mathcal{V}, \mathcal{E})\) to denote an undirected graph on the vertex set \(\mathcal{V}\) with edges \(\mathcal{E}\subseteq {\mathcal{V} \choose 2}\). For a set $S$, we write $|S|$ for its size, and for an integer $p > 0$, we write $[p]$ to denote the set $\{1,2,\dots, p\}$. For a {$p$-dimensional} vector $X \in \mathbb{R}^{p}$ and $i \in [p]$, let $X_{i}$ denote the $i$th coordinate of $X$ and for a set $S \subset [p]$, let $X_{S} \in \mathbb{R}^{|S|}$ denote the sub-vector of \(X\) indexed by \(S\). We abbreviate $X_{-S}:= X_{[p] \setminus S}$. {For integers $0 \leq n_1 \leq n_2$, we denote the sequence  \(\left(X^{(n_1)}, X^{(n_1+1)}, \dots, X^{(n_2)}\right)\) with the abbreviated notation $X^{(n_1:n_2)}$, allowing for $n_2 = \infty$. We use the same colon convention for any indexed collection, placing the range in the slot of the index it replaces (e.g., $I^{(1:N)}$ and $\xi_{1:n_0}$).} For a matrix $M \in \mathbb{R}^{p \times p}$, and sets $A, B \subset [p]$, let $M_{AB}$ be the {submatrix} with rows indexed from $A$ and columns from $B$. We use $\|X\|_l$ to denote the $l$-norm of a vector $X \in \mathbb{R}^p$. For $M \succeq 0 \in \mathbb{R}^{p \times p}$, the weighted 2-norm of {a} vector $X \in \mathbb{R}^p$ is defined by $\|X\|_M=\sqrt{X^{\top} M X}$. We use $\lambda_{\min }(M)$ to denote the minimum eigenvalue of the positive definite matrix $M$. For sigma algebras $\mathscr A$ and $\mathscr B$, we write $\mathscr A\vee\mathscr B:=\sigma(\mathscr A\cup\mathscr B)$ for their join, the smallest sigma algebra containing both.

\subsection{Gaussian Graphical Models}\label{sec:ggm}

Let $\Gcal = (\Vcal, \Ecal)$ be an undirected graph with vertex set $\Vcal = [p]$ and edge set $\Ecal \subseteq \binom{\Vcal}{2}$, and let $X = (X_1, \ldots, X_p) \sim \mathcal{N}_p(0, \Theta^{-1})$ with positive-definite precision matrix $\Theta \succ 0$ and covariance $\Sigma = \Theta^{-1}$. The pair $(\Gcal, X)$ is a \emph{Gaussian graphical model} (GGM)~\cite{lauritzen1996} when
\begin{equation}
    \Theta_{ij} = 0 \quad \Longleftrightarrow \quad \{i,j\} \notin \Ecal.
    \label{eq:ggm-edge}
\end{equation}
For each node $i \in \Vcal$, we write $S_i \coloneq \{ j \in \Vcal \setminus \{i\} : \{i, j\} \in \Ecal \}$ for its neighborhood, $d_i \coloneq |S_i|$ for the degree, and $d \coloneq \max_{i \in \Vcal} d_i$ for the maximum degree. For each $i \in \Vcal$, the conditional law of $X_i$ given $X_{-i}$ is Gaussian~\cite{lauritzen1996},
\begin{equation}
    X_i \big| X_{-i} \sim
    \mathcal{N}\!\left( \sum_{j \in S_i} \beta_{ij} X_j, \sigma_i^2 \right),
    \qquad
    \beta_{ij} \coloneq -\frac{\Theta_{ij}}{\Theta_{ii}},
    \quad
    \sigma_i^2 \coloneq \Theta_{ii}^{-1},
    \label{eq:full-cond}
\end{equation}
with $\beta_{ij} = 0$ when $j$ is not a neighbor {of} $i$ by~\eqref{eq:ggm-edge}. {We write} $\beta_i \coloneq (\beta_{ij})_{j \in \Vcal \setminus \{i\}} \in \Rbb^{p-1}$ for the full regression-coefficient vector; for $A \subseteq \Vcal \setminus \{i\}${, we write $\beta_{i,A} \coloneq (\beta_{ij})_{j \in A}$ for its restriction to $A$}. For an edge $\{i,j\} \in \Ecal$, {we} define the \emph{normalized edge strength}
\begin{equation}
    \kappa_{ij} \coloneq
    \frac{|\Theta_{ij}|}{\sqrt{\Theta_{ii} \Theta_{jj}}} = |\beta_{ij}|\sqrt{\frac{\Theta_{ii}}{\Theta_{jj}}} = \frac{\lvert \beta_{ij} \rvert}{\sigma_{i}\sqrt{\Theta_{jj}}}. \label{eq:kappa-ij}
\end{equation}
The \emph{minimum normalized edge strength} of the model is $\kappa \coloneq \min_{\{i,j\} \in \Ecal} \kappa_{ij}$.
The parameter $\kappa$ measures the strength of the weakest edge in the model and governs the intrinsic difficulty of structure learning. In the i.i.d.\ setting, \citet{wang2010information} show that exact recovery requires $\Omega(\log p/\kappa^{2})$ samples, and \citet{TRD25} and \citet{SWMM26} establish lower bounds with the same $\kappa^{-2}$ scaling for observations from a single Glauber trajectory\footnote{The bound of \citet{TRD25} is stated in terms of the minimum regression coefficient $\beta_{\min}=\min_{\{i,j\}\in\Ecal}|\Theta_{ij}|/\Theta_{ii}$. The precision matrices in their hard ensemble have constant diagonal, on which $\beta_{\min}$ and $\kappa$ coincide, so the bound reads $\Omega(p\log(p-d)/\kappa^{2})$.}. 

\subsection{Random-scan Gaussian Glauber dynamics}\label{sec:glauber}

The data are generated by discrete-time random-scan Glauber dynamics started from a deterministic state $X^{(0)}=x^{(0)}\in\Rbb^p$. At time $t\ge1$, an independent update label is drawn uniformly at random, i.e., $I^{(t)}\sim\operatorname{Unif}(\Vcal)$. {If} $I^{(t)}=i$, then the update is done as follows:
\begin{align}
    X_i^{(t)}
    &=
    \sum_{j\in S_i}\beta_{ij}X_j^{(t-1)}
    +
    \varepsilon_i^{(t)},
    \qquad
    \varepsilon_i^{(t)}\sim\mathcal N(0,\sigma_i^2),
    \label{eq:glauber-additive}\\
    X_j^{(t)}
    &=
    X_j^{(t-1)}
    \qquad \text{for all } j\neq i.
    \label{eq:glauber-frozen}
\end{align}
We note that the update rule~\eqref{eq:glauber-additive} replaces the state of node $i$ with a linear combination of its neighbors' states plus an independent Gaussian innovation. Local averaging updates of this form are the building blocks of randomized gossip algorithms for distributed computation~\citep{boyd2006randomized} and of DeGroot-type models of opinion formation, including their asynchronous variants~\citep{degroot, elboim2024asynchronous, Asynchronous_opinion_dynamics_in_social_networks}. Trajectories of the kind we study here therefore arise not only from Gibbs samplers but also as observations of such networked processes.

{Before proceeding, let us pause and examine the randomness underlying the update equations~\eqref{eq:glauber-additive} and~\eqref{eq:glauber-frozen}, as our analysis leans on its precise structure.
The process is driven by two independent sources: the label sequence
$I^{(1:\infty)}$, drawn i.i.d.\ uniformly from
$\Vcal$, and an array $\{\varepsilon_i^{(t)}: i\in\Vcal,\, t\ge 1\}$ of
mutually independent innovations with $\varepsilon_i^{(t)}\sim\mathcal
N(0,\sigma_i^2)$. The update at time $t$ consumes the entry
$\varepsilon_{I^{(t)}}^{(t)}$ of this array. Notice that this array is independent of the label sequence. The innovation $\varepsilon_{I^{(t)}}^{(t)}$ at time \(t\)  is not, since its variance is determined by the realized label. In particular, conditionally on the label sequence, the innovations consumed by the dynamics are independent Gaussians whose variances are fixed by the labels, and it is in this conditional form that the independence is used throughout our analysis.} For a deterministic horizon $N > 0$, the observed data are the pair {$(X^{(0:N)},I^{(1:N)})$}.  Given the observed data {$(X^{(0:N)},I^{(1:N)})$}, {we next associate with each subset of coordinates a \emph{clock}, recording the times at which the subset is updated, and a \emph{trace}, recording the states observed at those times.}

\begin{definition}[{$B$-clock}]\label{def:set-update-times}
Fix a nonempty set $B\subseteq \Vcal$. Set $t_B^{(0)}:=0$ and, for
$l\ge1$, define
\begin{equation}
    t_B^{(l)}
    :=
    \inf\{t>t_B^{(l-1)}: I^{(t)}\in B\}.
    \label{eq:update-times}
\end{equation}
{We call the ordered sequence $\mathcal{T}_B \coloneq \left(t^{(l)}_B\right)_{l \geq 1}$ the {${B}$-clock}: it ticks precisely when a coordinate in $B$ is refreshed, and we refer to its entries as the \emph{$B$-update times} or simply ticks of the \(B\)-clock.}
\end{definition}

\begin{definition}[{$B$-trace}]\label{def:set-subtrajectory}
Fix a nonempty set $B\subseteq \Vcal$. {The \emph{$B$-trace} is the trajectory read at the ticks of the $B$-clock,}
\begin{equation}
    {X^{(\mathcal{T}_B)}} := \left( X^{(t_B^{(r)})} \right)_{r \geq 1} = \left(X^{(t)}\right)_{t \in \mathcal{T}_B}.
    \label{eq:restricted-subtrajectory}
\end{equation}
\end{definition}

{Notice that the $B$-trace retains the full $p$-dimensional state; it is only the observation times that are restricted. We can write, for instance, $X_A^{(\mathcal{T}_B)}$ to denote the $A$-coordinates of the states observed along the $B$-clock. With a mild abuse of notation we write $\mathcal{T}_{i} := \mathcal{T}_{\{i\}}$ for the $i$-clock and $X^{(\mathcal{T}_i)}$ for the $i$-trace. Finally, since some coordinate is refreshed at every time step, the $\Vcal$-clock is the entire timeline, $\mathcal{T}_{\Vcal} = (1,2,\dots)$; we call it the \emph{global clock}, and the observed samples excluding the initialization are the first $N$ entries of the $\Vcal$-trace, namely $X^{(1:N)}$.}

\subsection{{Structure Learning Problem}}\label{sec:problem}
{We formalize the recovery problem as exact identification of the edge
set $\Ecal$ of an unknown GGM from a single Glauber trajectory.}

\paragraph{{Model class.}}
Fix integers $p \geq 3$, $d \in
\{1,\ldots,\lfloor(p-1)/2\rfloor\}$, and $\kappa \in (0,1]$; thus,
$p\geq 2d+1$. Let $\Mcal(p,d,\kappa)$ be the set of precision matrices
$\Theta \succ 0$ on $[p]$ whose graph
$\Gcal(\Theta)=([p],\Ecal(\Theta))$, defined by~\eqref{eq:ggm-edge},
has maximum degree at most $d$ and minimum normalized edge strength at
least $\kappa$. The restriction $p\geq 2d+1$ ensures that every
size-$d$ candidate admits a disjoint size-$d$ dueling set and that the
size-$2d$ sets used below exist.

\paragraph{{Observation model.}}
{For an unknown $\Theta \in \Mcal(p,d,\kappa)$, we observe a single
length-$N$ trajectory of random-scan Gaussian Glauber dynamics targeting
$\mathcal N_p(0,\Theta^{-1})$, started from a deterministic
initialization $X^{(0)}=x^{(0)}\in\Rbb^p$. The data are the state
sequence $X^{(0:N)}$ and the update labels $I^{(1:N)}$.}

\paragraph{{Estimator and recovery horizons.}}
Let $\mathfrak E_{p,d}$ be the collection of edge sets
$E\subseteq \binom{[p]}{2}$ whose graph on $[p]$ has maximum degree at
most $d$. A \emph{structure estimator} is a measurable map
$\Acal:\Rbb^{(N+1)\times p}\times [p]^N\to\mathfrak E_{p,d}$, with
output
\begin{align}
        \widehat\Ecal \coloneq
        \Acal\bigl(X^{(0:N)},I^{(1:N)}\bigr).
\end{align}
Write $\Pbb_{\Theta,x^{(0)}}$ for the law of the data
$(X^{(0:N)},I^{(1:N)})$ generated by the dynamics of
Section~\ref{sec:glauber} with target
$\mathcal N_p(0,\Theta^{-1})$ and deterministic initialization
$X^{(0)}=x^{(0)}$. For a fixed pair $(\Theta,x^{(0)})$, the
\emph{pointwise recovery horizon} of $\Acal$ at confidence
$\delta\in(0,1)$ is
\begin{equation}
    N_{\mathrm{pt}}^\star(\Theta,x^{(0)},\delta;\Acal)
    \coloneq
    \min\Bigl\{N\in\Nbb:
        \Pbb_{\Theta,x^{(0)}}\!
        \bigl(\widehat\Ecal\neq\Ecal(\Theta)\bigr)
        \leq \delta
    \Bigr\}.
    \label{eq:pointwise-recovery-horizon}
\end{equation}
For a class
$\mathfrak C\subseteq\Mcal(p,d,\kappa)\times\Rbb^p$ of precision
matrices and deterministic initializations, the corresponding
\emph{uniform sample complexity} is
\begin{equation}
    N_{\mathrm{unif}}^\star(\mathfrak C,\delta;\Acal)
    \coloneq
    \min\Bigl\{N\in\Nbb:
        \sup_{(\Theta,x^{(0)})\in\mathfrak C}
        \Pbb_{\Theta,x^{(0)}}\!
        \bigl(\widehat\Ecal\neq\Ecal(\Theta)\bigr)
        \leq \delta
    \Bigr\}.
    \label{eq:uniform-sample-complexity}
\end{equation}
We write
$\mathfrak C_{\mathrm{base}}(p,d,\kappa)
:=\Mcal(p,d,\kappa)\times\Rbb^p$ for the base class of all permitted
precision matrices and deterministic initializations. Theorem~\ref{thm:algoone-main}
bounds the pointwise recovery horizon of \algoone{}, and
Corollary~\ref{cor:algoone-uniform-envelope} gives a uniform guarantee
after imposing explicit bounds on $\rho_{2d}(\Theta)$ and the
initialization potential. Theorem~\ref{thm:algo2-finite-sample} gives a
uniform guarantee for \algotwo{} over $\mathfrak C_{\mathrm{base}}(p,d,\kappa)$.
We allow both estimators to depend on $p$, $d$, $\kappa$, and
$\delta$.

\section{A Meta Algorithm: \dunes}\label{sec:skeleton}

We now introduce the Dueling Neighborhood Search (\dunes{}) meta algorithm to recover the graph one node at a time. It is built around a \emph{local statistic}: a data-driven score $\ts \in \mathbb{R}$, indexed by a node $i$, a tested set $A \subseteq \Vcal\setminus\{i\}$, and a \emph{probe} $j \in A$, that measures the influence of the probe on node $i$ after accounting for the remaining coordinates of $A$. The neighborhood search proceeds via a sequence of duels. For the node under consideration, a \emph{candidate} neighborhood $C \subseteq \mathcal{V}\setminus \{i\}$ with $|C| = d$ is tested against every \emph{dueling set} $D \subseteq \mathcal{V} \setminus (C \cup \{i\})$ with $|D| = d$. Writing $A = C \cup D$ for the tested set, the candidate survives the duel if no \emph{challenger}, that is, no probe $j \in D$ drawn from the dueling side, registers a large score. The rationale is simple. If $C$ misses a true neighbor of $i$, then some dueling set contains the missed neighbor, and that challenger registers a large influence on $i$ and defeats the candidate. If instead $C \supseteq S_i$, then every challenger lies outside the neighborhood, its influence on $i$ is screened off by the candidate, and $C$ survives all of its duels. For this rationale to be sound, the statistic must \emph{separate} neighbors from non-neighbors whenever the tested set contains the true neighborhood (see Definition~\ref{def:generic-separation}). That is, one should be able to select a threshold $\tau$ such that, whenever $A = C \cup D \supseteq S_{i}$, for every $j \in A$:
\begin{align}
    \ts > \tau \text{ if } j \in S_{i} \qquad \text{ and } \qquad \ts \le \tau \text{ if } j \notin S_{i}.\label{eq:generic-separation-intuitive}
\end{align}

This dueling-sets formulation is inspired by the DICE (Degree-constrained Inverse Covariance Estimator) algorithm of~\citet{misra2020information}, which matches the information-theoretic lower bound for GGM structure learning from independent samples up to a factor of the maximum degree $d$. DICE proceeds in three phases: a first important phase estimates each conditional variance $\Theta_{ii}^{-1}$ by a cardinality-constrained least-squares problem; next, a support-testing phase thresholds estimated normalized edge strengths $\widehat\kappa_{ij}$, normalized by the first-phase variance estimates, at $\kappa/2$; and a final phase prunes the accepted candidate by thresholding the same statistic. The variance-estimation phase and the ensuing statistical machinery there are specific to i.i.d.\ samples and do not carry over to our Glauber setting.

Our procedure abstracts the combinatorial search and its recovery logic away from this specific statistic. We retain the candidate/dueling-set structure and the final phase that prunes the accepted candidate, but replace the estimated edge strength with a generic local statistic $\ts$. And, in Definition~\ref{def:generic-separation}, we isolate the one property, separation, that the search needs. As we shall see (Proposition~\ref{prop:generic-separation-recovery}), exact recovery is then a deterministic consequence. This abstraction lets us carry the search over to our setting: DICE certifies its statistic using empirical covariances of iid Gaussian samples, a luxury the Glauber trajectory does not afford. Our data form a single dependent trajectory in which coordinates refresh only at random update times and consecutive states differ in a single coordinate. All of the probabilistic work in the sequel therefore goes into establishing separation, over a finite horizon, for two statistics built directly from this dependent data.

We now specify the procedure under the standing model-class
restriction $p\geq 2d+1$, which ensures that all size-$d$ dueling sets
below exist. The procedure is nodewise, so we fix a node $i\in\Vcal$ throughout the nodewise routine below. The full algorithm applies this routine to each node and then aggregates (with a union) the resulting neighborhood estimates. Throughout this nodewise description, we abbreviate $\bar A \coloneq A \cup \{i\}$ for the tested set augmented with the node itself.

\textbf{Phase 1:}
{For a candidate set $C\subseteq\Vcal\setminus\{i\}$ with $|C|=d$, let}
\begin{align}
    \mathcal D(C) := \bigl\{ D\subseteq\Vcal\setminus(C\cup\{i\}): |D|=d \bigr\}
\end{align}
{be the collection of dueling sets against which $C$ is tested. A candidate $C$ is \emph{accepted} if}
\begin{equation}
    \max_{j\in D}\tsat{i}{C\cup D}{j} \le \tau \qquad \text{for every }D\in\mathcal D(C).
    \label{eq:generic-acceptable}
\end{equation}
{Let $\widehat C_i$ be any accepted candidate.}

\textbf{Phase 2:}
{Choose any $D\in\mathcal D(\widehat C_i)$ and set $A:=\widehat C_i\cup D$. The final nodewise estimate is obtained by thresholding the statistic once more on $\widehat C_i$:}
\begin{equation}
    {\widehat S_i := \bigl\{j\in\widehat C_i: \tsat{i}{A}{j}>\tau\bigr\}.}
    \label{eq:generic-cleanup}
\end{equation}
{After repeating this nodewise routine for every $i\in\Vcal$, aggregate the nodewise estimates by the OR rule,}
\begin{equation}
    \widehat \Ecal := \bigl\{\{i,j\}: j\in\widehat S_i \text{ or } i\in\widehat S_j\bigr\}.
    \label{eq:generic-Ehat}
\end{equation}

{Algorithm~\ref{alg:generic-search} summarizes the meta-procedure. Its only data-dependent input is the local statistic $\tssym$, which each of our two algorithms instantiates in the sequel. We allow the statistic to return $\undeftok$ (which as we will see below happens when there is insufficient data) to instantly terminate the procedure, in which case the search outputs the empty edge set.}

\begin{algorithm}[t]
\caption{{\dunes{} with local statistic $\tssym$}}
\label{alg:generic-search}
\begin{algorithmic}[1]
\Require {local statistic
$\tssym:\mathfrak T\to\Rbb\cup\{\undeftok\}$; threshold $\tau>0$}
\If{$\tsat{i}{A}{j}=\undeftok$ for some $(i,A,j)\in\mathfrak T$}
    \State \Return $\widehat\Ecal:=\emptyset$
    \Comment{insufficient data}
\EndIf
\For{$i\in\Vcal$}
    \For{$C\subseteq\Vcal\setminus\{i\}$ with $|C|=d$}
    \Comment{Phase 1: candidate search}
        \If{$\max_{j\in D}\tsat{i}{C\cup D}{j}\leq\tau$ for every
        $D\in\mathcal D(C)$}
            \State $\widehat C_i\gets C$; \textbf{break}
            \Comment{accept the first acceptable candidate}
        \EndIf
    \EndFor
    \If{no candidate was accepted}
        \State $\widehat S_i\gets\emptyset$; \textbf{continue}
    \EndIf
    \State choose any $D\in\mathcal D(\widehat C_i)$ and set
    $A\gets\widehat C_i\cup D$
    \Comment{Phase 2: cleanup}
    \State $\widehat S_i\gets
    \bigl\{j\in\widehat C_i:\tsat{i}{A}{j}>\tau\bigr\}$
\EndFor
\State \Return $\widehat\Ecal:=\bigl\{\{i,j\}: j\in\widehat S_i
\text{ or } i\in\widehat S_j\bigr\}$
\Comment{OR rule}
\end{algorithmic}
\end{algorithm}

The deterministic property needed of the test statistic $\ts$ for accurate graph learning is uniform separation over the triples queried by \dunes{}. We write the set of queried triples as
\begin{equation}
    \mathfrak T := \bigl\{ (i,A,j): i\in\Vcal,  A\subseteq\Vcal\setminus\{i\},  |A|=2d,  j\in A \bigr\}.  \label{eq:tested-triples}
\end{equation}
The token $\undeftok$ is only an availability flag: it records that a data-dependent statistic cannot be evaluated on some queried triple. We say $\tssym$ is \emph{available} on $\mathfrak T$ if $\tsat{i}{A}{j}\in\Rbb$ for every $(i,A,j)\in\mathfrak T$. We can now state the separation property on the event that the statistic is available.

\begin{definition}[Separation property]\label{def:generic-separation}
When $\tssym$ is available on $\mathfrak T$, we say it satisfies the separation property if there exists a uniform threshold $\tau$ such that, for every $(i,A,j)\in\mathfrak T$ with $S_i\subseteq A$,
\begin{align}
\tsat{i}{A}{j}\le \tau \quad \text{if } j\notin S_i,
    \qquad \text{and}\qquad
    \tsat{i}{A}{j}> \tau \quad \text{if } j\in S_i.
    \label{eq:generic-separation}
\end{align}
\end{definition}

The following proposition separates availability from correctness: once the statistic is available, the separation property alone drives exact recovery. This result follows immediately from \cite{misra2020information}, however, we provide a proof in Appendix~\ref{app:generic-separation-recovery-proof} for the sake of completeness. 

\begin{proposition}[Deterministic recovery from separation]
\label{prop:generic-separation-recovery}
Assume $p\geq 2d+1$. Suppose $\tssym$ is available on
$\mathfrak T$ and satisfies the separation property with threshold
$\tau$. Then Algorithm~\ref{alg:generic-search} run with statistic
$\tssym$ and threshold $\tau$ outputs $\widehat\Ecal=\Ecal$.
\end{proposition}

In the next two sections we describe two instantiations of \dunes{}: \algoone{} (Section~\ref{sec:algoone}) and \algotwo{} (Section~\ref{sec:algotwo}).

\section{\algoone{}}\label{sec:algoone}

With the \dunes{} template and its separation requirement in place, what remains is to construct a local statistic that meets them. We instantiate the \dunes{} template with a local statistic based on linear regression using conditional updates of the Glauber dynamics. We call this instantiation Least-Squares \dunes{} (\algoone{}).

The DICE statistic we are replacing is a useful guide for the one we must build. Recall from Section~\ref{sec:skeleton} that DICE tests candidates with an empirical estimate of the normalized edge strength $\kappa_{ij}$ in~\eqref{eq:kappa-ij}: fitted partial-regression coefficients, normalized using the conditional-variance estimates from its first phase. Two features of this score do the statistical work. It is invariant to coordinate scaling, and it separates zero from nonzero partial correlations at the signal scale. Any statistic that replaces it should preserve both.

Indeed, what we cannot preserve is the way DICE computes and certifies its score. The empirical covariance of the raw trajectory concentrates around the target covariance only after the chain has mixed, and such a reduction would reintroduce spectral-gap, condition-number, or Dobrushin-type assumptions~\cite{TRD25}. Instead, \algoone{} uses the structure in the update sequence itself. For each node $i$, we read the trajectory only at the $i$-clock, that is, at the times at which node $i$ itself is refreshed (recall Definitions~\ref{def:set-update-times} and~\ref{def:set-subtrajectory}). Along {the corresponding} $i$-trace, every observed response is generated by the exact conditional Gaussian regression for coordinate $i$, with fresh Gaussian noise at each $i$-update. The covariate rows are dependent, but each response noise is a fresh Gaussian innovation from the corresponding $i$-update.

This suggests a strategy to build the statistic. We can treat the first $m$ entries of the $i$-trace as a local regression problem, and for each candidate coordinate, we estimate its regression coefficient and scale it appropriately (i.e., by the standard error determined by the realized design matrix $Z$). 

The main work is then to show that, despite the dependence among the covariate rows, this regression score separates neighbors from non-neighbors once the $i$-trace contains enough rows.
We first record an instance-dependent quantity that governs the budget of \algoone{}. It measures the covariance $\Sigma$ only through its principal submatrices of size $2d$, rescaled by the corresponding conditional precisions. In this sense, it is a local quantity, in contrast to a global bound on the condition number of $\Theta$. 

\begin{definition}[Degree-constrained spectral radius]\label{def:rho2d}
For every $A\subseteq[p]$, write $D_A:=\diag{(\Theta_{kk}:k\in A)}$. The \emph{degree-constrained spectral radius} of $\Theta$ is
\begin{equation}
\rho_{2d}(\Theta)
    :=
    \max_{\substack{A\subseteq[p]\\ |A|=2d}}
    \lambda_{\max}\!\left(D_A^{1/2}\Sigma_{AA}D_A^{1/2}\right).
    \label{eq:rho2d-definition}
\end{equation}
\end{definition}

\begin{remark}[Interpretation]\label{rmk:rho2d-interpretation}
Let $\nu_{\max}(\Theta):=\max_{j\in[p]}\Var(X_j)/\Var(X_j\mid X_{\Vcal\setminus\{j\}})=\max_{j\in[p]}\Theta_{jj}\Sigma_{jj}$ be the worst nodewise marginal-to-conditional variance ratio. Since the largest eigenvalue of a positive semidefinite matrix is at least each of its diagonal entries and at most its trace,
\begin{equation}
\nu_{\max}(\Theta)
    \;\leq\;
    \rho_{2d}(\Theta)
    \;\leq\;
    2d\,\nu_{\max}(\Theta).
    \label{eq:rho2d-numax-sandwich}
\end{equation}
Up to the sparsity factor $2d$, the degree-constrained spectral radius is thus the largest factor by which conditioning on the rest of the graph reduces the variance of a single node.
\end{remark}
Let $\bar\rho\geq1$ be an upper bound on the degree-constrained spectral radius and $\bar E\geq0$ on the initialization potential. We allow these bounds to be arbitrary functions of $(p,d,\kappa)$. Define the class of precision matrices and deterministic initializations \begin{equation}
   \mathfrak C(p,d,\kappa;\bar\rho,\bar E)
    :=
    \left\{(\Theta,x^{(0)})\in\mathfrak C_{\mathrm{base}}(p,d,\kappa):
    \rho_{2d}(\Theta)\leq\bar\rho,\
    \norm{x^{(0)}}_\Theta^2\leq\bar E
    \right\}.
    \label{eq:algoone-envelope-class}
\end{equation}

We will state the uniform sample complexity defined in~\eqref{eq:uniform-sample-complexity} in terms of this restricted class.  The pointwise guarantee in Theorem~\ref{thm:algoone-main} depends on the degree-constrained spectral radius only through its logarithm, and places no upper bound on it. The $\kappa^{-2}$ dependence is therefore unchanged for every fixed instance. A guarantee obtained by first waiting for mixing, in contrast, must account for a mixing horizon that is valid uniformly over its permitted initializations. Since $\rho_{2d}(\Theta)\leq \lambda_{\min}(D^{-1/2}\Theta D^{-1/2})^{-1}$, a large $\rho_{2d}(\Theta)$ forces the normalized precision matrix toward singularity. The next proposition identifies a bounded-potential
initialization for which this degeneracy slows the dynamics.

\begin{proposition}[Large $\rho_{2d}$ forces slow mixing]
\label{prop:large-rho-slow-relaxation}
There are universal constants $c>0$ and $E_\star<\infty$ such
that, for every $\Theta\in\Mcal(p,d,\kappa)$ and every potential level
$E_0\geq E_\star$, at least $c\,p\,\rho_{2d}(\Theta)$ updates
are needed before the random-scan dynamics is within total variation
distance $1/4$ of $\mathcal N(0,\Theta^{-1})$ from every
initialization $x$ with $\norm{x}_\Theta^2\leq E_0$; that is, the
restricted mixing time defined in~\eqref{eq:energy-ball-mixing-time}
satisfies
$t_{\mathrm{mix}}^{(E_0)}(1/4)\geq c\,p\,\rho_{2d}(\Theta)$.
\end{proposition}
The formal definition of the restricted mixing time, and the
proof, appear in Appendix~\ref{app:large-rho-slow-mixing}.
Proposition~\ref{prop:large-rho-slow-relaxation} calibrates the
comparison with mixing-based guarantees~\citep{TRD25, SWMM26}: their
horizons grow with the mixing time, which is at least of order
$p\,\rho_{2d}(\Theta)$ from bounded-potential initializations, and is
therefore super-polynomial in $p$ whenever $\rho_{2d}(\Theta)$ is
super-polynomial.\footnote{Along a sequence $(p,d(p),\Theta_p)$, this
means $\rho_{2d(p)}(\Theta_p)=p^{\omega(1)}$, equivalently
$\rho_{2d(p)}(\Theta_p)/p^C\to\infty$ for every fixed $C>0$.}
Theorem~\ref{thm:algoone-main} pays only $\log\rho_{2d}(\Theta)$.
With these properties in place, we turn to the regression identity that makes the construction possible.
\subsection{Regression score and test statistic}\label{sec:algoone-regression-score}
The construction rests on a single structural fact: along the $i$-clock, each observed update of node $i$ obeys the exact conditional regression~\eqref{eq:full-cond} on the frozen coordinates, with fresh Gaussian noise. The next lemma records this fact.
\begin{lemma}[Regression form on the $i$-trace]\label{lem:algoone-subtrajectory-regression}
Fix a node $i$ and a set $A\subseteq \Vcal\setminus\{i\}$ and consider samples from the $i$-trace $X^{(\mathcal{T}_i)} = (X^{(t^{(r)}_{i})})_{r \geq 1}$.  If $A\supseteq S_i$, then for every $r\geq 1$,
\begin{equation}
    Y^{(r)}
    =
    \bigl(Z_A^{(r)}\bigr)^\top\beta_{i,A}
    +
    \varepsilon^{(r)},
    \qquad
    \varepsilon^{(r)}\stackrel{\mathrm{iid}}{\sim}
    \mathcal N(0,\sigma_i^2)
    \label{eq:algoone-subtrajectory-regression}
\end{equation}
where $Y^{(r)} := X_i^{(t_i^{(r)})}$ is the $r$th response variable, and $Z_A^{(r)} := X_A^{(t_i^{(r)})} = X_A^{(t_i^{(r)}-1)}$ is the $r$th covariate vector and $\epsilon^{(r)} \coloneq \epsilon^{(t^{(r)})}_{i}$.
\end{lemma}

We prove this in Appendix~\ref{app:algoone-subtrajectory-regression-proof}. The $i$-trace therefore satisfies a familiar regression equation at each update of $i$. Importantly, even though its covariate rows remain dependent, the corresponding response noise is injected {\em only after\/} that row has been determined. We will use this crucial structure in our construction and analysis of the test statistic. In the sequel, we suppose that we only work with the first $m$ such updates. Sections~\ref{sec:design-control} and~\ref{sec:algoone-separation} will reveal how large $m$ needs to be for our separation condition to hold (with high probability), while Section~\ref{sec:algoone-finite-horizon} later shows that these number of rows occur by a prescribed horizon of the full dynamics (with high probability).

Stacking $m$ rows of \eqref{eq:algoone-subtrajectory-regression}, we can write
    \begin{equation}
        \mathbf{Y} = \mathbf{Z}\beta_{i,A} + \boldsymbol{\varepsilon},
        \label{eq:algoone-matrix-regression}
    \end{equation}
where $\mathbf{Y}:=(Y^{(1)},\ldots,Y^{(m)})^\top$, $\boldsymbol{\varepsilon}:=(\varepsilon^{(1)},\ldots,\varepsilon^{(m)})^\top$, and $\mathbf{Z}\in\Rbb^{m\times |A|}$ has $r$th row $\bigl(Z_A^{(r)}\bigr)^\top$. Therefore, when $\mathbf{Z}^\top\mathbf{Z}$ is invertible, the regression coefficients $\beta$ can be estimated using ordinary least squares (OLS) 
\begin{align}
\widehat\beta_{i,A} :=(\mathbf{Z}^\top \mathbf{Z})^{-1}(\mathbf{Z}^\top \mathbf{Y}) \quad \text{ and } \quad \widehat\beta_{i,A}-\beta_{i,A} := (\mathbf{Z}^\top \mathbf{Z})^{-1}(\mathbf{Z}^\top \boldsymbol{\varepsilon}).
\end{align}
To keep track of the two random quantities that drive the estimation error, define the empirical Gram matrix and the noise-design cross term by
\begin{align}
    V_m &:= \frac{1}{m}\mathbf{Z}^\top\mathbf{Z}, \qquad S_m :=
    \frac{1}{m}\mathbf{Z}^\top\boldsymbol{\varepsilon}.
    \nonumber
\end{align}
The OLS error can then be written as
\begin{align}
    \widehat{\beta}_{i,A} - \beta_{i,A} = {V_m}^{-1}S_m.
\end{align}
In particular, its error in the weighted 2-norm $\norm{\cdot}_{V_m}$ (or Mahalanobis norm, defined in Section~\ref{sec:setup}) is
\begin{equation}
    \norm{\widehat\beta_{i,A}-\beta_{i,A}}_{V_{m}}^{2} = S_m^\top V_m^{-1}S_m.
    \label{eq:algoone-ols-error}
\end{equation}

This identity isolates the two random objects that our analysis must control: the empirical Gram matrix $V_m$ and the noise--design cross term $S_m$. It also suggests the form our test statistic must assume. The fitted regression vector $\widehat{\beta}$ provides one coefficient $\widehat\beta_{ij}$ for each candidate $j\in A$. Our question is whether an observed coefficient is evidence of a nonzero conditional effect, or whether it is just noise. The answer of course depends on how well the realized covariate rows $\mathbf{Z}$ determine that particular coefficient. To see this, recall the behavior of least squares with a fixed design: conditionally on $\mathbf{Z}$, the coefficient $\widehat\beta_{ij}$ has standard deviation $\sigma_i\sqrt{((\mathbf{Z}^\top\mathbf{Z})^{-1})_{jj}}=\frac{\sigma_i}{\sqrt{m}}\sqrt{(V_m^{-1})_{jj}}$, and $((\mathbf{Z}^\top\mathbf{Z})^{-1})_{jj}$ is exactly the inverse of the squared length of the component of the $j$th column of $\mathbf{Z}$ orthogonal to the span of the remaining columns (see, e.g.,~\citep{Horn_Johnson_2012_Matrix, seber2003linear}). Thus, when the covariate for $j$ carries little variation beyond the other covariates in $A$, this orthogonal component is small and the fitted coefficient fluctuates more. The factor $\sqrt{(V_m^{-1})_{jj}}$ therefore records the design-dependent part of the fluctuation, while the noise scale $\sigma_i$ sets the remainder. Since $\sigma_i$ is unknown, we estimate it from the residual variance of the same regression.
\begin{definition}[Conditional variance estimator]
\label{def:algoone-conditional-variance-estimator}
    Define
    \begin{equation}
        \widehat\sigma_{i,A}^{2} := \frac{1}{m} \left\| \mathbf{Y}-\mathbf{Z}\widehat\beta_{i,A} \right\|_{2}^{2}
        = \frac{1}{m} \sum_{r=1}^{m} \left( Y^{(r)}-\bigl(Z_A^{(r)}\bigr)^\top\widehat\beta_{i,A} \right)^2.
        \label{eq:algoone-conditional-variance-estimator}
    \end{equation}
    Equivalently, the associated diagonal-precision estimator is $\widehat\Theta_{ii,A} := \widehat\sigma_{i,A}^{-2}$.
\end{definition}

Combining the realized-design factor with this residual estimate gives the local score used by \algoone{}.

\begin{definition}[\algoone{} test statistic]
\label{def:algoone-plugin-test-statistic}
 Fix a node $i$, a tested set $A\subseteq\Vcal\setminus\{i\}$, and a candidate coordinate $j\in A$. On the event $V_m\succ0$ and $\widehat\sigma_{i,A}>0$, define
    \begin{equation}
        \widehat{\tssym}_{i,A}(j) := \frac{\lvert\widehat\beta_{ij}\rvert} {\widehat\sigma_{i,A}\sqrt{(V_m^{-1})_{jj}}}.
        \label{eq:algoone-plugin-test-statistic}
    \end{equation}
If $V_m$ is singular or $\widehat\sigma_{i,A}=0$, we set $\widehat{\tssym}_{i,A}(j):=+\infty$.
\end{definition}

\subsection{Comparison with a Partial Oracle}
\label{sec:algoone-partial-oracle-estimator}

For the analysis, it is helpful to first replace the residual variance by the true conditional variance. Indeed, the resulting ``partial oracle'' score is not computed by the algorithm. We simply use it as an intermediate quantity that isolates coefficient-estimation error from the additional error caused by estimating $\sigma_i^2$.
\begin{definition}[Partial oracle estimates]
\label{def:algoone-partial-oracle-estimates}
    Fix a node $i$, a tested set $A\subseteq\Vcal\setminus\{i\}$, and $j\in A$.  On the event $V_m\succ0$, define
    \begin{equation}
        \widehat{\Psi}^{\ast}_{i,A}(j)
        :=
        \frac{\lvert\widehat\beta_{ij}\rvert}
        {\sigma_i\sqrt{(V_m^{-1})_{jj}}},
        \qquad
        \Psi^{\ast}_{i,A}(j)
        :=
        \frac{\lvert\beta_{ij}\rvert}
        {\sigma_i\sqrt{(V_m^{-1})_{jj}}}.
        \label{eq:algoone-partial-oracle-estimates}
    \end{equation}
\end{definition}

To show that our test statistic is effective without unnecessary condition number assumptions, we must directly compare it with the normalized edge strength $\kappa_{ij}$. We therefore work in the following normalized covariate coordinates for analytical convenience.
\begin{equation}
D_A:=\diag{(\Theta_{kk}:k\in A)},\qquad
    \widetilde Z^{(r)}:=D_A^{1/2}Z_A^{(r)},\qquad
    \widetilde V_m:=\frac{1}{m}\sum_{r=1}^m
    \widetilde Z^{(r)}{\bigl(\widetilde Z^{(r)}\bigr)}^\top
    =D_A^{1/2}V_mD_A^{1/2}.
    \label{eq:algoone-normalized-design}
\end{equation}
This change of coordinates is used only in the proof, and our test statistic is still defined in terms of the original empirical Gram matrix $V_m$. A lower bound on $\widetilde V_m$ will later turn the oracle score into a lower bound on $\kappa_{ij}$, while an upper bound will control the self-normalized error term.

We next control the error in this oracle comparison. The following bound is the quantity that will enter the separation argument. {We postpone its proof to Appendix~\ref{app:algoone-partial-oracle-error-proof}{, and only sketch the idea here. When the OLS error identity~\eqref{eq:algoone-ols-error} is rewritten in the normalized coordinates of~\eqref{eq:algoone-normalized-design}, the error takes the form of a noise term normalized by the rescaled empirical Gram matrix $\widetilde V_m$, a matrix built from the same random rows as the noise term. Quantities of this self-normalized form can be controlled by the martingale inequality of~\citet[Theorem~1]{abbasi2011improved} even though the covariate rows are dependent. The two-sided spectral condition appearing in the lemma then converts this control into a bound on each coordinate.}}
\begin{lemma}
\label{lem:algoone-partial-oracle-error}
Fix a tested pair $(i,A)$ and deterministic constants $0<\lambda_-\leq\lambda_+$. For every $\delta\in(0,1)$,
\begin{equation}
\Pbb\!\left(
    \bigl\{\lambda_- I_A\preceq\widetilde V_m\preceq\lambda_+ I_A\bigr\}
    \cap
    \left\{
    \max_{j\in A}
    \left| \widehat{\Psi}^{\ast}_{i,A}(j)- \Psi^{\ast}_{i,A}(j) \right|
    >
    2\sqrt{\frac{1}{m}
    \left(
        d\log\frac{2\lambda_+}{\lambda_-}
        +\log\frac{1}{\delta}
    \right)}
    \right\}
    \right)
    \leq\delta .
    \label{eq:algoone-partial-oracle-score-error}
\end{equation}
\end{lemma}

Lemma~\ref{lem:algoone-partial-oracle-error} provides useful performance bounds for our test statistic as soon as the deterministic constants $\lambda_-$ and $\lambda_+$ are available. The next subsection provides two-sided bounds on the rescaled empirical Gram matrix $\widetilde V_m$ along the $i$-trace.

\subsection{Controlling the Spectrum of the Empirical Gram Matrix}\label{sec:design-control}

{The regression analysis has reduced our task to controlling the spectrum of the rescaled empirical Gram matrix $\widetilde V_m$, and we present this control in this section. We state the upper and lower spectral inequalities together in Proposition~\ref{prop:algoone-design-control}, but the two sides rest on different features of the trajectory. For the upper bound, we control the largest covariate row that can enter $\widetilde V_m$ by bounding the potential ${\norm{X^{(t)}}_\Theta^2}$ along the trajectory. Indeed, converting such a global potential into a bound on the local covariates is exactly where the degree-constrained spectral radius $\rho_{2d}(\Theta)$ enters our guarantee. For the lower bound, we use the update structure once more: a row of the $i$-trace carries a fresh Gaussian innovation whenever a coordinate of the tested set is refreshed just before the row is observed. We call such rows \emph{regenerated}, and we show that they occur often enough to excite every direction corresponding to the tested set.}

{We first collect the deterministic quantities that enter our bound below. Let $\mathfrak{L}(x):={\norm{x}_\Theta^2}=x^\top\Theta x$, let $q:=2d+1$, and let $E_0:=\mathfrak{L}(X^{(0)})={\norm{X^{(0)}}_\Theta^2}$ denote the initial potential. For $m\geq1$ and $\delta\in(0,1)$, define}
\begin{align}
    t_m(\delta)&:=\left\lceil\max\left\{2pm,8p\log\frac{p}{\delta}\right\}\right\rceil,\nonumber\\
    B_m(\delta)&:=E_0+t_m(\delta/2)
    +2\sqrt{t_m(\delta/2)\log\frac{2}{\delta}
    +2\log\frac{2}{\delta},}
    \label{eq:algoone-Bm}\quad    g_m:=\left\lfloor\frac{m-1}{4d+2}\right\rfloor.
\end{align}
{The upper and lower spectral bounds will take the form}
\begin{align}
    \lambda_m^+(\delta)&:=\max\left\{1,\rho_{2d}(\Theta)B_m(\delta)\right\},\qquad  \lambda_m^-:=\frac{c_0g_m}{4m},\label{eq:algoone-lambda-plus}
\end{align}
{where $c_0:=\frac14\log 3$.}

\begin{proposition}[{Two-sided Control on Gram Matrix Spectrum}]\label{prop:algoone-design-control}
{Fix $\delta\in(0,1)$. There is a universal constant $C\geq1$ such that, if}
\begin{equation}
    m \geq C q \left\{ 2d \log \bigl(Cq \lambda_m^{+}(\delta/3)\bigr) +\log\frac{|\mathfrak T|}{\delta} \right\},
    \label{eq:algoone-design-budget}
\end{equation}
{then, with probability at least $1-\delta$, simultaneously for every $i\in[p]$ and every
$A\subseteq[p]\setminus\{i\}$ with $|A|=2d$,}
\begin{equation}
    \lambda_m^{-} I_A
    \preceq
    \widetilde V_m
    \preceq
    \lambda_m^{+}(\delta/3) I_A .
    \label{eq:algoone-design-control}
\end{equation}
\end{proposition}
The full proof is given in Appendix~\ref{app:algoone-design-control-proof}. The next two subsections develop the two sides in turn.

\subsubsection{{Upper Bound on the Spectrum of the Empirical Gram Matrix}}\label{sec:upper-design-control}

{We begin by observing that for every unit vector $u\in\Rbb^{|A|}$, the empirical Gram matrix satisfies}
\begin{equation}
u^\top\widetilde V_m u
    =\frac{1}{m}\sum_{r=1}^m\left(u^\top\widetilde Z^{(r)}\right)^2
    \leq\max_{r\in[m]}\|\widetilde Z^{(r)}\|_2^2.
    \label{eq:algoone-upper-reduction}
\end{equation}
{Equation~\eqref{eq:algoone-upper-reduction} reduces the upper spectral bound to a bound on the covariate rows themselves: it suffices to control the largest rescaled row that can enter $\widetilde V_m$, simultaneously over all tested pairs. Since each row is read off the trajectory at the time the corresponding node is updated, it is enough to control the potential of every state of the trajectory up to a common time containing all of these update times.}
{Now, suppose that we define $T_m$ as the (random) global time by which the first $m$ rows of every $i$-trace occurs. That is, $T_m:=\max_{i\in[p]}t_i^{(m)}$. The following lemma establishes a deterministic horizon length that contains $T_m$ with high probability.}
\begin{lemma}\label{lem:row-collection}
{For every $m\geq1$ and $\delta\in(0,1)$,}
\begin{equation*}
\Pbb\!\left(T_m>t_m(\delta)\right)\leq\delta.
\end{equation*}
\end{lemma}

{We next control the size of the covariates throughout the window $[0,t_m(\delta/2)]$. As mentioned earlier, we will then use the degree-constrained spectral radius (Definition~\ref{def:rho2d})} to convert the global potential $\mathfrak{L}(x)={\norm{x}_\Theta^2}$ into a bound on the rescaled covariate row $D_A^{1/2}x_A$.
\begin{lemma}\label{lem:algoone-energy-to-covariate}
For every $A\subseteq[p]$ with $|A|=2d$ and every $x\in\Rbb^p$,
\begin{equation}
\|D_A^{1/2}x_A\|_2^2\leq\rho_{2d}(\Theta){\mathfrak{L}(x)}.
    \label{eq:algoone-energy-to-covariate}
\end{equation}
\end{lemma}
{Therefore, our task now becomes one of controlling $\mathfrak{L}(X^{(t)})$ along the deterministic window $[0,t_m(\delta/2)]$. We obtain that control from the explicit nature of single-step updates of the Glauber dynamics. Indeed, the reader may have observed that $\mathfrak{L}$ plays the role of a Lyapunov function here. Such functions are a classical tool for establishing the stability of dynamical systems. In Markov-chain theory, drift conditions on Lyapnov functions are used to establish stability and quantitative convergence of the chain~\citep{meyn1994computable, rosenthal1995minorization}. We use the drift inequality differently: it bounds the potential over a finite window, and we derive no convergence or mixing statement from it. The next lemma records the elementary inequality we need.}
\begin{lemma}\label{lem:algoone-one-step-drift}
{For every update time $t\geq1$ with $I^{(t)}=i$,}
\begin{equation}
    {\mathfrak{L}(X^{(t)})-\mathfrak{L}(X^{(t-1)})
    \leq\left(\zeta_i^{(t)}\right)^2.}
    \label{eq:one-step-lyapunov-drift}
\end{equation}
{where $\zeta_i^{(t)}\stackrel{\mathrm{iid}}{\sim}\mathcal N(0,1)$.}
\end{lemma}

{We now combine these ingredients. We first sum the drift inequality~\eqref{eq:one-step-lyapunov-drift} over the first $t_m(\delta/2)$ updates. This bounds the potential $\mathfrak{L}(X^{(s)})$, simultaneously for all $s\leq t_m(\delta/2)$, by $E_0$ plus a chi-squared sum with $t_m(\delta/2)$ degrees of freedom. We then control the chi-squared sum using a standard tail bound (see~\citet{laurent2000adaptive}). Indeed, this is the origin of the quantity $B_m(\delta)$ in~\eqref{eq:algoone-Bm}. Finally, Lemma~\ref{lem:row-collection}, applied with failure probability $\delta/2$, places $T_m$ inside this window with high probability, and we obtain}
\begin{equation*}
    \max_{0\leq s\leq T_m}\mathfrak{L}(X^{(s)})\leq B_m(\delta)
\end{equation*}
{with probability at least $1-\delta$. On this event, we invoke Lemma~\ref{lem:algoone-energy-to-covariate} to convert the potential bound into a bound on the covariates: every rescaled row that enters $\widetilde V_m$ has squared norm at most $\lambda_m^+(\delta)$. Feeding this into the reduction~\eqref{eq:algoone-upper-reduction}, we obtain the upper inequality of Proposition~\ref{prop:algoone-design-control}, simultaneously over all tested pairs. We provide the fully detailed proof in Appendix~\ref{app:algoone-upper-design-control-proof}. We turn to the lower inequality next.}

\subsubsection{{Lower Bound on the Spectrum of the Empirical Gram Matrix}}\label{sec:lower-design-control}

The upper spectral bound does not prevent the rows that enter $\widetilde V_m$ from concentrating in a smaller subspace. For the lower bound, we use the fact that the rows of the $i$-trace contain fresh (and useful) randomness whenever a coordinate in $A$ is refreshed between two consecutive updates of $i$. For a fixed $j \in A$, we call such a row $(A,j)$-regenerated when the update preceding $i$ in the $\bar A$-clock refreshes coordinate $j$. {This use of conditional anti-concentration in place of mixing parallels the block martingale small-ball method of~\citet{simchowitz2018learning} for linear system identification.}

\begin{definition}[{$(A,j)$-regenerated rows}]\label{def:algoone-regen-row}
Fix $(i,A,j)\in\mathfrak T$. For each row index $r\in\{2,\ldots,m\}$ of the $i$-trace, let $\ell_r$ be the unique index such that the $r$th update of $i$ occurs at the $\ell_r$th update time of $\bar A$. That is, $t_i^{(r)}=t_{\bar A}^{(\ell_r)}$. Write $s_r:=t_{\bar A}^{(\ell_r-1)}$
for the preceding $\bar A$-update time. Row $r$ is said to be $(A,j)$-regenerated if the update at time $s_r$ has label $j$, that is, if
\begin{equation}
I^{(s_r)}=j.
    \label{eq:algoone-regen-row}
\end{equation}
{We write}
\begin{equation}
    \mathfrak R_{i,A,j}^{(m)} :=
    \bigl\{r\in\{2,\ldots,m\}:r\text{ is }(A,j)\text{-regenerated}\bigr\}
    \label{eq:algoone-regen-set}
\end{equation}
{for the set of $(A,j)$-regenerated rows among the first $m$ rows of the $i$-trace.}
\end{definition}

The point of this definition is the following. Since $s_r$ is the last $\bar A$-update time before $t_i^{(r)}$, no coordinate of $\bar A$ changes between the two times. If the update at time $s_r$ refreshes coordinate $j$, the row observed at time $t_i^{(r)}$ therefore carries the fresh innovation $\varepsilon_j^{(s_r)}$ in its $j$th coordinate, while its remaining coordinates are determined by the past. In other words, a regenerated row is an affine function of a single fresh standard Gaussian. The next lemma records this representation.
\begin{lemma}\label{lem:algoone-regen-row-form}
{Fix a tested triple $(i,A,j)\in\mathfrak T$, and fix a row index $r\in\{2,\ldots,m\}$ of the $i$-trace. Let $\ell_r$ be the unique index satisfying $t_i^{(r)}=t_{\bar A}^{(\ell_r)}$, and set $s_r:=t_{\bar A}^{(\ell_r-1)}$ for the preceding $\bar A$-update time. If this preceding update refreshes coordinate $j$, i.e. $I^{(s_r)}=j$, then}
\begin{equation}
\widetilde Z^{(r)} = a^{(r)} + e_j\zeta_j^{(s_r)}, \qquad \zeta_j^{(s_r)} = \sqrt{\Theta_{jj}}\varepsilon_j^{(s_r)} \sim \mathcal N(0,1),
    \label{eq:algoone-affine-regenerated-row}
\end{equation}
{where $e_j$ is the coordinate vector for $j$ inside $A$, and $a^{(r)}$ is measurable with respect to $X^{(0)}$, the full update-label sequence, and the innovations up to time $s_r-1$.}
\end{lemma}

{Lemma~\ref{lem:algoone-regen-row-form} identifies the randomness on which the lower bound rests: each regenerated row carries a fresh Gaussian innovation, and it is these innovations that will bound the quadratic form $u^\top\widetilde V_m u$ away from zero in every direction $u$. What remains is to ensure that there are enough such regenerated rows. The next lemma shows that, with high probability, every tested triple has at least $g_m$ of them among the first $m$ rows of the corresponding $i$-trace. {We record this guarantee as the \emph{regeneration event}}}
\begin{equation}
    {\mathcal R_m(g_m):=
    \left\{\abs{\mathfrak R_{i,A,j}^{(m)}}\geq g_m
    \quad\text{for every }(i,A,j)\in\mathfrak T\right\},}
    \label{eq:algoone-regen-event}
\end{equation}
on which Lemma~\ref{lem:algoone-fixed-direction} below will also rely. We prove the lemma in Appendix~\ref{app:algoone-regen-counts-proof}.
\begin{lemma}\label{lem:algoone-regen-counts}
{For every $m\geq1+2q$, the regeneration event~\eqref{eq:algoone-regen-event} satisfies}
\begin{equation}
    \Pbb\!\left(\mathcal R_m(g_m)^c\right)
    \leq|\mathfrak T|\exp\left\{-\frac{m-1}{8q}\right\}.
    \label{eq:algoone-regen-count-prob}
\end{equation}
\end{lemma}

{We now use the regenerated rows to bound the quadratic form $u^\top\widetilde V_m u$ from below, one direction $u$ at a time. Fix $(i,A)$ and a unit vector $u\in\mathbb S^{2d-1}$, and let $\mathfrak R_{i,A}^{(m)}:=\bigcup_{j\in A}\mathfrak R_{i,A,j}^{(m)}$. The union is disjoint because each row has one preceding $\bar A$-update, whose label determines the unique $j\in A$ for which the row is regenerated. Since all summands in the quadratic form are nonnegative, we may discard the rows outside this union:}
\begin{equation*}
    u^\top\widetilde V_m u
    \geq\frac{1}{m}\sum_{r\in\mathfrak R_{i,A}^{(m)}}\left(u^\top\widetilde Z^{(r)}\right)^2.
\end{equation*}
{By Lemma~\ref{lem:algoone-regen-row-form}, each retained summand is affine in a fresh Gaussian innovation. The following lemma controls the lower tail of the resulting quadratic form for one fixed direction.}
\begin{lemma}\label{lem:algoone-fixed-direction}
{Let $m\geq1+2q$ and fix $i\in[p]$ and $A\subseteq[p]\setminus\{i\}$ with $|A|=2d$. Write $I^{(1:\infty)}:=\bigl(I^{(t)}\bigr)_{t\ge1}$ for the full update-label sequence. On the event $\mathcal R_m(g_m)$, for every fixed unit vector $u\in\mathbb S^{2d-1}$,}
\begin{equation}
    \Pbb\!\left(
        u^\top \widetilde V_m u < \frac{c_0 g_m}{2m}
        \;\middle|\;
        I^{(1:\infty)}
    \right)\leq e^{-c_0 g_m}.
    \label{eq:algoone-fixed-direction}
\end{equation}
\end{lemma}

{Lemma~\ref{lem:algoone-fixed-direction} is of course not yet a lower eigenvalue bound, since it controls a single direction at a time. We convert this into our desired bound by using an appropriate $\epsilon$--net of $\mathbb S^{2d-1}$. Notice that if we applied the argument of Section~\ref{sec:upper-design-control}, with failure probability $\delta/3$, we would get a bound of $\lambda_m^+(\delta/3)$ on the operator norm of $\widetilde V_m$. And, on the event that this bound holds, the quadratic form changes slowly between nearby directions. That is, }
\begin{equation*}
    \abs{u^\top\widetilde V_m u-v^\top\widetilde V_m v}
    \leq 2\epsilon\lambda_m^+(\delta/3)
\end{equation*}
for unit vectors $u$ and $v$ with $\norm{u-v}_2\leq\epsilon$. This allows us to now apply Lemma~\ref{lem:algoone-fixed-direction} to every point of the net, invoke Lemma~\ref{lem:algoone-regen-counts} to guarantee the required number of regenerated rows, and take the corresponding union bounds. This gives $\widetilde V_m\succeq\lambda_m^-I_A$ on the same event, simultaneously over the tested pairs. Combining the upper and lower bounds proves Proposition~\ref{prop:algoone-design-control}. We refer the reader to Appendix~\ref{app:algoone-design-control-proof} for the full proof.

\subsection{{Separation of the \algoone{} Test Statistic}}\label{sec:algoone-separation}

{We are now in a position to state the main result of this section: with high probability, the \algoone{} statistic $\widehat{\tssym}$ satisfies the separation property of Definition~\ref{def:generic-separation} at an explicit per-node budget.}

\begin{proposition}[Separation of the \algoone{} statistic]\label{prop:algoone-separation}
Fix $\delta\in(0,1)$. There is a universal constant $C\geq1$ such that, if
\begin{equation}
    m
    \geq
    \frac{C q}{\kappa^{2}}
    \left\{
        2d\log\!\bigl(Cq\lambda_m^{+}(\delta)\bigr)
        +\log\frac{|\mathfrak T|}{\delta}
    \right\},
    \label{eq:algoone-separation-budget}
\end{equation}
then, with probability at least $1-\delta$, the statistic $\widehat{\tssym}$ satisfies the separation property of Definition~\ref{def:generic-separation} with threshold $\tau:=\frac{\kappa}{2}\sqrt{\lambda_m^{-}}$: for every $(i,A,j)\in\mathfrak T$ with $S_i\subseteq A$,
\begin{equation}
    \widehat{\tssym}_{i,A}(j)\leq\tau
    \quad\text{if }j\notin S_i,
    \qquad
    \widehat{\tssym}_{i,A}(j)>\tau
    \quad\text{if }j\in S_i.
    \label{eq:algoone-separation-conclusion}
\end{equation}
\end{proposition}

\begin{remark}[Per-node sample complexity]\label{rmk:algoone-per-node-budget}
The quantity $m$ in the above proposition is a per-node data requirement for the test statistic to have the separation property. Since $q=2d+1$, $\log|\mathfrak T|=O(d\log p)$, and $\lambda_m^{+}$ enters inside the logarithm, we get a per-node sample complexity as
\begin{equation*}
    m = O\!\left( \frac{d^{2}}{\kappa^{2}} \log\frac{p \lambda_m^{+}}{\delta} \right).
\end{equation*}
\end{remark}

We prove the proposition in Appendix~\ref{app:algoone-separation-proof}, and sketch the argument here. {The argument proceeds in three stages. First, we apply} Proposition~\ref{prop:algoone-design-control} with failure probability $\delta/3$. This gives us simultaneous upper and lower spectral bounds for the rescaled empirical Gram matrix $\widetilde V_m$, with deterministic constants $\lambda_m^{-}$ and $\lambda_m^{+}$, for every tested pair. {Next, we apply} Lemma~\ref{lem:algoone-partial-oracle-error} with failure probability $\delta/(3|\mathfrak T|)$ to each tested pair. Under the budget~\eqref{eq:algoone-separation-budget}, this controls the difference between the partial-oracle estimate $\widehat{\Psi}^{\ast}_{i,A}$ and the partial-oracle score $\Psi^{\ast}_{i,A}$ by $\frac{\kappa}{8}\sqrt{\lambda_m^{-}}$, uniformly over $\mathfrak T$. {Finally, we translate these bounds into guarantees on the test statistic itself. This proceeds via two further lemmas.} The statistic $\widehat{\tssym}_{i,A}$ uses the estimated residual standard deviation $\widehat\sigma_{i,A}$, whereas $\widehat{\Psi}^{\ast}_{i,A}$ uses the true conditional standard deviation $\sigma_i$; Lemma~\ref{lem:algoone-residual} shows that this estimated normalization remains close enough to $\sigma_i$. Lemma~\ref{lem:algoone-signal-conversion} then shows that the partial-oracle score is zero for non-neighbors and is bounded below by the normalized edge strength for true neighbors, up to the factor $\sqrt{\lambda_m^{-}}$. Combining these two lemmas with the preceding bounds yields separation at the threshold $\tau=\frac{\kappa}{2}\sqrt{\lambda_m^{-}}$. {Both lemmas are stated and proved in Appendix~\ref{app:algoone-separation-proof}, where we also assemble all of these ingredients into a proof of Proposition~\ref{prop:algoone-separation}.}

\subsection{{Finite-Sample Guarantee for \algoone{}}}\label{sec:algoone-finite-horizon}

{We are now in a position to state the main result of this section, the finite-sample guarantee for \algoone{}.}
\begin{theorem}[Finite-sample guarantee for \algoone{}]\label{thm:algoone-main}
There is a universal constant $C\geq1$ for which the following holds.  Assume $p\geq2d+1$, fix $\delta\in(0,1)$, and fix a pair $(\Theta,x^{(0)})\in\mathfrak C_{\mathrm{base}}(p,d,\kappa)$. Let $X^{(0)},\ldots,X^{(N)}$ be generated by the random-scan Gaussian Glauber dynamics of Section~\ref{sec:glauber} with $X^{(0)}=x^{(0)}$, degree-constrained spectral radius $\rho_{2d}(\Theta)$ as in Definition~\ref{def:rho2d}, and initialization potential $E_0={\norm{x^{(0)}}_\Theta^2}$. Run Algorithm~\ref{alg:generic-search} with the \algoone{} statistic of Definition~\ref{def:algoone-plugin-test-statistic}, row budget $m:=\lfloor N/(2p)\rfloor$, and threshold $\tau=\frac{\kappa}{2}\sqrt{\lambda_m^{-}}$, and denote the resulting estimator by $\Acal_{\mathrm{LS}}$. If
\begin{equation}
    N \geq \frac{C\,p\,d}{\kappa^{2}}
    \left\{
        d\log\!\left(C\,p\,d\,\rho_{2d}(\Theta)
        \left(E_0+N+\log\frac{2}{\delta}\right)\right)
        +\log\frac{1}{\delta}
    \right\},
    \label{eq:algoone-main-budget}
\end{equation}
then the algorithm recovers the true edge set with probability at least $1-\delta$ under $\Pbb_{\Theta,x^{(0)}}$. Thus, Theorem~\ref{thm:algoone-main} bounds the pointwise recovery horizon $N_{\mathrm{pt}}^\star (\Theta,x^{(0)},\delta;\Acal_{\mathrm{LS}})$ for the \algoone{} estimator.
\end{theorem}

\begin{remark}[Pointwise recovery horizon]\label{rmk:algoone-trajectory-complexity}
The bound in Theorem~\ref{thm:algoone-main} is implicit because $N$ appears on both sides of~\eqref{eq:algoone-main-budget}, but only inside a logarithm. For the fixed pair $(\Theta,x^{(0)})$, the following explicit observation horizon suffices:
\begin{equation*}
    N = O\!\left(
    \frac{p\,d^{2}}{\kappa^{2}}\,
    \log\frac{p\,d\,\rho_{2d}(\Theta)\,(E_0+1)}{\kappa\,\delta}
    \right).
\end{equation*}
\end{remark}

\begin{corollary}[Uniform recovery under bounded $\rho_{2d}$ and potential]\label{cor:algoone-uniform-envelope}
There is a universal constant $C\geq1$ such that, for every $\bar\rho\geq1$, $\bar E\geq0$, and $\delta\in(0,1)$,
\begin{equation}
    N_{\mathrm{unif}}^\star
    \bigl(\mathfrak C(p,d,\kappa;\bar\rho,\bar E),
    \delta;\Acal_{\mathrm{LS}}\bigr)
    \leq
    \frac{C\,p\,d^2}{\kappa^2}
    \log\frac{C\,p\,d\,\bar\rho\,(\bar E+1)}
    {\kappa\,\delta}.
    \label{eq:algoone-uniform-envelope-budget}
\end{equation}
In particular, \algoone{} has a uniform sample-complexity guarantee once the degree-constrained spectral radius and initialization potential are included among the model-class parameters.
\end{corollary}
\begin{proof}
We prove the claim by applying the pointwise horizon uniformly.  Fix $(\Theta,x^{(0)})\in \mathfrak C(p,d,\kappa;\bar\rho,\bar E)$. By the definition in~\eqref{eq:algoone-envelope-class}, $\rho_{2d}(\Theta)\leq\bar\rho$ and $E_0\leq\bar E$. Substituting these two bounds into the explicit horizon in Remark~\ref{rmk:algoone-trajectory-complexity} gives \eqref{eq:algoone-uniform-envelope-budget} with a universal constant that does not depend on $(\Theta,x^{(0)})$. Taking the supremum over the restricted class proves the result.
\end{proof}

The proof of Theorem~\ref{thm:algoone-main}, given in Appendix~\ref{app:algoone-main-proof}, combines Proposition~\ref{prop:algoone-separation} with Lemma~\ref{lem:row-collection}. The theorem allots the statistic the first $m=\lfloor N/(2p)\rfloor$ rows of each $i$-trace, and these rows are contained in the observed data precisely on the good ``row-collection'' event
\begin{equation}
    \mathcal C_N := \bigl\{T_m\leq N\bigr\},
    \label{eq:algoone-clock-event}
\end{equation}
where, as in Section~\ref{sec:upper-design-control}, $T_m=\max_{i\in[p]}t_i^{(m)}$ is the time by which every node has received its $m$th update. Under the budget~\eqref{eq:algoone-main-budget}, Lemma~\ref{lem:row-collection} shows that $\mathcal C_N$ holds with high probability; if it fails, the statistic returns $\undeftok$, as described in Section~\ref{sec:skeleton}. On $\mathcal C_N$, Proposition~\ref{prop:algoone-separation} supplies the separation property, and Proposition~\ref{prop:generic-separation-recovery} converts it into exact recovery.

In the next section, we describe our second instantiation, \algotwo{}.

\section{\algotwo{}}\label{sec:algotwo}

Our second instantiation of \dunes{}, Alternating-Update \dunes{} (\algotwo{}), replaces the regression statistic of Section~\ref{sec:algoone} with a statistic built from a local update pattern. In doing so, it removes the dependence on the degree-constrained spectral radius $\rho_{2d}(\Theta)$ altogether. Here, letting $A=C\cup D$ and $\bar A=A\cup\{i\}$ as usual, we read the updates on the $\bar A$-clock and look for four consecutive updates with labels $i,i,j,i$. As we will see below, when this pattern occurs and $A$ contains $S_i$, the conditional influence of $j$ on $i$ can be isolated, and we obtain a natural one-dimensional regression identity. Update-pattern-based methods of this type have been used for structure learning from Glauber dynamics in Ising models~\citep{bresler2014learning, gaitonde2025better, gaitonde2025bypassing} and in Gaussian graphical models~\citep{TRD25,SWMM26}.

These update-pattern constructions, which do not use the dueling structure of \dunes{}, face a {\em crucial difficulty}. The pattern occurence is informative only when no neighbor of $i$ other than the tested coordinate $j$ is updated in the interim. Since $S_i$ is unknown, one cannot determine directly whether a putative pattern is free of updates to $S_i\setminus\{j\}$. The works cited above address this difficulty by fixing a global window length, estimating how often such patterns occur in a window of that length, and estimating how often the neighborhood updates occur. The latter is then used to correct their test statistic, which dilates the signal and leads to suboptimal sample-complexity scaling. Indeed, the poor sample complexity comes from the fact that the choice of the window length needs to be made carefully to balance a fundamental tension: a longer window is more likely to contain the desired pattern, but it is also more likely to suffer from an intervening neighbor update that contaminates it. 

Our search procedure removes this obstacle by making the candidate set explicit. For a tested set $A$, the observed labels along the $\bar A$-clock reveal exactly when a coordinate in $A\cup\{i\}$ is updated. Thus, we can cycle through the candidate sets without estimating or subtracting contamination, retaining the full signal carried by the pattern. Reading the pattern along the $\bar A$-clock also dissolves the window-length tension: blocks are defined by counts of $\bar A$-ticks rather than by a prescribed number of global updates, so the quiet stretches between ticks may be arbitrarily long in global time. This is what permits the optimal dependence on $\kappa$ in our sample requirement.

In what follows, we will describe the construction of our second test statistic and hence our second procedure \algotwo{}. We will fix $i\in\mathcal V$ and $A\subseteq\mathcal V\setminus\{i\}$ and recall that we let $\bar A =A\cup\{i\}$. We will also write $t^{(l)}:=t_{\bar A}^{(l)}$. We partition the $\bar A$-ticks into consecutive blocks of four.
\begin{definition}[Block partition]
    For each block index $b \geq 1$, the $b$-th block is the ordered four-tuple
    \begin{align}
        \mathcal{T}_{\bar A}^{(b)}:= \bigl( t^{(4(b-1) + 1)},\ t^{(4(b-1) + 2)},\ t^{(4(b-1) + 3)},\ t^{(4(b-1) + 4)} \bigr).
    \end{align}
\end{definition}

For notational ease, we index the four ticks in block $b$ by $t^{(b)}_k:=t^{(4(b-1)+k)}$, $k\in[4]$. The next definition characterizes the signal-carrying blocks.
\begin{definition}[$A$-$iiji$ block]
    Fix $i \in \mathcal{V}, j \in A, i \neq  j$. A block $b$ is an \emph{$A$-$iiji$ block} if its four consecutive update labels along the $\bar A$-clock follow the pattern $i,i,j,i$; that is,
    \begin{equation*}
        \bigl(I^{(t^{(b)}_1)},\ I^{(t^{(b)}_2)},\ I^{(t^{(b)}_3)},\ I^{(t^{(b)}_4)}\bigr) = (i,  i,  j,  i).
    \end{equation*}
\end{definition}

For each $A$-$iiji$ block $b$ we define
\begin{equation}
    x_b := X_j^{(t^{(b)}_3)} - X_j^{(t^{(b)}_2)},
    \qquad
    y_b := X_i^{(t^{(b)}_4)} - X_i^{(t^{(b)}_1)},
    \qquad
    \xi_b := \varepsilon_i^{(t^{(b)}_4)} - \varepsilon_i^{(t^{(b)}_1)}.
    \label{eq:algo2-block-vars}
\end{equation}

Notice that between consecutive ticks of a block, no coordinate of $\bar A$ is updated: every $n$ with $t^{(b)}_l < n < t^{(b)}_{l+1}$ has $I^{(n)} \notin \bar A$. When $A \supseteq S_{i}$, this neighborhood quietness yields the following regression identity.
\begin{lemma}[Regression form on a block]\label{lem:iiji-regression}
Fix $i \in \mathcal{V}, A \subset \mathcal{V}\setminus \{i\}, j \in A$, and suppose $S_i \subseteq A$. Then on every $A$-$iiji$ block $b$,
\begin{equation}
    y_b = \beta_{ij}  x_b + \xi_b,
    \label{eq:algo2-iiji-regression}
\end{equation}
where $\xi _{b} \stackrel{\mathrm{iid}}{\sim} \mathcal{N}(0,2\sigma_{i}^{2})$ and $x_{b} \perp \xi _{b}$.
\end{lemma}

We postpone the simple proof of this lemma to Appendix~\ref{app:iiji-regression-proof}. Under the uniformly random update rule, each $A$-$iiji$ block has positive probability of occurring, and disjoint blocks are independent. Thus, along the infinite trajectory, the first $n_0$ such blocks occur at an almost surely finite random time. Fixing $n_0$, we stack the corresponding values into vectors $\mathbf{X},\mathbf{Y},\boldsymbol{\xi}\in\Rbb^{n_0}$ and obtain
\begin{equation}
    \mathbf{Y} = \beta_{ij}  \mathbf{X} + \boldsymbol{\xi},
    \qquad
    \boldsymbol{\xi} \sim \mathcal{N}\bigl(0,  2\sigma_i^2  I_{n_0}\bigr),
    \qquad
    x_b \perp \xi_b.
    \label{eq:algo2-vecform}
\end{equation}

The vector regression in~\eqref{eq:algo2-vecform} reduces the test of $j\in S_i$ to the hypothesis test $H_0:\beta_{ij}=0$ versus $H_1:\beta_{ij}\neq0$. Under the null, $\mathbf{X}$ has no systematic influence on $\mathbf{Y}$, and under the alternative, the coefficient $\beta_{ij}$ accounts for part of the variation in $\mathbf{Y}$. We therefore compare the residual sum of squares under the null coefficient with the residual sum of squares at the best-fitting coefficient:

\begin{equation}
    \mathrm{RSS}_0 := \norm{\mathbf{Y}}^2, \qquad \mathrm{RSS}_1 := \min_{a \in \Rbb}  \norm{\mathbf{Y} - a \mathbf{X}}^2 = \norm{\mathbf{Y}}^2 - \frac{{(\mathbf{X}^\top \mathbf{Y})}^2}{\norm{\mathbf{X}}^2},
    \label{eq:algo2-rss}
\end{equation}
where the last equality {holds} for $\mathbf{X} \neq 0$. The difference $\mathrm{RSS}_0-\mathrm{RSS}_1$ measures the variation explained by $\mathbf{X}$. Because both residual sums of squares scale with the unknown noise variance $\sigma_i^2$, we normalize this reduction by $\mathrm{RSS}_1$ and use the resulting scale-free quantity as the local statistic.
\begin{definition}[Test statistic]\label{def:algo2-statistic}
    For a tested triple $(i,A,j)$, the local test statistic is the normalized drop in residual sum of squares
    \begin{equation}
        \Phi_{i,A}(j) := \frac{\mathrm{RSS}_0 - \mathrm{RSS}_1}{\mathrm{RSS}_1}.
        \label{eq:algo2-statistic}
    \end{equation}
\end{definition}

\algotwo{} is Algorithm~\ref{alg:generic-search} run with $\tssym=\Phi$, where $\Phi_{i,A}(j)$ is computed from the first $n_0$ $A$-$iiji$ blocks of the tested triple and set to $\undeftok$ when fewer than $n_0$ such blocks complete within the horizon $N$. 

{To invoke the recovery guarantee of Proposition~\ref{prop:generic-separation-recovery}, we must show that the statistic $\Phi$ is available and satisfies the separation property of Definition~\ref{def:generic-separation} on the family of triples that the search queries. The next two subsections establish the two requirements in turn: Section~\ref{sec:separation} determines the block budget $n_0$ for which separation holds, and Section~\ref{sec:algo2-finite-sample-guarantee} determines the observation horizon that makes the statistic available and combines the two bounds over the tested family.}

\subsection{{Separation of the \algotwo{} Test Statistic}}\label{sec:separation}

We now establish the separation property from Definition~\ref{def:generic-separation} for the test statistic in Definition~\ref{def:algo2-statistic}. {Recall from Section~\ref{sec:skeleton} the family $\mathfrak T$ of triples queried by the search; we record it here together with its cardinality:}
\begin{equation}
    \mathfrak T := \bigl\{(i,A,j): i\in\Vcal,\ A\subseteq \Vcal\setminus\{i\},\ |A|=2d,\ j\in A\bigr\}, 
    \text{ with }|\mathfrak T| = p\binom{p-1}{2d}(2d).
    \label{eq:algo2-tested-triples}
\end{equation}
{Now fix} a tested triple $(i,A,j)\in\mathfrak T$, and suppose for now that the first $n_0$ completed $A$-$iiji$ blocks are available. The following proposition gives the separation guarantee and specifies how large $n_0$ must be. Section~\ref{sec:algo2-finite-sample-guarantee} will then show how long the global trajectory must run so that, with high probability, every tested triple has accumulated its first $n_0$ blocks.

\begin{proposition}
\label{prop:algo2-fixed-triple-separation}
    Fix a triple $(i,A,j)$ with $j\in A$ and $S_i\setminus\{j\}\subseteq A$. For every $\delta \in (0,1)$, if
    \begin{align*}
        n_0 \geq \frac{5120}{\kappa^2}\log\frac{3e n_0}{\delta}
    \end{align*}
    and $\tau = \frac{\kappa^{2}}{128}$, then with probability at least $1-\delta$,
    \begin{equation}
        \begin{cases}
            \Phi_{i,A}(j)\leq \tau, & j\notin S_i,\\[1mm]
            \Phi_{i,A}(j)> \tau, & j\in S_i.
        \end{cases}
    \end{equation}
\end{proposition}

{We present a proof sketch here, with details deferred to Appendix~\ref{app:algo2-fixed-triple-separation-proof}. Let us focus on one tested triple, $(i,A,j)\in\mathfrak T$ and suppose that $A\supseteq S_i\setminus\{j\}$. Recall from~\eqref{eq:algo2-vecform} that $\mathbf{X}=(x_1,\ldots,x_{n_0})^\top$ and $\boldsymbol{\xi}=(\xi_1,\ldots,\xi_{n_0})^\top$ collect the covariates and noise contrasts from the $n_0$ completed $A$-$iiji$ blocks. We now consider the null and signal cases separately.}

\underline{\textbf{Case 1: $\beta_{ij} = 0$.}}\quad Here we have {$\mathbf{Y} = \boldsymbol{\xi}$}. This implies that the residual sums of squares in~\eqref{eq:algo2-rss} reduce to
\begin{equation}
    \mathrm{RSS}_0 = \norm{\boldsymbol{\xi}}^2,
    \qquad
    \mathrm{RSS}_0 - \mathrm{RSS}_1 = \frac{{(\mathbf{X}^\top \boldsymbol{\xi})}^2}{\norm{\mathbf{X}}^2},
    \qquad
    \mathrm{RSS}_1 = \norm{\boldsymbol{\xi}}^2 - \frac{{(\mathbf{X}^\top \boldsymbol{\xi})}^2}{\norm{\mathbf{X}}^2},
    \label{eq:algo2-null-rss}
\end{equation}
and the statistic becomes
\begin{equation}
    \Phi_{i,A}(j) = \frac{{(\mathbf{X}^\top \boldsymbol{\xi})}^2 / \norm{\mathbf{X}}^2}{\norm{\boldsymbol{\xi}}^2 - {(\mathbf{X}^\top \boldsymbol{\xi})}^2 / \norm{\mathbf{X}}^2}.
    \label{eq:algo2-null-stat}
\end{equation}

{To show that $\Phi_{i,A}(j) \leq \tau$, it suffices to bound $(\mathbf{X}^\top \boldsymbol{\xi})^2 / \norm{\mathbf{X}}^2$ from above and $\norm{\boldsymbol{\xi}}^2$ from below. Since the entries of $\boldsymbol{\xi}$ are iid with $\xi_b \sim \mathcal{N}(0,2\sigma_i^2)$, we have
\begin{equation*}
    \norm{\boldsymbol{\xi}}^2 = \sum_{b=1}^{n_0}\xi_b^2 \sim 2\sigma_i^2\chi^2_{n_0}.
\end{equation*}
We use a standard chi-squared tail bound (specifically one from~\cite{laurent2000adaptive}) to control the noise energy.  Lemma~\ref{lem:algo2-noise-energy} records the resulting bound.}

The numerator, on the other hand, is a prototypical self-normalized quantity: its denominator measures the dispersion of its numerator. Controlling it is the technical heart of the argument, for two reasons. First, the self-normalization machinery~\cite{de2004self} requires a filtration under which each covariate $x_b$ is predictable and each noise term $\xi_b$ is a fresh innovation. The $i,i,j,i$ construction does not yield such a filtration off the shelf, since earlier block noises propagate through the trajectory into later designs. We construct a suitable base sigma-field in Definition~\ref{def:algo2-base-sigma-field} of Appendix~\ref{app:algo2-conditional-regression}, and Lemma~\ref{lem:algo2-causal-design} there shows that this novel filtration satisfies the required properties. Second, with this filtration in hand, we bound the self-normalized term using a ``dyadic peeling'' argument, carried out in the proof of Lemma~\ref{lem:algo2-snm-bound} in the same appendix. This yields the logarithmic bound below, which suffices for our purposes although it may be possible to improve this. The peeling step is what accommodates the unboundedness of the Gaussian statistic, a difficulty absent in the bounded Ising setting. The analogous difficulty in~\citet{TRD25} is handled by conditioning on a global boundedness event for the entire trajectory, whereas the current argument requires only local control of $\norm{\mathbf{X}}^2$ and $\norm{\boldsymbol{\xi}}$ for the tested triple. {The following lemma records the resulting control on the self-normalized term.}
\begin{lemma}
\label{lem:algo2-snm-bound}
    For every $\delta \in (0,1)$, with probability at least $1 - \delta$,
    \begin{equation}
        \frac{(\mathbf{X}^\top \boldsymbol{\xi})^2}{\norm{\mathbf{X}}^2} \;\leq\; 40  \sigma_i^2 \log\frac{e  n_0}{\delta},
        \label{eq:algo2-snm-bound}
    \end{equation}
\end{lemma}

{Combining Lemmas~\ref{lem:algo2-noise-energy} and~\ref{lem:algo2-snm-bound}, we obtain the first half of the separation argument.}
\begin{lemma}
\label{lem:algo2-null}
    Fix a triple $(i,A,j)$ with $j\in A$ and $S_i\setminus\{j\}\subseteq A$. Suppose $\beta_{ij} = 0$. For every $\delta \in (0,1)$, if $n_0 > 80\log(2e n_0/\delta)$
    then, with probability at least $1-\delta$,
    \begin{equation}
        \Phi_{i,A}(j)
        \leq
        40\frac{\log(2e n_0/\delta)}{n_0}.
        \label{eq:algo2-null-bound}
    \end{equation}
\end{lemma}
{The upper bound in~\eqref{eq:algo2-null-bound} decreases in $n_0$, and it will guide our choice of $n_0$ and the threshold $\tau$. We next lower bound the test statistic when $\beta_{ij} \neq 0$.}

\underline{\textbf{Case 2: $\beta_{ij} \neq 0$.}}\quad The response $\mathbf{Y}$ now carries signal{: $\mathbf{Y} = \beta_{ij} \mathbf{X} + \boldsymbol{\xi}$. {The minimizing coefficient in~\eqref{eq:algo2-rss} therefore decomposes into signal and noise parts, $\widehat a = \mathbf{X}^\top \mathbf{Y} / \norm{\mathbf{X}}^2 = \beta_{ij} + \mathbf{X}^\top \boldsymbol{\xi}/\norm{\mathbf{X}}^2$.}} The residual is exactly as in Case 1,
\begin{equation}
    \mathrm{RSS}_1 = \norm{\boldsymbol{\xi}}^2 - \frac{(\mathbf{X}^\top \boldsymbol{\xi})^2}{\norm{\mathbf{X}}^2}.
\end{equation}
The numerator of the statistic on the other hand (the explained variation), satisfies
\begin{equation}
    \mathrm{RSS}_0 - \mathrm{RSS}_1 = \frac{\bigl(\beta_{ij}\norm{\mathbf{X}}^2 + \mathbf{X}^\top \boldsymbol{\xi}\bigr)^2}{\norm{\mathbf{X}}^2} \;\geq\; \tfrac{1}{2}\beta_{ij}^2 \norm{\mathbf{X}}^2 - \frac{(\mathbf{X}^\top \boldsymbol{\xi})^2}{\norm{\mathbf{X}}^2},
    \label{eq:algo2-signal-drop}
\end{equation}
which follows from the bound $(u + v)^{2} \geq \frac{1}{2}u^{2} - v^{2}$. {To show that $\Phi_{i,A}(j)$ is bounded from below, we lower-bound the explained variation and upper-bound the residual. The noise terms are controlled by Lemmas~\ref{lem:algo2-noise-energy} and~\ref{lem:algo2-snm-bound}, as in Case 1. The one new ingredient is a lower bound on the excitation $\norm{\mathbf{X}}^2$, which the next lemma supplies.}
\begin{lemma}\label{lem:algo2-excitation}
    Fix a triple $(i,A,j)$ with $j \in A$, and write $\sigma_j^2 := \Theta_{jj}^{-1}$. For every $\delta \in (0,1)$, if $n_0\geq 4\log(1/\delta)$, then with probability at least $1 - \delta$,
    \begin{equation}
        \norm{\mathbf{X}}^2 = \sum_{b=1}^{n_0} x_b^2 \;\geq\; \frac{1}{4}\sigma_j^2 n_0
        \qquad\bigl(\text{equivalently } \Theta_{jj}\norm{\mathbf{X}}^2 \geq n_0/4\bigr).
    \end{equation}
\end{lemma}

The proof, given in Appendix~\ref{app:algo2-excitation-proof}, mirrors the regeneration argument of Section~\ref{sec:lower-design-control}: the $j$-update inside each block injects a fresh Gaussian innovation, so each rescaled covariate $\sqrt{\Theta_{jj}}\,x_b$ is a predictable term plus a fresh standard Gaussian, and a noncentral chi-square lower-tail bound yields the claim.

Putting together the excitation bound (Lemma~\ref{lem:algo2-excitation}) with the noise-energy and self-normalized bounds (Lemmas~\ref{lem:algo2-noise-energy} and~\ref{lem:algo2-snm-bound}), we control the test statistic under Case 2. The next lemma formalizes this.

\begin{lemma}\label{lem:algo2-alt}
    Fix a triple $(i,A,j)$ with $j\in A$ and $S_i\setminus\{j\}\subseteq A$. Suppose $\beta_{ij} \neq 0${, and recall the normalized edge strengths $\kappa_{ij}$ and $\kappa=\min_{\{i,j\}\in\Ecal}\kappa_{ij}$ from Section~\ref{sec:ggm}}. For every $\delta \in (0,1)$, if
    \begin{equation}
        n_0
        \geq
        \frac{640}{\kappa^2}\log\frac{3e n_0}{\delta},
    \end{equation}
    then, with probability at least $1 - \delta$,
    \begin{equation}
        \Phi_{i,A}(j) \;\geq\; \frac{\kappa^2}{64}.
        \label{eq:algo2-alt-bound}
    \end{equation}
\end{lemma}

Combining Lemmas~\ref{lem:algo2-null} and~\ref{lem:algo2-alt} yields the separation guarantee of Proposition~\ref{prop:algo2-fixed-triple-separation}.

\subsection{{Finite-Sample Guarantee for \algotwo{}}}\label{sec:algo2-finite-sample-guarantee}

Proposition~\ref{prop:algo2-fixed-triple-separation} controls $\Phi_{i,A}(j)$ once the first $n_0$ $A$-$iiji$ blocks are available. To obtain a finite-horizon guarantee, we must therefore bound the probability that at least $n_0$ such blocks have been completed by a deterministic horizon $N$. We first establish this block-count bound for a fixed $(i,A,j)$, and then combine it with the separation bound over the tested family $\mathfrak T$.

Fix $(i,A,j)\in\mathfrak T$, write $\bar A=A\cup\{i\}$, and set $q:=|\bar A|=|A|+1=2d+1$.
We count completed $A$-$iiji$ blocks along the $\bar A$-clock. For a horizon $N$, we first count the $\bar A$-ticks observed by time $N$ and then the disjoint four-tick windows with labels $(i,i,j,i)$. The following definition introduces these two counts.

\begin{definition}[Finite-horizon block count]\label{def:algo2-finite-horizon-block-count}
For a deterministic horizon $N$, define
\begin{equation}
    K_{\bar A}(N)
    :=
    \max\{\ell\geq0:t_{\bar A}^{(\ell)}\leq N\},
    \label{eq:algo2-retained-tick-count}
\end{equation}
the number of $\bar A$-clock ticks observed by time $N$. The finite-horizon $A$-$iiji$ block count is
\begin{equation}
    N_{(i,A,j)}(N)
    :=
    \sum_{b=1}^{\lfloor K_{\bar A}(N)/4\rfloor}
    \mathds{1}\!\left\{
    \bigl(I^{(t^{(b)}_1)},I^{(t^{(b)}_2)},I^{(t^{(b)}_3)},I^{(t^{(b)}_4)}\bigr)
    =
    (i,i,j,i)
    \right\}.
    \label{eq:algo2-finite-horizon-block-count}
\end{equation}
\end{definition}

Since the update labels are independent and uniform on $[p]$, we have that $K_{\bar A}(N)\sim\operatorname{Bin}(N,q/p)$. Conditional on $K_{\bar A}(N)=m$, the retained labels are independent and uniform on $\bar A$. Therefore, after grouping the retained sequence into disjoint four-label windows, the block count $N_{(i,A,j)}(N)$ has distribution
$\operatorname{Bin}\!\left(\left\lfloor m/4\right\rfloor,q^{-4}\right)$. Our clock estimate is therefore a two-stage lower-tail argument: first we lower-bound the number of retained $\bar A$-updates, and then we lower-bound the number of $i,i,j,i$ windows among them. Both stages are binomial lower-tail estimates, and both are carried out in the proof of the next lemma, given in Appendix~\ref{app:fixed-triple-clock-proof}.

\begin{lemma}\label{lem:algo2-fixed-triple-clock}
Fix a triple $(i,A,j)$ with $j\in A$, and write $q:=|A|+1$. For every $\delta\in(0,1)$ and every $n_0\geq 1$, if
\begin{equation}
    N
    \geq
    128 p q^3\left(n_0+\log\frac{2}{\delta}\right),
    \label{eq:algo2-fixed-triple-clock-condition}
\end{equation}
then
\begin{equation}
    \Pbb\left(
        N_{(i,A,j)}(N)<n_0
    \right)
    \leq
    \delta.
\end{equation}
\end{lemma}

We now assemble the fixed-triple separation and clock estimates into a uniform recovery guarantee over the exhaustively tested family $\mathfrak T$.
\begin{theorem}[Finite-sample recovery]\label{thm:algo2-finite-sample}
Assume $p\geq2d+1$, fix
$(\Theta,x^{(0)})\in\mathfrak C_{\mathrm{base}}(p,d,\kappa)$, let $q:=2d+1$, and
set $\tau:=\kappa^2/128$. For any $\delta\in(0,1)$, define
\begin{equation}
    n_\star
    :=
    \left\lceil
    \frac{10240}{\kappa^2}
    \log\frac{61440 e|\mathfrak T|}{\kappa^2\delta}
    \right\rceil.
    \label{eq:algo2-final-nstar}
\end{equation}
Given the first $N$ updates of the trajectory initialized at $X^{(0)}=x^{(0)}$, form each statistic using the first $n_\star$ completed blocks of the corresponding tested triple. If
\begin{equation}
    N
    \geq
    256\,p\,q^{3}\,n_\star,
    \label{eq:algo2-final-horizon}
\end{equation}

then $\algotwo{}$ recovers the true edge set $\mathcal{E}$ with probability at least $1-\delta$ under $\Pbb_{\Theta,x^{(0)}}$. The bound does not depend on either $\Theta$ or $x^{(0)}$ beyond the base model-class parameters.
\end{theorem}

The proof, given in Appendix~\ref{app:algo2-finite-sample-proof}, simply assembles the above ingredients as hinted at at the start of this section. The choice of $n_\star$ in~\eqref{eq:algo2-final-nstar} allows us to invoke Proposition~\ref{prop:algo2-fixed-triple-separation} for each relevant triple, and the horizon condition~\eqref{eq:algo2-final-horizon} allows us to use Lemma~\ref{lem:algo2-fixed-triple-clock}. Then, we union bound over $\mathfrak T$, with the failure probability split appropriately. This shows that with probability at least $1-\delta$ the statistic is available and satisfies the separation property of Definition~\ref{def:generic-separation} on every queried triple. Proposition~\ref{prop:generic-separation-recovery} then yields exact recovery.

\begin{remark}[Uniform sample complexity]\label{rmk:algo2-budget}
Let $\Acal_{\mathrm{AU}}$ denote \algotwo{} with the inputs in Theorem~\ref{thm:algo2-finite-sample}. Since $q=2d+1$ and $|\mathfrak T|\leq p^{2d+2}$, the block budget is $n_\star=O\bigl(\kappa^{-2}\,d\log(p/\kappa\delta)\bigr)$. The horizon condition~\eqref{eq:algo2-final-horizon} therefore gives
\begin{equation}
    N_{\mathrm{unif}}^\star
    \bigl(\mathfrak C_{\mathrm{base}}(p,d,\kappa),\delta;
    \Acal_{\mathrm{AU}}\bigr)
    =
    O\!\left(\frac{p\,d^{4}}{\kappa^{2}}\,
    \log\frac{p}{\kappa\,\delta}\right).
\end{equation}
\end{remark}

\section{Discussion}\label{sec:discussion}

We have introduced \dunes{}, a dueling-sets template that reduces graph recovery from a single Glauber trajectory to the verification of a single separation property. And, we have instantiated it with two mixing-free statistics. Both instantiations attain the $\kappa^{-2}$ dependence on the signal strength exhibited by the information-theoretic lower bounds~\citep{TRD25, SWMM26}.

The two instantiations trade degree factors against dependence on the precision matrix and the initialization.~\algoone{} essentially uses every update of every node. For
each fixed pair $(\Theta,x^{(0)})$, its pointwise recovery horizon
\begin{align}
    N=O\!\left(\kappa^{-2}pd^{2}
    \log\frac{pd\rho_{2d}(\Theta)(E_0+1)}
    {\kappa\delta}\right)
\end{align}
from Remark~\ref{rmk:algoone-trajectory-complexity} carries the smaller polynomial factors in $d$, at the price of a logarithmic dependence on $\rho_{2d}(\Theta)$ and $E_0$. That price is indeed mild: the horizon stays essentially linear in $p$ for polynomial $\rho_{2d}(\Theta)$, a regime in which any mixing-based horizon is polynomial.~\algotwo{} retains only the parts of the trajectory that contain the useful $i,i,j,i$ blocks. Its uniform sample complexity $N=O\bigl(\kappa^{-2}pd^{4}\log(p/\kappa\delta)\bigr)$ over $\mathfrak C_{\mathrm{base}}(p,d,\kappa)$ (Remark~\ref{rmk:algo2-budget}) pays two extra factors of $d$, but removes the dependence on $\rho_{2d}(\Theta)$ and $E_0$ entirely.

A third axis of comparison is computational. Both our instantiations inherit the search complexity of the DICE skeleton: the dueling-sets search needs $O(p^{2d+1})$ evaluations~\citep{misra2020information}, which is polynomial only when $d$ is constant. The guarantees in this line of work therefore trade off three desiderata: mixing-freeness, optimal dependence on the signal strength, and polynomial running time. Our procedures achieve the first two; the polynomial-time procedure of \citet{SWMM26} achieves the first and the third, with sample complexity scaling as $\kappa^{-5}$. Whether all three can be achieved simultaneously is open, and mirrors the corresponding open problem in the i.i.d.\ setting~\citep{misra2020information, SWMM26}.

Another thing to note is that exact recovery requires $\Omega(\kappa^{-2}\log p)$ updates per node, and our budgets exceed this by factors of $d^{2}$ and $d^{4}$ respectively. Indeed, a part of the gap is inherited from DICE, whose i.i.d.\ guarantee is itself a factor of $d$ off~\citep{misra2020information}, and part comes from the union bound over the tested family. The logarithms also carry an additive $\log(1/\kappa)$ absent from the lower bound, which is immaterial whenever $\kappa\geq p^{-O(1)}$. Whether the factors of $d$ can be reduced, and whether the additive $\log(1/\kappa)$ can be removed, are natural open questions. Two further directions follow from the structure of the analysis. First, both procedures take the parameters $d$ and $\kappa$ as inputs; adapting to unknown parameters from the trajectory itself is not addressed here. Second, both separation arguments rest on a single mechanism: a just-refreshed coordinate carries a fresh Gaussian innovation, so each covariate decomposes into a predictable part and a fresh Gaussian variable (Lemmas~\ref{lem:algoone-regen-row-form} and~\ref{lem:algo2-excitation}). This mechanism is not specific to the Gaussian conditional law, and extending it to other local dynamics~\citep{gaitonde2023unifiedapproachlearningising, gaitonde2025bypassing} is a natural next step.

\bibliography{references}

\clearpage
\appendix
\section*{Appendices}
\section{Proof of Proposition~\ref{prop:generic-separation-recovery}}\label{app:generic-separation-recovery-proof}
We prove the deterministic recovery guarantee for \dunes{} stated in Proposition~\ref{prop:generic-separation-recovery}. For each node $i$, availability and separation give an accepted candidate containing $S_i$, rule out candidates that omit a neighbor, and ensure that cleanup followed by the OR rule recovers the edge set.
\begin{proof}
We first prove the nodewise statement and then aggregate. Since $\tssym$ is available on $\mathfrak T$, Algorithm~\ref{alg:generic-search} does not terminate at the insufficient-data step. Fix a node $i$.

We begin by showing that a candidate containing $S_i$ is accepted. Since $|S_i|\le d$, choose a size-$d$ set $C^\star\subseteq\Vcal\setminus\{i\}$ with $S_i\subseteq C^\star$. For every $D\in\mathcal D(C^\star)$, all elements of $D$ are non-neighbors of $i$, and $S_i\subseteq C^\star\cup D$. Hence $\tsat{i}{C^\star\cup D}{j}\le\tau$ for all $j\in D$, so $C^\star$ is accepted. Thus the search has an accepted candidate containing the full neighborhood.

{We now rule out accepted candidates that omit a neighbor.} Let $C$ be any size-$d$ candidate that does not contain $S_i$. Choose a missing neighbor $k\in S_i\setminus C$, and choose a dueling set $D\in\mathcal D(C)$ containing all nodes in $S_i\setminus C$; this is possible because $|S_i|\le d$. With $A=C\cup D$, we have $S_i\subseteq A$ and $k\in D$. Therefore $\tsat{i}{A}{k}>\tau$, so \eqref{eq:generic-acceptable} fails for this $D$. Hence every accepted candidate, including the selected $\widehat C_i$, contains $S_i$.

It remains to analyze cleanup. The accepted candidate $\widehat C_i$ may contain non-neighbors, so the final thresholding step is needed to recover exactly $S_i$. The chosen set $A=\widehat C_i\cup D$ contains $S_i$. Therefore, for each $j\in\widehat C_i$, separation gives $\tsat{i}{A}{j}>\tau$ exactly when $j\in S_i$, and $\tsat{i}{A}{j}\le\tau$ otherwise. Thus $\widehat S_i=S_i$. Since this holds for every node $i$, the OR rule in~\eqref{eq:generic-Ehat} gives $\widehat\Ecal=\Ecal$.
\end{proof}

\section{Proofs for \algoone{}}
We next establish the LS-DuNeS argument. Our first step is to establish the regression identity and row availability, then control the rescaled empirical Gram matrix $\widetilde V_m$. The remaining results establish two-sided design control, separation, and finite-horizon recovery.
\subsection{Regression identity and row availability}
The first two lemmas establish the data representation used throughout the remaining proofs. The trace identity gives the regression model, while the row-collection bound connects the first $m$ trace updates to a finite horizon of the global dynamics.
\paragraph{Proof of Lemma~\ref{lem:algoone-subtrajectory-regression}}\label{app:algoone-subtrajectory-regression-proof}
\begin{proof}
Fix $r\geq1$ and let $t:=t_i^{(r)}$ for the $r$th tick of the $i$-clock. Since $I^{(t)}=i$ by definition, the Glauber update~\eqref{eq:glauber-additive} gives
\begin{equation*}
    X_i^{(t)} = \sum_{k\in S_i}\beta_{ik}X_k^{(t-1)} + \varepsilon_i^{(t)}, \qquad \varepsilon_i^{(t)}\sim\mathcal N(0,\sigma_i^2).
\end{equation*}
At this update, coordinate $i$ is refreshed while the coordinates in $A$ remain fixed and the unchanged $A$-coordinates are therefore the covariate row used by the trace regression of \algoone{}.

Since $S_i\subseteq A$ and $\beta_{ik}=0$ for $k\notin S_i$, the sum equals $\sum_{k\in A}\beta_{ik}X_k^{(t-1)}=\bigl(Z_A^{(r)}\bigr)^\top\beta_{i,A}$, where we used that coordinate $i$ is the only coordinate updated at time $t$, so $X_A^{(t-1)}=X_A^{(t)}=Z_A^{(r)}$. This is the identity~\eqref{eq:algoone-subtrajectory-regression} with $\varepsilon^{(r)}=\varepsilon_i^{(t_i^{(r)})}$. Finally, the innovations attached to distinct update times are iid and independent of the update labels, and the times $t_i^{(1)}<t_i^{(2)}<\cdots$ are functions of the labels alone; hence $(\varepsilon^{(r)})_{r\geq1}$ are iid $\mathcal N(0,\sigma_i^2)$.
\end{proof}

\paragraph{{Proof of Lemma~\ref{lem:row-collection}}}
\begin{proof}
The row-collection time $T_m$ is the first global time by which every node has appeared at least $m$ times on its own clock. We bound it by examining all node counts at one deterministic time $n$.
For a deterministic time $n$, let
\begin{align}
    N_i(n):=\sum_{t=1}^{n}\mathds{1}\{I^{(t)}=i\}
\end{align}
be the number of updates of node $i$ among the first $n$ global updates. Then $N_i(n)\sim \operatorname{Bin}(n,1/p)$ and $\Ebb N_i(n)=n/p$.  Set $n=t_m(\delta)=\left\lceil\max\left\{2pm,8p\log\frac{p}{\delta}\right\}\right\rceil$.  Since $n\geq 2pm$, we have $m\leq \Ebb N_i(n)/2$.  We use the standard binomial lower-tail bound \citep[Theorem~21]{mitzenmacher2017probability}: if $Z$ is binomial with mean $\mu$, then $\Pbb\{Z\leq \mu/2\}\leq e^{-\mu/8}$. Applying this with $Z=N_i(n)$ and $\mu=n/p$ gives 
\begin{align}
    \Pbb\{N_i(n)<m\} \leq \exp\!\left\{-\frac{n}{8p}\right\}.
\end{align}
This controls the probability that a fixed node has not received its $m$th update by time $n$.  Because $n\geq 8p\log(p/\delta)$, the right-hand side is at most $\delta/p$. Union bounding over $i\in[p]$ yields
\begin{align}
    \Pbb\{T_m>n\} = \Pbb\{\exists i\in[p]:N_i(n)<m\} \leq \delta.
\end{align}
\end{proof}

\subsection{Upper bound on the empirical Gram matrix spectrum}
We next control the largest eigenvalue of the rescaled empirical Gram matrix. Lemmas~\ref{lem:algoone-one-step-drift} and~\ref{lem:algoone-energy-to-covariate} bound the covariates through the Lyapunov energy, and Lemma~\ref{lem:algoone-upper-design-control} transfers this bound to $\widetilde V_m$.
\paragraph{Proof of Lemma~\ref{lem:algoone-one-step-drift}}
\begin{proof}
Since $I^{(t)}=i$, Equations~\eqref{eq:glauber-additive} and~\eqref{eq:glauber-frozen} give
\begin{equation*}
    X^{(t)}-X^{(t-1)} = \left( \sum_{j\in S_i}\beta_{ij}X_j^{(t-1)} + \varepsilon_i^{(t)} - X_i^{(t-1)} \right)e_i .
\end{equation*}
Moreover, using $\beta_{ij}=-\Theta_{ij}/\Theta_{ii}$ and $\beta_{ij}=0$ for $j\notin S_i$,
\begin{equation*}
    \bigl(\Theta X^{(t-1)}\bigr)_i = \Theta_{ii}\left( X_i^{(t-1)} - \sum_{j\in S_i}\beta_{ij}X_j^{(t-1)} \right).
\end{equation*}
Set $\zeta_i^{(t)}:=\sqrt{\Theta_{ii}}\varepsilon_i^{(t)}$. Expanding the quadratic energy after this single-coordinate update yields
\begin{align*}
    \mathfrak{L}(X^{(t)}) - \mathfrak{L}(X^{(t-1)}) & = 2\left( \sum_{j\in S_i}\beta_{ij}X_j^{(t-1)} + \varepsilon_i^{(t)} - X_i^{(t-1)} \right) \bigl(\Theta X^{(t-1)}\bigr)_i \\
    &\qquad + \Theta_{ii} \left( \sum_{j\in S_i}\beta_{ij}X_j^{(t-1)} + \varepsilon_i^{(t)} - X_i^{(t-1)} \right)^2 \\
    &\quad = -\Theta_{ii}\left( X_i^{(t-1)} - \sum_{j\in S_i}\beta_{ij}X_j^{(t-1)} \right)^2 + {\left(\zeta_i^{(t)}\right)}^2.
\end{align*}
The inequality follows by dropping the nonpositive term. Finally, $\varepsilon_i^{(t)}\sim\mathcal N(0,\Theta_{ii}^{-1})$ by~\eqref{eq:glauber-additive}, so $\zeta_i^{(t)}\sim\mathcal N(0,1)$. The claim that these are independent and identically distributed follows directly from the assumptions on the innovations at distinct Glauber updates.
\end{proof}

\paragraph{Proof of Lemma~\ref{lem:algoone-energy-to-covariate}}\label{app:algoone-energy-to-covariate-proof}
\begin{proof}
Fix $A$ with $|A|=2d$ and $x\in\Rbb^p$, and set $y:=D_A^{1/2}x_A$. The
first step is a Rayleigh-quotient comparison between $D_A$ and
$\Sigma_{AA}^{-1}$. Since
\begin{align*}
    x_A^\top \Sigma_{AA}^{-1}x_A = y^\top {\left(D_A^{1/2}\Sigma_{AA}D_A^{1/2}\right)}^{-1} y,
\end{align*}
we have
\begin{align*}
    \norm{D_A^{1/2}x_A}_2^2 = y^\top y \leq \lambda_{\max}\!\left(D_A^{1/2}\Sigma_{AA}D_A^{1/2}\right) y^\top{\left(D_A^{1/2}\Sigma_{AA}D_A^{1/2}\right)}^{-1}y \leq \rho_{2d}(\Theta)\, x_A^\top\Sigma_{AA}^{-1}x_A ,
\end{align*}
using Definition~\ref{def:rho2d}. The second step is the Schur-complement variational identity: after ordering coordinates as $A,A^{c}$, the block-inverse formula for positive definite matrices~\citep{Horn_Johnson_2012_Matrix} gives
\begin{align*}
    x_A^\top\Sigma_{AA}^{-1}x_A
    =
    \min_{x_{A^{c}}}
    x^\top\Theta x
    \leq
    {\mathfrak{L}(x)} .
\end{align*}
Combining the two displays proves~\eqref{eq:algoone-energy-to-covariate}.
\end{proof}

We first state the upper spectral bound invoked above; here $\preceq$ denotes the Loewner order on positive semidefinite matrices, and $\lambda_m^+(\delta)$ is the quantity defined in~\eqref{eq:algoone-lambda-plus}.
\begin{lemma}\label{lem:algoone-upper-design-control}
Fix $m\geq1$ and $\delta\in(0,1)$. With probability at least $1-\delta$, simultaneously for every $i\in[p]$ and every $A\subseteq[p]\setminus\{i\}$ with $|A|=2d$, 
\begin{equation*}
 \widetilde V_m=D_A^{1/2}V_mD_A^{1/2}
    \preceq\lambda_m^+(\delta)I_A.
\end{equation*}
\end{lemma}

\paragraph{Proof of Lemma~\ref{lem:algoone-upper-design-control}}\label{app:algoone-upper-design-control-proof}
\begin{proof}
Let $n:=t_m(\delta/2)$. By Lemma~\ref{lem:row-collection} applied at failure probability $\delta/2$, the clock event $\{T_m\leq n\}$ holds with probability at least $1-\delta/2$. By the Laurent--Massart upper-tail bound~\citep{laurent2000adaptive} applied to the chi-square sum below, with probability at least $1-\delta/2$,
\begin{align}
    \sum_{t=1}^{n}\bigl(\zeta_{I^{(t)}}^{(t)}\bigr)^{2}
    \leq
    n+2\sqrt{n\log\frac{2}{\delta}}
    +2\log\frac{2}{\delta}.
\end{align}
On the intersection of the two events, since $T_m\leq n$ the random window is contained in the deterministic one, and summing Lemma~\ref{lem:algoone-one-step-drift} over the deterministic window gives
\begin{align}
    \max_{0\leq s\leq T_m}{\mathfrak{L}(X^{(s)})}
    \leq
    \max_{0\leq s\leq n}{\mathfrak{L}(X^{(s)})}
    \leq
    E_0+n+2\sqrt{n\log\frac{2}{\delta}}+2\log\frac{2}{\delta}
    =
    B_m(\delta).
    \label{eq:algoone-stopped-envelope}
\end{align}
Every covariate row used in the first $m$ updates of any node occurs no later than $T_m$. Hence, for any tested pair $(i,A)$ and any $r\in[m]$, writing $s\leq T_m$ for the global time of the $r$th row, Lemma~\ref{lem:algoone-energy-to-covariate} and~\eqref{eq:algoone-stopped-envelope} give
\begin{align*}
    \norm{\widetilde Z^{(r)}}_2^2
    =
    \norm{D_A^{1/2}X_A^{(s)}}_2^2
    \leq
    \rho_{2d}(\Theta){\mathfrak{L}(X^{(s)})}
    \leq
    \rho_{2d}(\Theta)\,B_m(\delta)
    \leq
    \lambda_m^{+}(\delta).
\end{align*}
The reduction~\eqref{eq:algoone-upper-reduction} then yields $u^\top\widetilde V_m u\leq\lambda_m^{+}(\delta)$ for every unit vector $u$, that is, $\widetilde V_m\preceq\lambda_m^{+}(\delta)I_A$, simultaneously over all $i$ and $A$ on the same event.
\end{proof}

\subsection{Lower bounds on the empirical Gram matrix spectrum}
The previous section establishes an upper bound on the spectrum of the empirical Gram matrix. Our lower bound follows by first establishing that one can use only the rows whose covariates have been specifically regenerated so that we have ``enough randomness''. Lemmas~\ref{lem:algoone-regen-row-form} and~\ref{lem:algoone-regen-counts} identify and count these rows, while Lemma~\ref{lem:algoone-fixed-direction} establishes the resulting excitation in a fixed direction. Proposition~\ref{prop:algoone-design-control} combines this lower argument with the upper bound.
\paragraph{Proof of Lemma~\ref{lem:algoone-regen-row-form}}\label{app:algoone-regen-row-form-proof}
\begin{proof}
Fix $(i,A,j)$, $r$, $\ell_r$, and $s_r$ as in the statement, and suppose
$I^{(s_r)}=j$. Since $s_r$ is the $\bar A$-clock tick immediately
preceding $t_i^{(r)}$, no coordinate in $\bar A$ is updated during the
open interval $(s_r,t_i^{(r)})$. Since the terminal update at
$t_i^{(r)}$ is an update of $i\notin A$, the $A$-coordinates of the
covariate row satisfy
\begin{equation*}
    Z_A^{(r)}
    =
    X_A^{(t_i^{(r)})}
    =
    X_A^{(s_r)} .
\end{equation*}
For every $k\in A\setminus\{j\}$, the coordinate $k$ is not updated at
time $s_r$, so $X_k^{(s_r)}=X_k^{(s_r-1)}$. The $j$th coordinate is
updated at time $s_r$, and the Glauber update gives
\begin{equation*}
    X_j^{(s_r)}
    =
    \sum_{\ell\in S_j}\beta_{j\ell}X_\ell^{(s_r-1)}
    +
    \varepsilon_j^{(s_r)} .
\end{equation*}
After multiplying by $D_A^{1/2}$, this yields
\begin{equation*}
    \widetilde Z^{(r)}
    =
    a^{(r)}+e_j\zeta_j^{(s_r)},
    \qquad
    \zeta_j^{(s_r)}:=\sqrt{\Theta_{jj}}\varepsilon_j^{(s_r)},
\end{equation*}
where
\begin{equation*}
    a_k^{(r)}
    =
    \sqrt{\Theta_{kk}}X_k^{(s_r-1)}
    \quad (k\in A\setminus\{j\}),
    \qquad
    a_j^{(r)}
    =
    \sqrt{\Theta_{jj}}
    \sum_{\ell\in S_j}\beta_{j\ell}X_\ell^{(s_r-1)} .
\end{equation*}
The vector $a^{(r)}$ is measurable with respect to the history before
$s_r$ together with the update-label sequence. Finally,
$\varepsilon_j^{(s_r)}\sim\mathcal N(0,\Theta_{jj}^{-1})$, so
$\zeta_j^{(s_r)}\sim\mathcal N(0,1)$.
\end{proof}

\paragraph{Proof of Lemma~\ref{lem:algoone-regen-counts}}
\label{app:algoone-regen-counts-proof}
\begin{proof}
Fix a triple $(i,A,j)\in\mathfrak T$ and write $\bar A=A\cup\{i\}$, so $|\bar A|=q$. Along the $\bar A$-clock, the retained labels $(I^{(t_{\bar A}^{(\ell)})})_{\ell\geq1}$ are iid uniform on $\bar A$.  For each $r=2,\ldots,m$, consider the portion of this retained label sequence starting just after the $(r-1)$th occurrence of label $i$ and ending at the $r$th occurrence of label $i$. These portions are independent, and the $r$th row is $(A,j)$-regenerated exactly when the last retained label before this terminal $i$ is $j$, equivalently when $I^{(s_r)}=j$. Hence, with
\begin{equation*}
    B_r:=\mathds{1}\{I^{(s_r)}=j\},
\end{equation*}
we have
\begin{equation*}
    \abs{\mathfrak R_{i,A,j}^{(m)}}=\sum_{r=2}^{m} B_r .
\end{equation*}
The indicators $B_2,\ldots,B_m$ are independent because they depend on disjoint portions of the iid retained-label sequence. For a fixed $r$, the event $B_r=1$ occurs if, after the previous occurrence of label $i$, the retained sequence contains some number $h\geq0$ of labels in $\bar A\setminus\{i\}$, followed by label $j$, and then by the next label $i$. Therefore
\begin{equation*}
    \Pbb(B_r=1)
    =
    \sum_{h=0}^{\infty}
    \left(\frac{q-1}{q}\right)^h
    \frac{1}{q}\frac{1}{q}
    =
    \frac{1}{q}.
\end{equation*}
Thus
\begin{equation*}
    \abs{\mathfrak R_{i,A,j}^{(m)}}
    \sim
    \operatorname{Bin}\!\left(m-1,\frac{1}{q}\right),
    \qquad
    \mu:=\Ebb\abs{\mathfrak R_{i,A,j}^{(m)}}=\frac{m-1}{q}.
\end{equation*}
Since $g_m\leq \mu/2$, the Chernoff lower-tail bound gives
\begin{equation*}
    \Pbb\!\left(\abs{\mathfrak R_{i,A,j}^{(m)}}<g_m\right)
    \leq
    \Pbb\!\left(\abs{\mathfrak R_{i,A,j}^{(m)}}<\frac{\mu}{2}\right)
    \leq
    \exp\left\{-\frac{\mu}{8}\right\}
    =
    \exp\left\{-\frac{m-1}{8q}\right\}.
\end{equation*}
A union bound over the tested triples in $\mathfrak T$ gives~\eqref{eq:algoone-regen-count-prob}. The hypothesis $m\geq1+2q$ gives $g_m\geq1$, so $\mathcal R_m(g_m)$ is a valid regeneration event as defined in~\eqref{eq:algoone-regen-event}.
\end{proof}

\paragraph{Proof of Lemma~\ref{lem:algoone-fixed-direction}}
\label{app:algoone-fixed-direction-proof}
\begin{proof}
Throughout, write $g:=g_m$, and condition on the update-label sequence;
that is, work under $\Pbb(\cdot\mid I^{(1:\infty)})$ for a fixed label
realization in $\mathcal R_m(g_m)$, so that
$\abs{\mathfrak R_{i,A,j}^{(m)}}\geq g$ for every $j\in A$. The
regenerated-row sets $\mathfrak R_{i,A,j}^{(m)}$ are
functions of the labels alone, and under this conditioning the Gaussian
innovations remain independent.

Write $\bar A:=A\cup\{i\}$. For each row index $r=2,\ldots,m$, let
$s_r$ denote the $\bar A$-clock update immediately preceding
$t_i^{(r)}$.
Write $\mathfrak R:=\bigcup_{j\in A}\mathfrak R_{i,A,j}^{(m)}$. The union is
disjoint, since the regenerating update of a row is the single reduced
update immediately preceding it, whose label picks out one $j\in A$. For
each $r\in\mathfrak R$, write
\begin{equation*}
    j_r:=I^{(s_r)}
\end{equation*}
for this unique regenerating coordinate, so
$r\in\mathfrak R_{i,A,j_r}^{(m)}$. Since every summand below is nonnegative,
\begin{equation}
    u^\top\bigl(m\widetilde V_m\bigr)u
    =\sum_{r=1}^{m}\bigl(u^\top\widetilde Z^{(r)}\bigr)^2
    \geq
    \sum_{r\in\mathfrak R}\bigl(u^\top\widetilde Z^{(r)}\bigr)^2 .
    \label{eq:algoone-drop-nonregen}
\end{equation}

\emph{Step 1: one fresh Gaussian per regenerated row.}
By Lemma~\ref{lem:algoone-regen-row-form}, for every $r\in\mathfrak R$,
\begin{equation}
    u^\top\widetilde Z^{(r)}
    =
    u^\top a^{(r)}+u_{j_r} \zeta_{j_r}^{(s_r)},
    \label{eq:algoone-regen-affine-direction}
\end{equation}
where $\zeta_{j_r}^{(s_r)}\sim\mathcal N(0,1)$ is the innovation injected
at the regenerating time and $u^\top a^{(r)}$ is measurable with respect
to the history before $s_r$ together with the labels.

\emph{Step 2: moment-generating-function bound by iterated conditioning.}
For $a\in\mathbb R$, $b\in\mathbb R$, and $G\sim\mathcal N(0,1)$, a direct
Gaussian computation gives
\begin{equation}
    \Ebb  e^{-(a+bG)^2}
    =\frac{1}{\sqrt{1+2b^2}}
    \exp\left\{-\frac{a^2}{1+2b^2}\right\}
    \leq
    \bigl(1+2b^2\bigr)^{-1/2}.
    \label{eq:algoone-gaussian-mgf-identity}
\end{equation}
The regenerating times $s_r$ are strictly increasing in the row index, and for
any $r\in\mathfrak R$, every earlier regenerated row value
$\widetilde Z^{(r'')}$, $r''<r$, is measurable with respect to the history
before $s_r$, because its row time satisfies $t_i^{(r'')}<s_r$. Peeling
the rows of $\mathfrak R$ from the latest regenerating time backwards and
applying~\eqref{eq:algoone-gaussian-mgf-identity} with $b=u_{j_r}$ at each
step,
\begin{equation}
    \Ebb\left[
        \exp\left\{
            -\sum_{r\in\mathfrak R}\bigl(u^\top\widetilde Z^{(r)}\bigr)^2
        \right\}
         \middle| I^{(1:\infty)}
    \right]
    \leq
    \prod_{r\in\mathfrak R}\bigl(1+2u_{j_r}^2\bigr)^{-1/2}
    \leq
    \prod_{j\in A}\bigl(1+2u_j^2\bigr)^{-g/2},
    \label{eq:algoone-mgf-product}
\end{equation}
where the last inequality keeps, for each $j\in A$, exactly $g$ of the at
least $g$ available factors and drops the rest, each of which is at most
one.

\emph{Step 3: uniform lower bound on the exponent.}
The map $v\mapsto\log(1+2v)$ is concave on $[0,1]$ and takes the values $0$
at $v=0$ and $\log3$ at $v=1$, hence lies above its chord:
$\log(1+2v)\geq v\log3$ for all $v\in[0,1]$. Applying this with $v=u_j^2$
and summing over $j\in A$,
\begin{equation}
    \sum_{j\in A}\log\bigl(1+2u_j^2\bigr)
    \geq
    \log3 \sum_{j\in A}u_j^2
    =
    \log3 ,
\end{equation}
so that, combining with~\eqref{eq:algoone-drop-nonregen}
and~\eqref{eq:algoone-mgf-product},
\begin{equation}
    \Ebb\left[
        e^{-u^\top(m\widetilde V_m)u}
         \middle| I^{(1:\infty)}
    \right]
    \leq
    \exp\left\{-\frac{g}{2}\log3\right\}
    =
    e^{-2c_0 g}.
\end{equation}

\emph{Step 4: Chernoff bound for the lower tail.}
By Markov's inequality,
\begin{equation}
    \Pbb\!\left(
        u^\top\bigl(m\widetilde V_m\bigr)u<\frac{c_0 g}{2}
         \middle| I^{(1:\infty)}
    \right)
    =
    \Pbb\!\left(
        e^{-u^\top(m\widetilde V_m)u}>e^{-c_0 g/2}
         \middle| I^{(1:\infty)}
    \right)
    \leq
    e^{c_0 g/2}  e^{-2c_0 g}
    \leq
    e^{-c_0 g},
\end{equation}
which is~\eqref{eq:algoone-fixed-direction}.
\end{proof}

\paragraph{Proof of Proposition~\ref{prop:algoone-design-control}}
\label{app:algoone-design-control-proof}
\begin{proof}
Write $\lambda^{+}:=\lambda_m^{+}(\delta/3)$, and define the events
\begin{equation*}
    \mathcal E^{+}
    :=\bigl\{\widetilde V_m\preceq\lambda^{+} I_A
    \text{ for every tested }(i,A)\bigr\},
    \qquad
    \mathcal E^{-}
    :=\bigl\{\widetilde V_m\succeq\lambda_m^{-} I_A
    \text{ for every tested }(i,A)\bigr\},
\end{equation*}
where a tested pair $(i,A)$ ranges over $i\in[p]$ and
$A\subseteq[p]\setminus\{i\}$ with $|A|=2d$. Peeling off the regeneration
and upper-design events,
\begin{equation}
    \Pbb\bigl((\mathcal E^{+}\cap\mathcal E^{-})^c\bigr)
    \leq
    \Pbb\bigl(\mathcal R_m(g_m)^c\bigr)
    +\Pbb\bigl((\mathcal E^{+})^c\bigr)
    +\Pbb\bigl(\mathcal R_m(g_m)\cap\mathcal E^{+}\cap(\mathcal E^{-})^c\bigr).
    \label{eq:algoone-design-peeling}
\end{equation}

For the first two terms, since $\lambda^{+}\geq1$, $q\geq3$, and $|\mathfrak T|\geq2$, the
budget~\eqref{eq:algoone-design-budget} with a large enough universal
constant $C$ implies both $m\geq1+4q$ and
$m\geq1+8q\log(3|\mathfrak T|/\delta)$. The latter makes the right-hand side
of~\eqref{eq:algoone-regen-count-prob} at most $\delta/3$, so
$\Pbb(\mathcal R_m(g_m)^c)\leq\delta/3$ by
Lemma~\ref{lem:algoone-regen-counts}, and
$\Pbb((\mathcal E^{+})^c)\leq\delta/3$ by
Lemma~\ref{lem:algoone-upper-design-control} applied at failure probability
$\delta/3$.

For the third term, fix a tested pair $(i,A)$ and set
\begin{equation*}
    \epsilon:=\frac{c_0 g_m}{8m\lambda^{+}}\leq1 .
\end{equation*}
Let $\mathcal N_\epsilon$ be an $\epsilon$-net of the unit sphere of
$\mathbb R^{A}$ with
$\abs{\mathcal N_\epsilon}\leq(3/\epsilon)^{2d}$
\citep[Corollary~4.2.13]{vershynin2018high}. On $\mathcal E^{+}$, for unit
vectors $u,u'$ with $\norm{u-u'}_2\leq\epsilon$,
\begin{equation*}
    \abs{u^\top\widetilde V_m u-u'^\top\widetilde V_m u'}
    \leq
    2\norm{u-u'}_2 \opnorm{\widetilde V_m}
    \leq
    2\epsilon\lambda^{+}
    =
    \frac{c_0 g_m}{4m},
\end{equation*}
so if $u^\top\widetilde V_m u\geq c_0 g_m/(2m)$ for every
$u\in\mathcal N_\epsilon$, then
$\lambda_{\min}(\widetilde V_m)\geq c_0g_m/(2m)-c_0g_m/(4m)=\lambda_m^{-}$.
Consequently, on the label-measurable event $\mathcal R_m(g_m)$,
Lemma~\ref{lem:algoone-fixed-direction} and a union bound over
the net give
\begin{equation}
    \Pbb\!\left(
        \mathcal E^{+}\cap
        \bigl\{\lambda_{\min}(\widetilde V_m)<\lambda_m^{-}\bigr\}
         \middle| I^{(1:\infty)}
    \right)
    \leq
    \Pbb\!\left(
        \exists u\in\mathcal N_\epsilon:
        u^\top\widetilde V_m u<\frac{c_0 g_m}{2m}
         \middle| I^{(1:\infty)}
    \right)
    \leq
    \left(\frac{24m\lambda^{+}}{c_0 g_m}\right)^{2d}
    e^{-c_0 g_m}.
    \label{eq:algoone-net-conditional}
\end{equation}

Since $m\geq1+4q$, we have $(m-1)/(2q)\geq2$, hence
\begin{equation*}
    g_m=\left\lfloor\frac{m-1}{2q}\right\rfloor
    \geq\frac{m-1}{4q}
    \geq\frac{m}{8q},
\end{equation*}
so that, using $c_0=\tfrac14\log3$,
\begin{equation*}
    \frac{24m\lambda^{+}}{c_0 g_m}
    \leq\frac{768 q\lambda^{+}}{\log3}
    \leq 700 q\lambda^{+},
    \qquad
    c_0 g_m\geq\frac{m\log3}{32q}\geq\frac{m}{32q} .
\end{equation*}
Therefore, whenever
$m\geq32q\bigl\{2d\log(700q\lambda^{+})+\log(3|\mathfrak T|/\delta)\bigr\}$,
which the budget~\eqref{eq:algoone-design-budget} implies for $C$ large
enough, the right-hand side of~\eqref{eq:algoone-net-conditional} is at
most $\delta/(3|\mathfrak T|)$. Integrating~\eqref{eq:algoone-net-conditional} over
the label realizations in $\mathcal R_m(g_m)$ and union bounding over the
tested pairs, of which there are at most $|\mathfrak T|$,
\begin{equation*}
    \Pbb\bigl(\mathcal R_m(g_m)\cap\mathcal E^{+}\cap(\mathcal E^{-})^c\bigr)
    \leq\frac{\delta}{3} .
\end{equation*}

Combining the three bounds in~\eqref{eq:algoone-design-peeling} gives
$\Pbb((\mathcal E^{+}\cap\mathcal E^{-})^c)\leq\delta$, which
is~\eqref{eq:algoone-design-control}. Finally, $m\geq1+4q$ and
$g_m\geq m/(8q)$ give
\begin{equation*}
    \lambda_m^{-}
    =\frac{c_0}{4}\cdot\frac{g_m}{m}
    \geq\frac{\log3}{16}\cdot\frac{1}{8q}
    =\frac{\log3}{128 q} . \qedhere
\end{equation*}
\end{proof}

\subsection{Separation}
With two-sided control of $\widetilde V_m$, Lemmas~\ref{lem:algoone-partial-oracle-error} and~\ref{lem:algoone-residual} control the partial-oracle error and residual-variance normalization. Proposition~\ref{prop:algoone-separation} combines these estimates with the signal-conversion bound for neighbors and non-neighbors.
\paragraph{Proof of Lemma~\ref{lem:algoone-partial-oracle-error}}
\label{app:algoone-partial-oracle-error-proof}
\begin{proof}
Write $\widetilde{\mathbf Z}\in\Rbb^{m\times|A|}$ for the matrix with
rows $(\widetilde Z^{(r)})^\top=(D_A^{1/2}Z_A^{(r)})^\top$ and set
$\widetilde S_m:=\frac1m\widetilde{\mathbf Z}^\top\boldsymbol\varepsilon
=D_A^{1/2}S_m$. Since
$\widetilde V_m^{-1}=D_A^{-1/2}V_m^{-1}D_A^{-1/2}$,
\begin{equation}
    S_m^\top V_m^{-1}S_m
    =
    \widetilde S_m^\top\widetilde V_m^{-1}\widetilde S_m ,
    \label{eq:algoone-scale-invariance}
\end{equation}
so it suffices to bound the normalized quantity.

To apply the self-normalized inequality of~\citet[Theorem~1]{abbasi2011improved}, let $\mathscr F_r$ be the sigma algebra generated by the update labels through $t_i^{(r)}$ and the innovations through $t_i^{(r)}-1$, and set $F_s:=\mathscr F_{s+1}$. Then $X_s:=\widetilde Z^{(s)}$ is $F_{s-1}$-measurable, while $\eta_s:=\varepsilon^{(s)}$ is $F_s$-measurable and conditionally $\mathcal N(0,\sigma_i^2)$ given $F_{s-1}$. We choose the deterministic regularizer $V:=m\lambda_- I_A$ and set
\begin{equation*}
    N_m
    :=
    m\lambda_- I_A
    +\sum_{r=1}^m\widetilde Z^{(r)}\bigl(\widetilde Z^{(r)}\bigr)^\top
    =
    m\bigl(\lambda_- I_A+\widetilde V_m\bigr).
\end{equation*}
The cited inequality, evaluated at time $m$, gives with probability at least $1-\delta$
\begin{equation}
    (m\widetilde S_m)^\top N_m^{-1}(m\widetilde S_m)
    \leq
    2\sigma_i^2
    \log\left(
        \frac{\det(N_m)^{1/2}\det(m\lambda_- I_A)^{-1/2}}{\delta}
    \right).
    \label{eq:algoone-abbasi-applied}
\end{equation}
On the event
$\{\lambda_- I_A\preceq\widetilde V_m\preceq\lambda_+ I_A\}$ we have
$N_m\preceq2m\widetilde V_m$, hence
$(m\widetilde V_m)^{-1}\preceq2N_m^{-1}$ and
\begin{equation*}
    m \widetilde S_m^\top\widetilde V_m^{-1}\widetilde S_m
    =
    (m\widetilde S_m)^\top(m\widetilde V_m)^{-1}(m\widetilde S_m)
    \leq
    2 (m\widetilde S_m)^\top N_m^{-1}(m\widetilde S_m);
\end{equation*}
moreover $N_m\preceq2m\lambda_+ I_A$, so
\begin{equation*}
    \log\left(
        \frac{\det(N_m)^{1/2}}{\det(m\lambda_- I_A)^{1/2}}
    \right)
    \leq
    \frac{|A|}{2}\log\frac{2\lambda_+}{\lambda_-}
    =
    d\log\frac{2\lambda_+}{\lambda_-} .
\end{equation*}
Combining the three displays
with~\eqref{eq:algoone-scale-invariance}: on the intersection of the
rescaled-design event and the event in~\eqref{eq:algoone-abbasi-applied},
\begin{equation}
    S_m^\top V_m^{-1}S_m
    \leq
    \frac{4\sigma_i^2}{m}
    \left(d\log\frac{2\lambda_+}{\lambda_-}+\log\frac{1}{\delta}\right).
    \label{eq:algoone-snm-quadratic}
\end{equation}
{On this event, for every $j\in A$, Cauchy--Schwarz in the $V_m^{-1}$ inner product gives}
\begin{align*}
    \abs{\widehat\beta_{ij}-\beta_{ij}}
    &= \abs{e_j^\top V_m^{-1}S_m} \\
    &= \abs{\langle e_j,S_m\rangle_{V_m^{-1}}} \\
    &\leq \sqrt{(V_m^{-1})_{jj}}
    \sqrt{S_m^\top V_m^{-1}S_m}.
\end{align*}
The definitions in~\eqref{eq:algoone-partial-oracle-estimates} along with the reverse triangle inequality give,
\begin{equation*}
    \max_{j\in A}
    \abs{\widehat\Psi^*_{i,A}(j)-\Psi^*_{i,A}(j)}
    \leq
    2\sqrt{\frac{1}{m}
    \left(d\log\frac{2\lambda_+}{\lambda_-}
    +\log\frac{1}{\delta}\right)}.
\end{equation*}
{This is the claimed oracle-score error bound.}
\end{proof}

We next state and prove the two lemmas invoked in the proof sketch of Section~\ref{sec:algoone-separation}. The first controls the conditional-variance estimate $\widehat\sigma_{i,A}^{2}$.
\begin{lemma}
\label{lem:algoone-residual}
Suppose $S_i\subseteq A$ and $V_m\succ0$. Then
\begin{equation}
    \widehat\sigma_{i,A}^{2}
    =
    \frac{1}{m}\norm{\boldsymbol\varepsilon}_2^{2}
    -
    S_m^\top V_m^{-1}S_m .
    \label{eq:algoone-residual-identity}
\end{equation}
Moreover, for every $\delta\in(0,1)$, with probability at least
$1-\delta$,
\begin{equation}
    \abs{\frac{1}{m}\norm{\boldsymbol\varepsilon}_2^{2}-\sigma_i^2}
    \leq
    \sigma_i^2
    \left(2\sqrt{\frac{1}{m}\log \left(\frac{2}{\delta}\right) }+\frac{2}{m}\log\left(\frac{2}{\delta}\right) \right).
    \label{eq:algoone-noise-energy}
\end{equation}
\end{lemma}

\paragraph{Proof of Lemma~\ref{lem:algoone-residual}}
\label{app:algoone-residual-proof}
\begin{proof}
Since $S_i\subseteq A$, the matrix regression
form~\eqref{eq:algoone-matrix-regression} holds, and on $V_m\succ0$ the
fitted residual is
\begin{equation*}
    \mathbf Y-\mathbf Z\widehat\beta_{i,A}
    =(I_m-P)\boldsymbol\varepsilon,
    \qquad
    P:=\mathbf Z(\mathbf Z^\top\mathbf Z)^{-1}\mathbf Z^\top ,
\end{equation*}
because
$\mathbf Y-\mathbf Z\widehat\beta_{i,A}
=\mathbf Z\beta_{i,A}+\boldsymbol\varepsilon
-\mathbf Z\widehat\beta_{i,A}$ and
$\mathbf Z(\widehat\beta_{i,A}-\beta_{i,A})=P\boldsymbol\varepsilon$.
Since $P$ is an orthogonal projection,
\begin{equation*}
    m\widehat\sigma^2_{i,A}
    =\norm{(I_m-P)\boldsymbol\varepsilon}_2^2
    =\norm{\boldsymbol\varepsilon}_2^2
    -\norm{P\boldsymbol\varepsilon}_2^2,
    \qquad
    \norm{P\boldsymbol\varepsilon}_2^2
    =\boldsymbol\varepsilon^\top
    \mathbf Z(\mathbf Z^\top\mathbf Z)^{-1}\mathbf Z^\top
    \boldsymbol\varepsilon
    =m S_m^\top V_m^{-1}S_m ,
\end{equation*}
which gives~\eqref{eq:algoone-residual-identity}.

For the noise energy, the innovations $\varepsilon^{(r)}$ are iid
$\mathcal N(0,\sigma_i^2)$ by
Lemma~\ref{lem:algoone-subtrajectory-regression}, so
$\norm{\boldsymbol\varepsilon}_2^2/\sigma_i^2\sim\chi^2_m$. The
chi-square tail bounds of \citet{laurent2000adaptive} give, for every
$x>0$, each with probability at most $e^{-x}$,
\begin{equation*}
    \norm{\boldsymbol\varepsilon}_2^2
    >\sigma_i^2\bigl(m+2\sqrt{mx}+2x\bigr),
    \qquad
    \norm{\boldsymbol\varepsilon}_2^2
    <\sigma_i^2\bigl(m-2\sqrt{mx}\bigr).
\end{equation*}
Setting $x=\log(2/\delta)$, a union bound and division by $m$
give~\eqref{eq:algoone-noise-energy}.
\end{proof}

The second lemma identifies the signal carried by the partial-oracle score $\Psi^{\ast}_{i,A}$: whenever $\widetilde V_m\succeq\lambda_-I_A$, the score is at least the normalized edge strength, up to the deterministic factor $\sqrt{\lambda_-}$, on every true edge, and it vanishes off the neighborhood.
\begin{lemma}
\label{lem:algoone-signal-conversion}
Suppose $\widetilde V_m\succeq\lambda_- I_A$ for some $\lambda_->0$.
Then, for every $j\in A$,
\begin{equation}
    (V_m^{-1})_{jj}\leq\frac{\Theta_{jj}}{\lambda_-},
    \label{eq:algoone-inverse-diagonal}
\end{equation}
and consequently
\begin{equation}
    \Psi^{\ast}_{i,A}(j)
    \geq
    \kappa_{ij}\sqrt{\lambda_-}
    \quad\text{if }j\in S_i,
    \qquad
    \Psi^{\ast}_{i,A}(j)=0
    \quad\text{if }j\notin S_i .
    \label{eq:algoone-signal-conversion}
\end{equation}
\end{lemma}

\paragraph{Proof of Lemma~\ref{lem:algoone-signal-conversion}}
\begin{proof}
Since $V_m^{-1}=D_A^{1/2}\widetilde V_m^{-1}D_A^{1/2}$, we have $(V_m^{-1})_{jj}=\Theta_{jj}(\widetilde V_m^{-1})_{jj} \leq\Theta_{jj} \lambda_{\max}(\widetilde V_m^{-1}) \leq\Theta_{jj}/\lambda_-$, which is~\eqref{eq:algoone-inverse-diagonal}. For $j\notin S_i$ we have $\beta_{ij}=0$ and hence $\Psi^{\ast}_{i,A}(j)=0$. For $j\in S_i$, by Definition~\ref{def:algoone-partial-oracle-estimates} and~\eqref{eq:algoone-inverse-diagonal},
\begin{equation*}
    \Psi^{\ast}_{i,A}(j)
    =
    \frac{\abs{\beta_{ij}}}{\sigma_i\sqrt{(V_m^{-1})_{jj}}}
    \geq
    \frac{\abs{\beta_{ij}}}{\sigma_i\sqrt{\Theta_{jj}}} \sqrt{\lambda_-}
    =
    \kappa_{ij}\sqrt{\lambda_-},
\end{equation*}
where the last equality is~\eqref{eq:kappa-ij}.
\end{proof}

\paragraph{Proof of Proposition~\ref{prop:algoone-separation}}
\label{app:algoone-separation-proof}
\begin{proof}
Set $\bar\lambda^{+}:=\lambda_m^{+}(\delta/9)$ and define three events.
Let $\mathcal E_1$ be the event of
Proposition~\ref{prop:algoone-design-control} applied at failure
probability $\delta/3$: on $\mathcal E_1$, simultaneously for all tested
pairs $(i,A)$,
$\lambda_m^{-}I_A\preceq\widetilde V_m\preceq\bar\lambda^{+}I_A$, and
$\Pbb(\mathcal E_1^c)\leq\delta/3$. Let
\begin{equation*}
    \mathcal E_2
    :=
    \left\{
    S_m^\top V_m^{-1}S_m
    \leq
    \sigma_i^2\varrho_m^2
    \ \text{for every tested pair }(i,A)
    \right\},
    \qquad
    \varrho_m^2
    :=
    \frac{4}{m}
    \left(
        d\log\frac{2\bar\lambda^{+}}{\lambda_m^{-}}
        +\log\frac{3|\mathfrak T|}{\delta}
    \right),
\end{equation*}
and let
$\mathcal E_3:=\bigl\{\abs{\norm{\boldsymbol\varepsilon_i}_2^2/m
-\sigma_i^2}\leq\sigma_i^2\bigl(2\sqrt{x/m}+2x/m\bigr)
\text{ for every }i\in[p]\bigr\}$ with $x:=\log(6p/\delta)$, where
$\boldsymbol\varepsilon_i$ denotes the innovation vector of the first
$m$ updates of node $i$. {The quadratic-form bound~\eqref{eq:algoone-snm-quadratic}, established in the proof of Lemma~\ref{lem:algoone-partial-oracle-error} and instantiated with $(\lambda_-,\lambda_+)=(\lambda_m^{-},\bar\lambda^{+})$ at failure probability $\delta/(3|\mathfrak T|)$,} together with a union bound over the at most $|\mathfrak T|$
tested pairs, gives
$\Pbb(\mathcal E_1\cap\mathcal E_2^c)\leq\delta/3$;
Lemma~\ref{lem:algoone-residual} applied at failure probability $\delta/(3p)$ and a union
bound over the $p$ nodes gives $\Pbb(\mathcal E_3^c)\leq\delta/3$.
Hence $\Pbb(\mathcal E_1\cap\mathcal E_2\cap\mathcal E_3)\geq1-\delta$.

We first record that the separation
budget~\eqref{eq:algoone-separation-budget} implies, for $C$ large
enough, the design budget~\eqref{eq:algoone-design-budget} at failure
probability $\delta/3$ together with the two numeric conditions
\begin{equation}
    \varrho_m\leq\frac{\kappa}{8}\sqrt{\lambda_m^{-}},
    \qquad
    2\sqrt{\frac{x}{m}}+\frac{2x}{m}\leq\frac14 .
    \label{eq:algoone-separation-numeric}
\end{equation}
Indeed, comparing~\eqref{eq:algoone-Bm} at the failure probabilities
$\delta/9$ and $\delta$ gives $\bar\lambda^{+}\leq5\lambda_m^{+}(\delta)$,
and $\lambda_m^{-}\in[\log3/(128q), 1]$ under the budget, so
$\log(2\bar\lambda^{+}/\lambda_m^{-})
\leq\log\bigl(C'q\lambda_m^{+}(\delta)\bigr)$ for a universal
$C'$. The first condition in~\eqref{eq:algoone-separation-numeric} then
reads
$m\geq\bigl(256/(\kappa^2\lambda_m^{-})\bigr)
\bigl\{d\log(C'q\lambda_m^{+}(\delta))+\log(3|\mathfrak T|/\delta)\bigr\}$,
which holds under~\eqref{eq:algoone-separation-budget} because
$1/\lambda_m^{-}\leq128q/\log3$. The second condition and the
design budget are implied similarly, using $\kappa\leq1$ and $|\mathfrak T|\geq p$.

Now, on $\mathcal E_1\cap\mathcal E_2\cap\mathcal E_3$, fix a
triple $(i,A,j)\in\mathfrak T$ with $S_i\subseteq A$. On $\mathcal E_1$,
$\widetilde V_m\succeq\lambda_m^{-}I_A\succ0$, so $V_m\succ0$ and
$\widehat\beta_{i,A}$ is well defined. By
{the coordinatewise bound in the proof of Lemma~\ref{lem:algoone-partial-oracle-error}}, $\mathcal E_1$, $\mathcal E_2$,
and~\eqref{eq:algoone-separation-numeric},
\begin{equation*}
    \abs{\widehat{\Psi}^{\ast}_{i,A}(j)-\Psi^{\ast}_{i,A}(j)}
    \leq
    \frac{\sqrt{S_m^\top V_m^{-1}S_m}}{\sigma_i}
    \leq
    \varrho_m
    \leq
    \frac{\kappa}{8}\sqrt{\lambda_m^{-}} .
\end{equation*}
By the residual identity~\eqref{eq:algoone-residual-identity} together
with $\mathcal E_2$, $\mathcal E_3$,
and~\eqref{eq:algoone-separation-numeric},
\begin{equation*}
    \frac{\widehat\sigma^2_{i,A}}{\sigma_i^2}
    \geq
    1-\frac14-\varrho_m^2
    \geq
    \frac12,
    \qquad
    \frac{\widehat\sigma^2_{i,A}}{\sigma_i^2}
    \leq
    1+\frac14
    \leq
    2,
\end{equation*}
using $\varrho_m^2\leq\kappa^2\lambda_m^{-}/64\leq1/64$ (recall
$\kappa\leq1$ and $\lambda_m^{-}\leq1$). In particular
$\widehat\sigma_{i,A}>0$, so the statistic takes its regular value and
$\sigma_i/\widehat\sigma_{i,A}\in[1/\sqrt2,\sqrt2]$, with
\begin{equation*}
    \widehat{\tssym}_{i,A}(j)
    =
    \frac{\sigma_i}{\widehat\sigma_{i,A}} 
    \widehat{\Psi}^{\ast}_{i,A}(j) .
\end{equation*}

If $j\notin S_i$, Lemma~\ref{lem:algoone-signal-conversion} gives
$\Psi^{\ast}_{i,A}(j)=0$, so
\begin{equation*}
    \widehat{\tssym}_{i,A}(j)
    \leq
    \sqrt2 \widehat{\Psi}^{\ast}_{i,A}(j)
    \leq
    \sqrt2\cdot\frac{\kappa}{8}\sqrt{\lambda_m^{-}}
    \leq
    \frac{\kappa}{2}\sqrt{\lambda_m^{-}}
    =\tau .
\end{equation*}
If $j\in S_i$, Lemma~\ref{lem:algoone-signal-conversion} gives
$\Psi^{\ast}_{i,A}(j)\geq\kappa_{ij}\sqrt{\lambda_m^{-}}
\geq\kappa\sqrt{\lambda_m^{-}}$, so
\begin{equation*}
    \widehat{\tssym}_{i,A}(j)
    \geq
    \frac{1}{\sqrt2}
    \left(\kappa\sqrt{\lambda_m^{-}}
    -\frac{\kappa}{8}\sqrt{\lambda_m^{-}}\right)
    =
    \frac{7}{8\sqrt2} \kappa\sqrt{\lambda_m^{-}}
    >
    \frac{\kappa}{2}\sqrt{\lambda_m^{-}}
    =\tau ,
\end{equation*}
since $7/(8\sqrt2)>1/2$. This is exactly the separation
property~\eqref{eq:algoone-separation-conclusion}.
\end{proof}

\subsection{Finite-horizon guarantee}
The final proof combines Proposition~\ref{prop:algoone-separation} with the row-collection event from Lemma~\ref{lem:row-collection}. On their intersection, Proposition~\ref{prop:generic-separation-recovery} gives exact recovery for the observed trajectory.
\paragraph{Proof of Theorem~\ref{thm:algoone-main}}
\label{app:algoone-main-proof}
\begin{proof}
Throughout, $m=\lfloor N/(2p)\rfloor$. Let
$\mathcal C_N=\{T_m\leq N\}$ be the row-collection event
of~\eqref{eq:algoone-clock-event}, and let $\mathcal S$ be the event
that $\widehat{\tssym}$ satisfies the separation property of
Definition~\ref{def:generic-separation} with threshold
$\tau=\frac{\kappa}{2}\sqrt{\lambda_m^{-}}$.

\emph{Reduction to the generic constants.} Since $\kappa\leq1$,
$E_0+N+\log(2/\delta)\geq1$, and $\rho_{2d}(\Theta)\geq1$ (indeed
$\rho_{2d}(\Theta)\geq\max_{j}\Theta_{jj}\Sigma_{jj}\geq1$, the
marginal variance dominating the conditional one), the right-hand side
of~\eqref{eq:algoone-main-budget} is at least
$Cpd\{d\log(Cpd)+\log(1/\delta)\}$, so the budget with $C\geq16$
gives
\begin{equation}
    N\geq8p\log\frac{2p}{\delta} .
    \label{eq:algoone-proof-clockscale}
\end{equation}
By the choice of $m$, $2pm\leq N$; combined
with~\eqref{eq:algoone-proof-clockscale}, the definition of $t_m$ in
\eqref{eq:algoone-Bm} gives
$t_m(\delta/2)\leq N+1\leq2N$. Hence, by~\eqref{eq:algoone-Bm} and
$2\sqrt{2Nx}\leq N+2x$,
\begin{equation*}
    B_m(\delta)
    \leq
    E_0+3N+4\log\frac{2}{\delta}
    \leq
    4\left(E_0+N+\log\frac{2}{\delta}\right),
    \qquad
    \lambda_m^{+}(\delta)
    \leq
    4\rho_{2d}(\Theta)\left(E_0+N+\log\frac{2}{\delta}\right).
\end{equation*}
Since also $q=2d+1\leq3d$ and
$\log(|\mathfrak T|/\delta)\leq4d\log p+\log(1/\delta)$, it follows that, for any
prescribed universal constant $C_0$, the
budget~\eqref{eq:algoone-main-budget} with $C$ large enough implies
\begin{equation}
    N
    \geq
    \frac{C_0\,p\,q}{\kappa^{2}}
    \left\{
        2d\log\!\bigl(C_0q\lambda_m^{+}(\delta)\bigr)
        +\log\frac{|\mathfrak T|}{\delta}
    \right\}.
    \label{eq:algoone-proof-generic-budget}
\end{equation}

\emph{Clock.} By~\eqref{eq:algoone-proof-clockscale} and $2pm\leq N$,
since $N$ is an integer we have $N\geq t_m(\delta/2)$, and
Lemma~\ref{lem:row-collection} applied at failure probability
$\delta/2$ yields
\begin{equation*}
    \Pbb\bigl(\mathcal C_N^c\bigr)
    \leq
    \Pbb\bigl(T_m>t_m(\delta/2)\bigr)
    \leq
    \frac{\delta}{2} .
\end{equation*}

\emph{Separation.} The inequality~\eqref{eq:algoone-proof-generic-budget} gives
\begin{equation*}
    m
    \geq
    \frac{N}{2p}-1
    \geq
    \frac{C_0\,q}{2\kappa^{2}}
    \left\{
        2d\log\!\bigl(C_0q\lambda_m^{+}(\delta)\bigr)
        +\log\frac{|\mathfrak T|}{\delta}
    \right\}
    -1 .
\end{equation*}
A direct comparison of~\eqref{eq:algoone-Bm} at the failure
probabilities $\delta/2$ and $\delta$ gives
$\lambda_m^{+}(\delta/2)\leq2\lambda_m^{+}(\delta)$, and
$\log(2|\mathfrak T|/\delta)\leq2\log(|\mathfrak T|/\delta)$. Hence, for $C_0$ large enough, the
inequality above implies the separation
budget~\eqref{eq:algoone-separation-budget} at failure probability
$\delta/2$, and Proposition~\ref{prop:algoone-separation} yields
$\Pbb(\mathcal S^c)\leq\delta/2$.

\emph{Combination.} On $\mathcal C_N$, every node receives at least $m$
updates within the horizon, so the statistic computed from the observed
trajectory coincides with the fixed-row statistic analyzed in
Sections~\ref{sec:design-control} and~\ref{sec:algoone-separation}, and
the empty-output convention is not invoked. On
$\mathcal C_N\cap\mathcal S$, the statistic satisfies the separation
property, so Proposition~\ref{prop:generic-separation-recovery} gives
$\widehat{\Ecal}=\Ecal$. A union bound concludes:
\begin{equation*}
    \Pbb\bigl(\widehat{\Ecal}\neq\Ecal\bigr)
    \leq
    \Pbb\bigl(\mathcal C_N^c\bigr)
    +
    \Pbb\bigl(\mathcal S^c\bigr)
    \leq
    \delta . \qedhere
\end{equation*}
\end{proof}

\section{Proofs for \algotwo{}}\label{app:algo2-proof-of-lemmas}
We present the proofs for \algotwo{} in this section. We first establish the block regression identity and the ingredients required for the ``seperation'' condition for a fixed-triple. We then establish an estimate of the horizon length that ensures the required number of blocks are available with high probability. The final subsection then combines these estimates over the tested family to prove finite-sample recovery.

\subsection{Block regression and fixed-triple separation}
We first turn to one tested triple $(i,A,j)$. The $i,i,j,i$ pattern gives the one-dimensional regression identity used by the statistic. The fixed-triple separation proposition then follows by combining the conditional regression bounds with the null and signal cases, which are proved below.
\paragraph{Proof of Lemma~\ref{lem:iiji-regression}}\label{app:iiji-regression-proof}
\begin{proof}
By definition, the labels of an $A$-$iiji$ block are $i,i,j,i$ at the consecutive {$\bar A$-update} times $t^{(b)}_1 < t^{(b)}_2 < t^{(b)}_3 < t^{(b)}_4$. Therefore, between two consecutive {$\bar A$-update} times only coordinates outside {$\bar A$} are updated. So,
\begin{align*}
    y_b
    &= X_i^{(t^{(b)}_4)} - X_i^{(t^{(b)}_1)} \\
    &\stackrel{\mathrm{(a)}}{=} \sum_{k \in S_i} \beta_{ik}\bigl(X_k^{(t^{(b)}_4 - 1)} - X_k^{(t^{(b)}_1 - 1)}\bigr) + \xi_b \\
    &\stackrel{\mathrm{(b)}}{=} \beta_{ij}\bigl(X_j^{(t^{(b)}_4 - 1)} - X_j^{(t^{(b)}_1 - 1)}\bigr) + \xi_b \\
    &\stackrel{\mathrm{(c)}}{=} \beta_{ij}\bigl(X_j^{(t^{(b)}_3)} - X_j^{(t^{(b)}_2)}\bigr) + \xi_b \\
    &=\beta_{ij}  x_b + \xi_b ,
\end{align*}
which is the claimed identity~\eqref{eq:algo2-iiji-regression}. In step~(b), every neighbor $k \in S_i \subseteq A$ is frozen across the block and $X_k^{(t^{(b)}_4 - 1)} = X_k^{(t^{(b)}_1 - 1)}$; the corresponding terms cancel. In step~(c), coordinate $j$ is updated only at $t^{(b)}_3$, so $X_j^{(t^{(b)}_1 - 1)} = X_j^{(t^{(b)}_2)}$, and $X_j^{(t^{(b)}_4 - 1)} = X_j^{(t^{(b)}_3)}$. For the noise structure, recall that the innovations $\{\varepsilon_i^{(n)} : I^{(n)} = i\}$ are iid $\mathcal{N}(0, \sigma_i^2)$. Each $\xi_b = \varepsilon_i^{(t^{(b)}_4)} - \varepsilon_i^{(t^{(b)}_1)}$ is a difference of two of them. Since the blocks are non-overlapping and the innovations are independent of the update labels, the $\xi_b$ are iid $\mathcal{N}(0, 2\sigma_i^2)$.

For $x_b \perp \xi_b$ it suffices that $x_b$ involves neither innovation composing $\xi_b$. Since new innovations are independent of past values and $x_b$ is fixed by step $t^{(b)}_3 < t^{(b)}_4$, we get $\varepsilon_i^{(t^{(b)}_4)} \perp x_b$. Further, the $j$-update at $t^{(b)}_3$ is
\begin{align*}
    X_j^{(t^{(b)}_3)}
    &=
    \sum_{m \in S_j} \beta_{jm}  X_m^{(t^{(b)}_3 - 1)}
    +
    \varepsilon_j^{(t^{(b)}_3)},
\end{align*}
whose only $i$-dependent regressor is $X_i^{(t^{(b)}_3 - 1)} = X_i^{(t^{(b)}_2)} = \sum_{k \in S_i} \beta_{ik}  X_k^{(t^{(b)}_1 - 1)} + \varepsilon_i^{(t^{(b)}_2)}$, which carries $\varepsilon_i^{(t^{(b)}_2)}$, not $\varepsilon_i^{(t^{(b)}_1)}$. Further, no other regressor carries $\varepsilon_i^{(t^{(b)}_1)}$, since it could spread only through a neighbor of $i$ updating during the block, and every neighbor of $i$ lies in $A$ and stays frozen. Hence $x_b \perp \xi_b$.
\end{proof}
The regression identity reduces the statistic for a fixed triple to a one-dimensional comparison. To turn that identity into separation, we control the noise-normalized regression drop under the null and under a nonzero edge. The proposition records the resulting dichotomy.
\paragraph{Proof of Proposition~\ref{prop:algo2-fixed-triple-separation}}\label{app:algo2-fixed-triple-separation-proof}

\textit{\textbf{Proposition~\ref{prop:algo2-fixed-triple-separation}:}} Fix a triple $(i,A,j)$ with $j\in A$ and $S_i\setminus\{j\}\subseteq A$. For every $\delta \in (0,1)$, if
\begin{align*}
    n_0 \geq \frac{5120}{\kappa^2}\log\frac{3e n_0}{\delta},
\end{align*}
and $\tau = \frac{\kappa^{2}}{128}$, then with probability at least $1-\delta$,
\begin{equation*}
    \begin{cases}
        \Phi_{i,A}(j)\leq \tau, & j\notin S_i,\\[1mm]
        \Phi_{i,A}(j)> \tau, & j\in S_i.
    \end{cases}
\end{equation*}
\begin{proof}
    First suppose $j\notin S_i$. Since $\kappa\leq 1$,
    \begin{equation*}
        n_0
        \geq
        \frac{5120}{\kappa^2}\log\frac{3e n_0}{\delta}
        >
        80\log\frac{2e n_0}{\delta}.
    \end{equation*}
    Lemma~\ref{lem:algo2-null} gives, with probability at least $1-\delta$,
    \begin{equation*}
        \Phi_{i,A}(j)
        \leq
        40\frac{\log(2e n_0/\delta)}{n_0}
        \leq
        40\frac{\log(3e n_0/\delta)}{n_0}
        \leq
        \frac{\kappa^2}{128}
        =
        \tau.
    \end{equation*}

    Now suppose $j\in S_i$. The sample-size condition implies the hypothesis of Lemma~\ref{lem:algo2-alt}. Hence, with probability at least $1-\delta$,
    \begin{equation*}
        \Phi_{i,A}(j)
        \geq
        \frac{\kappa^2}{64}
        >
        \frac{\kappa^2}{128}
        =
        \tau.
    \end{equation*}
\end{proof}

\subsection{Block availability}
The statistic for $(i,A,j)$ is formed from the first $n_0$ completed blocks, so it is available by time $N$ on the event $N_{(i,A,j)}(N)\geq n_0$. We now record the separate clock estimate that guarantees this event by a prescribed global time.
\paragraph{Proof of Lemma~\ref{lem:algo2-fixed-triple-clock}}\label{app:fixed-triple-clock-proof}

\textit{\textbf{Lemma~\ref{lem:algo2-fixed-triple-clock}:}} Fix a triple $(i,A,j)$ with $j\in A$, and write $q:=|A|+1$. For every $\delta\in(0,1)$ and every $n_0\geq 1$, if
\begin{equation*}
    N
    \geq
    128 p q^3\left(n_0+\log\frac{2}{\delta}\right),
\end{equation*}
then
\begin{equation*}
    \Pbb\left(
        N_{(i,A,j)}(N)<n_0
    \right)
    \leq
    \delta.
\end{equation*}

\begin{proof}
Let
\begin{equation*}
    K_{\bar A}(N):=\max\{\ell:t_{\bar A}^{(\ell)}\leq N\}
\end{equation*}
be the number of {$\bar A$}-clock ticks among the first $N$ global updates. Since the global labels are iid uniform on $\Vcal$,
\begin{equation}
    K_{\bar A}(N)\sim \operatorname{Bin}\left(N,\frac{q}{p}\right),
    \qquad
    \mu:=\Ebb K_{\bar A}(N)=\frac{Nq}{p}.
\end{equation}
Set $L:=\log(2/\delta)$. Under~\eqref{eq:algo2-fixed-triple-clock-condition},
\begin{equation}
    \mu
    =
    \frac{Nq}{p}
    \geq
    128q^4(n_0+L).
\end{equation}
By the standard binomial Chernoff lower-tail bound,
\begin{equation}
    \Pbb\left(K_{\bar A}(N)<\frac{\mu}{2}\right)
    \leq
    \exp\left(-\frac{\mu}{8}\right).
    \label{eq:algo2-retained-count-lower-tail}
\end{equation}
Since $q\geq 2$ and $n_0\geq 1$, this is at most $e^{-L}=\delta/2$.

Condition on $K_{\bar A}(N)=m$. The retained labels are iid uniform on {$\bar A$}. After partitioning them into disjoint four-label windows, the number of $A$-$iiji$ blocks has distribution
\begin{equation}
    \operatorname{Bin}\left(
        \left\lfloor \frac{m}{4}\right\rfloor,
        q^{-4}
    \right).
\end{equation}
On the event $m\geq \mu/2$, we have $\mu\geq 16$ and hence
\begin{equation}
    \left\lfloor \frac{m}{4}\right\rfloor
    \geq
    \frac{m}{8}
    \geq
    \frac{\mu}{16}.
\end{equation}
Thus the conditional mean of the block count is at least
\begin{equation}
    \frac{\mu}{16}q^{-4}
    =
    \frac{N}{16pq^3}.
\end{equation}
By~\eqref{eq:algo2-fixed-triple-clock-condition}, this lower bound is at least $8(n_0+L)$, and in particular at least $2n_0$. Therefore the event that the block count is below $n_0$ is contained in its lower-tail event below half its conditional mean. Applying the standard binomial Chernoff lower-tail bound again gives, uniformly over all $m\geq \mu/2$,
\begin{equation}
    \Pbb\left(
        N_{(i,A,j)}(N)<n_0
         \middle| 
        K_{\bar A}(N)=m
    \right)
    \leq
    \exp\left(-\frac{N}{128pq^3}\right).
    \label{eq:algo2-block-count-lower-tail}
\end{equation}
The exponent is at least $L$, so this is at most $\delta/2$.
Combining~\eqref{eq:algo2-retained-count-lower-tail} and~\eqref{eq:algo2-block-count-lower-tail},
\begin{equation}
    \Pbb\left(
        N_{(i,A,j)}(N)<n_0
    \right)
    \leq
    \exp\left(-\frac{Nq}{8p}\right)
    +
    \exp\left(-\frac{N}{128pq^3}\right).
\end{equation}
The two terms are each at most $\delta/2$, which proves the claim.
\end{proof}

\subsection{Conditional regression and case bounds}\label{app:algo2-conditional-regression}
We next supply the concentration estimates invoked in Proposition~\ref{prop:algo2-fixed-triple-separation}. Lemma~\ref{lem:algo2-snm-bound} controls the self-normalized cross term, Lemma~\ref{lem:algo2-noise-energy} controls the noise energy, and Lemma~\ref{lem:algo2-causal-design} explains why these bounds remain valid even though the block covariates are dependent across blocks. The last three proofs treat the null and signal cases separately.
{We first construct the filtration on which the self-normalization argument rests. For a fixed triple $(i,A,j)$, recall that $\xi_b = \varepsilon_i^{(t^{(b)}_4)}-\varepsilon_i^{(t^{(b)}_1)}$, and define $\eta_b := \varepsilon_i^{(t^{(b)}_4)}+\varepsilon_i^{(t^{(b)}_1)}$. Let $\mathcal M:=\{t^{(b)}_1,t^{(b)}_4:1\le b\le n_0\}$ be the set of measurement update times. Recall that {$\xi_{1:n}, \eta_{1:n}$} denotes $\{ \xi_{b}\}_{b \in [n]}$ and $\{ \eta_{b}\}_{b \in [n]}$, respectively.}
\begin{definition}[Base sigma-field and filtration]
\label{def:algo2-base-sigma-field}
The base sigma-field for the first $n_0$ $A$-$iiji$ blocks is
\begin{align}
    \mathscr B
    :=
    \sigma\!\left(
        {I^{(1:\infty)}},
        \{\varepsilon(s):s\notin\mathcal M\},\,
        {\eta_{1:n_{0}}}
    \right),
\end{align}
where $I^{(1:\infty)}$ denotes the full update schedule. {The associated filtration is $\mathcal{F}_{b} := \mathscr B \vee \sigma({\xi_{1:b}})$ for $b \geq 1$, with $\mathcal{F}_0 := \mathscr B$.}
\end{definition}
The following lemma then establishes the required properties of the sequence $\{x_{b}, \xi_{b}\}_{1 \leq b \leq n_{0}}$ under the filtration defined above. 
\begin{lemma}
\label{lem:algo2-causal-design}
Let $\mathscr B$ be the base sigma-field and $(\mathcal{F}_{b})_{b\geq 0}$ {the filtration} of Definition~\ref{def:algo2-base-sigma-field}. Then the following {hold} for $b > 0$:
\begin{enumerate}
    \item $x_{b} \mid \mathscr B$ is an affine function of $(\xi_{1},\dots,\xi_{b-1})$. In particular, there {exist} $u_{b} \in \mathbb{R}$ and $L_{b,a} \in \mathbb{R}, a < b$, $\mathscr B$-{measurable} such that
    \begin{align}
        x_{b} \mid \mathscr B = u_{b} + \sum_{a < b} L_{ba} \xi_{a}.
    \end{align}
    \item  $x_{b}$ is $\mathcal{F}_{b-1}$ measureable.
    \item  $\xi_{b} \mid \mathcal{F}_{b-1} \sim \mathcal{N}(0, 2\sigma_{i}^{2})$.
\end{enumerate}
\end{lemma}

\paragraph{Proof of Lemma~\ref{lem:algo2-snm-bound}}\label{app:snm-bound-lem-proof}
With Lemma~\ref{lem:algo2-causal-design}, we apply the standard self-normalization machinery~\cite{de2004self}. Write $S:=\mathbf{X}^{\top}\boldsymbol{\xi}$ and $V:=\norm{\mathbf{X}}^{2}$. For each fixed deterministic $\lambda\in\mathbb{R}$, we have the conditional exponential identity $\Ebb[\exp\{\lambda S-\lambda^{2}\sigma_i^{2}V\}\mid\mathscr B]=1$. The exponent is maximized at $\lambda^\ast:=S/(2\sigma_i^{2}V)$, with maximum $S^{2}/(4\sigma_i^{2}V)$. Since $\lambda^\ast$ depends on the realized values of $S$ and the unbounded quantity $V$, we cannot use the fixed-$\lambda$ identity at $\lambda^\ast$ directly. We therefore first restrict $V$ and $\norm{\boldsymbol{\xi}}$ to high-probability ranges, partition the range of $V$ into dyadic intervals, and construct a deterministic grid covering the corresponding values of $\lambda^\ast$ on each interval{; this dyadic partitioning is an instance of the peeling device of empirical process theory~\citep{vandegeer2000empirical}, and parallels the ``stitching'' construction of~\citet{howard2021time} for self-normalized bounds}. Applying the fixed-$\lambda$ identity on these grids and taking a union bound gives the logarithmic estimate below. This bound is sufficient for the separation argument, though  it may be possible to sharpen it.

\begin{proof}
    Write $S = \mathbf{X}^\top \boldsymbol{\xi} = \sum_{b=1}^{n_0} x_b\xi_b, V = \norm{\mathbf{X}}^2 = \sum_{b=1}^{n_0} x_b^2 .$ We prove the bound conditionally on the base sigma-field $\mathscr B$. The unconditional claim follows by averaging. By Lemma~\ref{lem:algo2-causal-design}, conditional on $\mathscr B$ there is a $\mathscr B$-measurable vector $u \in  \mathbb{R}^{n_{0}}$ and a strictly lower triangular $\mathscr B$-measurable matrix $L \in \mathbb{R}^{n_{0} \times n_{0}}$ such that
    \begin{align*}
        x = u+L\xi \text{ with } \xi \sim \mathcal N(0,2\sigma_i^2 I_{n_0}).
    \end{align*}
    Write $\zeta:=\xi/(\sqrt{2}\sigma_i)$ and $\bar L:=\sqrt{2}\sigma_i L$. Then $\zeta\sim\mathcal N(0,I_{n_0})$ and $ x = u+\bar L\zeta$.  Note that the ratio $S^2/V$ is invariant to scaling in $x,u,\bar L$ by any $\mathscr B$-measurable scalar $r>0$. That is, taking $r^2:=\norm{u}^2+\norm{\bar L}_{\mathrm F}^2,$ and setting, $\tilde{x} = \frac{x}{r}, \tilde{u} = \frac{u}{r}$ and $\tilde{L} = \frac{L}{r}$ we have
    \begin{align*}
    \frac{S^{2}}{V} = \frac{(\mathbf{X}^\top \boldsymbol{\xi} )^{2}}{\norm{\mathbf{X}}^{2}} = \frac{(\tilde{x}^\top {\xi})^{2}}{\norm{\tilde{x}}^{2}}
    \end{align*}
    So we can work with the scaled terms or we may assume without loss of generality that
    \begin{align*}
        \norm{u}^2+\norm{\bar L}_{\mathrm F}^2
        &=
        1.
    \end{align*}
    Now, for every fixed $\lambda\in\mathbb R$, the fixed-$\lambda$ exponential identity gives
    \begin{align*}
        \E[]{\exp\left\{\lambda S-\lambda^2\sigma_i^2 V\right\} \mid \mathscr B} = 1
    \end{align*}
    Let $M(\lambda) := \exp\left\{\lambda S-\lambda^2\sigma_i^2 V\right\}$. This function is maximized at the data-dependent value
    \begin{align*}
        \lambda^\ast = \frac{S}{2\sigma_i^2 V}, \qquad \sup_{\lambda\in\mathbb R}\left\{\lambda S-\lambda^2\sigma_i^2 V\right\}
        = \frac{S^2}{4\sigma_i^2 V}.
    \end{align*}
    We cannot substitute $\lambda^\ast$ directly into the preceding exponential identity, since $\lambda^\ast$ depends on the data. The remainder of the proof discretizes a high-probability range for $\lambda^\ast$.

    Under this normalization,
    \begin{align*}
        |\lambda^\ast|
        &=
        \frac{|S|}{2\sigma_i^2 V}
        \leq
        \frac{\norm{\mathbf{X}}\norm{\boldsymbol{\xi}}}{2\sigma_i^2 V}
        \leq
        \frac{\norm{\boldsymbol{\xi}}}{2\sigma_i^2\sqrt V}.
    \end{align*}
    Thus an upper bound on $\norm{\boldsymbol{\xi}}$ and a lower bound on $V$ give a finite high-probability interval for $\lambda^\ast$. We formalize these controls through the event
    \begin{align*}
        \mathcal E
        :=
        \left\{
            \norm{\boldsymbol{\xi}}\leq R
        \right\}
        \cap
        \left\{
            v_0\leq V\leq U
        \right\}.
    \end{align*}
    We choose the three cutoffs so that each failure event has small conditional probability. First, since $\xi\mid\mathscr B\sim \mathcal N(0,2\sigma_i^2 I_{n_0})$, we have $\norm{\boldsymbol{\xi}}^2/(2\sigma_i^2)\sim\chi^2_{n_0}$ conditionally on $\mathscr B$. The Laurent--Massart chi-square bound~\citep{laurent2000adaptive} states that, for $Z\sim\chi^2_m$ and every $t>0$,
    \begin{align*}
        \Pbb\left(Z\geq m+2\sqrt{mt}+2t\right) \leq
        e^{-t}.
    \end{align*}
    Applying this with $m=n_0$ and $t=\log(12/\delta)$ gives
    \begin{align*}
        \Pbb\left(\norm{\boldsymbol{\xi}}>R \middle| \mathscr B\right) \leq \delta/12, \text{ where } R^2 := 2\sigma_i^2
        \left( n_0 + 2\sqrt{n_0\log\frac{12}{\delta}} + 2\log\frac{12}{\delta} \right).
    \end{align*}
    Second, under the normalization we have,
    \begin{align*}
        \Ebb[V\mid\mathscr B]
        &= \Ebb[\norm{u+\bar L\zeta}^{2}\mid\mathscr B] \\
        &= \norm{u}^{2} + \Ebb[\zeta^\top \bar L^\top\bar L\zeta\mid\mathscr B] \\
        &= \norm{u}^{2} + \operatorname{tr}(\bar L^\top\bar L) \\
        &= \norm{u}^{2}+\norm{\bar L}_{\mathrm F}^{2} = 1.
    \end{align*}
    For concentration, notice that the map $f(\zeta)=\norm{u+\bar L\zeta}$ is $\norm{\bar L}_{\mathrm{op}}$-Lipschitz, and $\norm{\bar L}_{\mathrm{op}}\leq\norm{\bar L}_{\mathrm F}\leq 1$. The Gaussian concentration inequality for Lipschitz functions~\citep[Theorem~5.6]{boucheron2013concentration} gives, for every $t>0$,
    \begin{align*}
        \Pbb\left(
            \norm{\mathbf{X}}
            >
            \Ebb[\norm{\mathbf{X}}\mid\mathscr B]+t
             \middle| \mathscr B
        \right)
        &\leq
        \exp\left(-\frac{t^2}{2}\right).
    \end{align*}
    By Jensen's inequality and the preceding second-moment identity, $\Ebb[\norm{\mathbf{X}}\mid\mathscr B]\leq \sqrt{\Ebb[V\mid\mathscr B]}=1$. Taking $t=\sqrt{2\log(12/\delta)}$ gives
    \begin{align*}
        \Pbb\left(
            V
            >
            \left(1+\sqrt{2\log\frac{12}{\delta}}\right)^2
             \middle| \mathscr B
        \right)
        &\leq
        \delta/12 .
    \end{align*}
    Thus we may take
    \begin{align*}
        U
        :=
        \left(1+\sqrt{2\log\frac{12}{\delta}}\right)^2 .
    \end{align*}
    Finally, we lower bound $V=\norm{\mathbf{X}}^2$. Since $\norm{x}\geq|a^\top x|$ for any unit vector $a$, it suffices to exhibit one direction carrying a definite amount of energy and then to rule out that $x$ is atypically small along it. The only properties of $a$ we use are that it is $\mathscr B$-measurable and that $\Ebb[(a^\top x)^2\mid\mathscr B]\geq1/n_0$; any such direction serves. Write $M:=\Ebb[xx^\top\mid\mathscr B]$, which is positive semidefinite, $\mathscr B$-measurable, and satisfies $\operatorname{tr}M=\Ebb[\norm{x}^2\mid\mathscr B]=1$. Its eigenvalues are nonnegative and sum to one, so its largest eigenvalue is at least $1/n_0$; let $a$ be a corresponding unit eigenvector, selected by a fixed measurable rule so that $a$ is $\mathscr B$-measurable, and set $q_a:=\Ebb[(a^\top x)^2\mid\mathscr B]=a^\top Ma\geq1/n_0$. Conditional on $\mathscr B$ the vector $a$ is deterministic, so $a^\top x=a^\top u+(\bar L^\top a)^\top\zeta$ is a univariate Gaussian whose variance $\norm{\bar L^\top a}^2$ may vanish, as it does when $\bar L=0$. We therefore prove a small-ball bound for univariate Gaussians in terms of the second moment. Let $Y\sim\mathcal N(\mu,\sigma^2)$, possibly with $\sigma=0$, and write
    \begin{align*}
        q:=\Ebb[Y^2]=\mu^2+\sigma^2.
    \end{align*}
    We claim that, for every $0\leq r\leq\sqrt q$,
    \begin{align}
        \Pbb\left(|Y|<r\right)
        \leq
        2\Phi\left(\frac{r}{\sqrt q}\right)-1
        \leq
        \frac{2r}{\sqrt{2\pi q}},
        \label{eq:one-dimensional-gaussian-small-ball}
    \end{align}
    where $\Phi$ denotes the standard Gaussian distribution function. If $\sigma=0$, then $|Y|=\sqrt q$ almost surely, and the claim is immediate. Suppose that $\sigma>0$. By symmetry, we may take $\mu\geq0$. Set
    \begin{align*}
        t:=\frac{\mu}{\sigma},
        \qquad
        \rho:=\frac{r}{\sqrt q}\leq1.
    \end{align*}
    For a standard Gaussian random variable $Z$, $\frac{Y}{\sqrt q} \stackrel{\mathrm d}{=} \frac{t+Z}{\sqrt{1+t^2}}.$ 
    Hence, with $\phi$ denoting the standard Gaussian density,
    \begin{align*}
        p_\rho(t) := \Pbb\left(|Y|<r\right) = \Phi\left(\rho\sqrt{1+t^2}-t\right) + \Phi\left(\rho\sqrt{1+t^2}+t\right) -1.
    \end{align*}
    To show that $p_\rho$ is nonincreasing, define $A:=\rho\sqrt{1+t^2}, b:=\frac{\rho t}{\sqrt{1+t^2}}$. Differentiation gives
    \begin{align*}
        p_\rho'(t)
        &=
        \phi(A-t)
        \left[
            b-1+(b+1)e^{-2At}
        \right].
    \end{align*}
    If $y:=t/\sqrt{1+t^2}$, then
    \begin{align*}
        \operatorname{arctanh}(b)
        =
        \operatorname{arctanh}(\rho y)
        \leq
        \frac{\rho y}{1-\rho^2y^2}
        \leq
        \frac{\rho y}{1-y^2}
        =
        At.
    \end{align*}
    Thus $e^{-2At}\leq(1-b)/(1+b)$, and therefore $p_\rho'(t)\leq0$. It follows that
    \begin{align*}
        p_\rho(t)
        \leq
        p_\rho(0)
        =
        2\Phi(\rho)-1
        \leq
        \frac{2\rho}{\sqrt{2\pi}},
    \end{align*}
    which proves~\eqref{eq:one-dimensional-gaussian-small-ball}.  Recall that, conditional on $\mathscr B$, the random variable $a^\top x$ is a possibly noncentered Gaussian with second moment $q_a\geq1/n_0$. We now take
    \begin{align*}
        v_0
        &:=
        \frac{\pi\delta^2}{288n_0}.
    \end{align*}
    Since $\delta\in(0,1)$, we have $v_0<1/n_0\leq q_a$. We may therefore apply~\eqref{eq:one-dimensional-gaussian-small-ball} conditionally with $Y=a^\top x$ and $r=\sqrt{v_0}$ to obtain
    \begin{align*}
        \Pbb\left(V<v_0 \middle| \mathscr B\right)
        &\leq
        \Pbb\left(|a^\top x|<\sqrt{v_0} \middle| \mathscr B\right)\\
        &\leq
        2\sqrt{v_0}
        \sqrt{\frac{n_0}{2\pi}}\\
        &=
        2
        \sqrt{\frac{\pi\delta^2}{288n_0}}
        \sqrt{\frac{n_0}{2\pi}}
        =
        \frac{\delta}{12}.
    \end{align*}
    By a union bound over the three failure events defining $\mathcal E$, we get $\Pbb(\mathcal E^c\mid\mathscr B) \leq \frac{\delta}{4}.$ Now, on the event $\mathcal E$, the random maximizer $\lambda^\ast$ lies in a bounded interval, but its location still depends on the data. In particular,
    \begin{align}
        |\lambda^{\ast}|
        =
        \frac{|S|}{2\sigma^{2}_{i}V}
        \leq
        \frac{R}{2\sigma^{2}_{i}\sqrt{V}}.
        \label{eq:lambda*-range}
    \end{align}
    The fixed-$\lambda$ identity can only be applied to values chosen before observing $\boldsymbol{\xi}$, so we replace the random choice $\lambda^\ast$ by a finite deterministic net. The net is constructed separately on dyadic ranges of $V$ such that, within each dyadic range, the possible interval for $\lambda^\ast$ and the mesh size are both fixed, and every possible $\lambda^\ast$ is close to some grid point.

    We now make this construction explicit. Set $K:=\left\lceil \log_2(U/v_0)\right\rceil$ and $v_k:=2^k v_0$ for $k=0,\ldots,K$. The intervals $I_k:=[v_k,2v_k]$, $k=0,\ldots,K$, cover $[v_0,U]$. For each such $k$, define
    \begin{align}
        B_k = \frac{R}{2\sigma_i^2\sqrt{v_k}}
    \end{align}
    representing the range $\lambda^{\ast}$ if $V$ happens to fall in the $k$th dyadic interval $I_{k}$, i.e. if $V \in I_{k}$ then $v_k\leq V$. From \eqref{eq:lambda*-range} we see $\lambda^\ast \in  [-B_k, B_{k}]$. Next we define the grid spacing for each $B_{k}$.
    \begin{align}
        h_k = \frac{1}{4\sigma_i\sqrt{v_k}}
    \end{align}

    With this grid spacing we now define the deterministic set of $\lambda$'s such that adjacent grid points are at most $h_k$ apart. Let
    \begin{align}
        \Lambda_k = \left\{-B_{k} + m h_{k}: m = 0, 1, \dots \left\lceil{\frac{2B_{k}}{h_{k}}}\right\rceil\right\}.
    \end{align}
    Within the $k$th grid, there are $N_{k} = \ceil{\frac{2B_{k}}{h_{k}}} + 1\leq \frac{2B_{k}}{h_{k}} + 2$ points spread at $h_{k}$. Finally, let $\Lambda:=\bigcup_{k=0}^K\Lambda_k$. For this grid, if $V\in I_k$, there is a point $\lambda_k\in\Lambda_k$ such that $|\lambda_k-\lambda^\ast|\leq h_k/2$. Now, to see what approximating $\lambda^\ast$ by the fixed grid point $\lambda_k$ costs us, define the quadratic exponent and its exponential by
    \begin{align*}
        g(\lambda) = \lambda S-\lambda^2\sigma_i^2 V,\qquad
        M(\lambda) = \exp\{g(\lambda)\}.
    \end{align*}
    Computing the difference between $\lambda^{\ast}$ and a fixed $\lambda_{k}$ we get,
    \begin{align*}
        g(\lambda^\ast)-g(\lambda_k)
        &=
        \sigma_i^2 V(\lambda_k-\lambda^\ast)^2
        \leq
        \sigma_i^2(2v_k)\frac{h_k^2}{4}
        =
        \frac{1}{32}.
    \end{align*}
    Approximating $\lambda^\ast$ by the fixed grid point $\lambda_k$ loses at most a constant in the exponent. Consequently, on $\mathcal E\cap\{S^2/V>t\}$,
    \begin{align*}
        g(\lambda_k) &\geq g(\lambda^\ast)-\frac{1}{32} = \frac{S^2}{4\sigma_i^2V}-\frac{1}{32} > \frac{t}{4\sigma_i^2}-\frac{1}{32}.
    \end{align*}
    Equivalently, for some grid point $\lambda\in\Lambda$,
    \begin{align*}
        M(\lambda) &\geq \exp\left\{\frac{t}{4\sigma_i^2}-\frac{1}{32}\right\}.  \end{align*}
    By construction, $ |\Lambda_k| = N_k \leq \frac{2B_k}{h_k}+2 = \frac{4R}{\sigma_i}+2$.  Therefore
    \begin{align*}
        |\Lambda| &\leq (K+1) \left( \frac{4R}{\sigma_i}+2 \right).
    \end{align*}
    Let $A_{\Lambda} := (K+1) \left( \frac{4R}{\sigma_i}+2 \right),$ and $t_{\Lambda} := 4\sigma_i^2 \left( \log A_{\Lambda} + \log\frac{4}{3\delta} + \frac{1}{32} \right)$.
    Applying Markov's inequality to the fixed grid yields, for any $t>0$,
    \begin{align*}
        \Pbb\left( \mathcal E\cap \left\{ \frac{S^2}{V}>t \right\} \middle| \mathscr B \right)
        &\leq \Pbb\left( \bigcup_{\lambda\in\Lambda} \left\{ M(\lambda) \geq \exp\left\{ \frac{t}{4\sigma_i^2} - \frac{1}{32} \right\} \right\} \middle| \mathscr B \right)
        \\
        &\leq \sum_{\lambda\in\Lambda} \exp\left\{ -\frac{t}{4\sigma_i^2} + \frac{1}{32} \right\} \Ebb[M(\lambda)\mid\mathscr B] \\
        &= |\Lambda| \exp\left\{ -\frac{t}{4\sigma_i^2} + \frac{1}{32} \right\}.
    \end{align*}
    Taking $t=t_{\Lambda}$ gives
    \begin{align*}
        |\Lambda| \exp\left\{ -\frac{t_{\Lambda}}{4\sigma_i^2} + \frac{1}{32} \right\} &\leq
        A_{\Lambda} \exp\left\{ -\log A_{\Lambda} - \log\frac{4}{3\delta} \right\} = \frac{3\delta}{4}.
    \end{align*}
    Thus
    \begin{align*}
        \Pbb\left(
            \mathcal E\cap
            \left\{
                \frac{S^2}{V}>
                t_{\Lambda}
            \right\}
             \middle| \mathscr B
        \right)
        &\leq
        \frac{3\delta}{4}.
    \end{align*}
    Combining this with $\Pbb(\mathcal E^c\mid\mathscr B)\leq\delta/4$ yields
    \begin{align*}
        \Pbb\left(
            \frac{S^2}{V}
            >
            t_{\Lambda}
             \middle| \mathscr B
        \right)
        &\leq
        \delta .
    \end{align*}

    To show that $t_{\Lambda} = \mathcal{O}(\sigma _{i}^{2}\log(en_{0}/\delta ))$, let $L=\log(e n_0/\delta)$ and $a=\log(12/\delta)$. We first collect the deterministic bounds that enter the calculation. First $a = \log 12+\log\frac{1}{\delta} \leq 3L$. Using this we can see,
    \begin{align}
        \frac{R^2}{\sigma_i^2}
        &=
        2\left(n_0+2\sqrt{n_0a}+2a\right)
        \leq
        4n_0+6a
        \leq
        22n_0L,
    \end{align}
    Here we use $2\sqrt{n_0a}\leq n_0+a$, $a\leq3L$, $n_0\leq n_0L$, and $L\leq n_0L$. This implies $\frac{R}{\sigma_i} \leq \sqrt{22n_0L}$. Now for count of the elements in the grid $\Lambda$, we have
    \begin{align}
        U
        &=
        \left(1+\sqrt{2a}\right)^2
        \leq
        6a
        \leq
        18L,
    \end{align}
    Here we use $a\geq1$ and $2\sqrt{2a}\leq 3a$. Putting in the value of chosen $v_{0}$ we get
    \begin{align}
        \log\frac{U}{v_0}
        &\leq
        \log 2000+\log n_0+\log L+2\log\frac{1}{\delta}
        \leq
        11L.
    \end{align}
    Therefore, $K+1 \leq \log_2\frac{U}{v_0}+2 \leq 18L$. With these preliminary bounds, the grid-size estimate becomes
    \begin{align*}
        \log A_\Lambda &= \log(K+1) + \log\left(\frac{4R}{\sigma_i}+2\right) \\
        &\leq \log(18L) + \log\left(\frac{4R}{\sigma_i}+2\right) \\
        &\leq 4L+5L \\
        &= 9L.
    \end{align*}
    The same constants give
    \begin{align*}
        t_\Lambda &= 4\sigma_i^2 \left( \log A_\Lambda + \log\frac{4}{3\delta} + \frac{1}{32} \right) \\
        &\leq 4\sigma_i^2 \left( 9L + \log\frac{4}{3\delta} + \frac{1}{32} \right) \\
        &\leq 4\sigma_i^2(10L) \\
        &= 40\sigma_i^2L.
    \end{align*}
    Averaging over $\mathscr B$ completes the proof.
\end{proof}

\paragraph{Proof of Lemma~\ref{lem:algo2-causal-design}}\label{app:causal-design-proof}
\begin{proof}
    We use the standing condition for the tested triples in this section: $j\in A$ and $S_i\setminus\{j\}\subseteq A$. Hence all neighbors of $i$ belong to $A$. The initial state is fixed. Let $D_r\in\Rbb^{p\times p}$ be the diagonal matrix with the single nonzero entry $(D_r)_{rr}={\Theta_{rr}^{-1}}$, and set $B_n := I-D_{I^{(n)}}{\Theta}$.
    The Glauber update can be written as
    \begin{align*}
        X^{(n)} = B_n X^{(n-1)} + e_{I^{(n)}}\varepsilon_{I^{(n)}}^{(n)} .
    \end{align*}
    For $n_1\leq n_2$, define $\Pi_{n_1}^{n_2}:=B_{n_2}B_{n_2-1}\cdots B_{n_1}$, and set $\Pi_{n+1}^{n}:=I$.
    Since the full update schedule is contained in $\mathscr B$, every matrix $\Pi_{n_1}^{n_2}$ is $\mathscr B$-measurable. Unrolling the recursion gives
    \begin{align*}
        X^{(n)} = \Pi_1^n X^{(0)} + \sum_{s=1}^{n} \Pi_{s+1}^n e_{I^{(s)}}\varepsilon_{I^{(s)}}^{(s)}.
    \end{align*}

    Fix a block $b$. Applying the last equation at times $t_3^{(b)}$ and $t_2^{(b)}$ gives an affine expansion
    \begin{align*}
        x_b = e_j^\top \left( X^{(t_3^{(b)})} - X^{(t_2^{(b)})} \right) = d_b + \sum_{s\leq t_3^{(b)}} c_{b,s} \varepsilon_{I^{(s)}}^{(s)},
    \end{align*}
    where
    \begin{align*}
        d_b:=e_j^\top(\Pi_1^{t_3^{(b)}}-\Pi_1^{t_2^{(b)}})X^{(0)} \quad \text{ and } \quad c_{b,s}:=e_j^\top(\Pi_{s+1}^{t_3^{(b)}}-\mathds{1}\{s\leq t_2^{(b)}\}\Pi_{s+1}^{t_2^{(b)}})e_{I^{(s)}}.
    \end{align*}
    The coefficients $d_b$ and $c_{b,s}$ are $\mathscr B$-measurable. We now separate the innovations at the measurement times $\mathcal M$ from the rest. For each measurement pair recall the definitions $\eta_{a}, \xi_{a}$ as sums and differences of the measurement pair and so,
    \begin{align*}
        \varepsilon_i^{(t_1^{(a)})} = \frac{\eta_a-\xi_a}{2}, \quad \text{ and } \quad \varepsilon_i^{(t_4^{(a)})} = \frac{\eta_a+\xi_a}{2}.
    \end{align*}
    Only measurement pairs with index $a<b$ can contribute to $x_b$. If $a>b$, then both measurement times occur after $t_3^{(b)}$. If $a=b$, the same-block term is absent because Lemma~\ref{lem:iiji-regression} gives $x_b\perp \xi_b$. Therefore, with
    \begin{align*}
        u_b &:= d_b + \sum_{\substack{s\leq t_3^{(b)}\\ 
        s\notin\mathcal M}} c_{b,s}\varepsilon_{I^{(s)}}^{(s)} + \frac{1}{2} \sum_{a<b} \left( c_{b,t_1^{(a)}} + c_{b,t_4^{(a)}} \right) \eta_a \quad \text{ and } \quad L_{ba} := \frac{1}{2} \left( c_{b,t_4^{(a)}} - c_{b,t_1^{(a)}} \right),
        \qquad a<b,
    \end{align*}
    we can write $x_b = u_b+\sum_{a<b}L_{ba}\xi_a$. By definition of $\mathscr B$, both $u_b$ and $L_{ba}$ are $\mathscr B$-measurable.
    This proves the affine representation and immediately implies that $x_b$ is $\mathcal F_{b-1}$-measurable.

    It remains to identify the conditional law of $\xi_b$. The two innovations $\varepsilon_i^{(t_1^{(b)})}$ and $\varepsilon_i^{(t_4^{(b)})}$ are independent $\mathcal N(0,\sigma_i^2)$ variables, independent of the update schedule and of all other innovations. Their sum $\eta_b$ and difference $\xi_b$ are independent centered Gaussians, and
    \begin{align*}
        \xi_b
        \sim
        \mathcal N(0,2\sigma_i^2).
    \end{align*}
    Since $\mathcal F_{b-1}$ is generated by $\mathscr B$ and $\xi_1,\ldots,\xi_{b-1}$, it does not reveal any information about $\xi_b$ beyond $\eta_b$, which is independent of $\xi_b$. Therefore
    \begin{align*}
        \xi_b\mid\mathcal F_{b-1}
        \sim
        \mathcal N(0,2\sigma_i^2).
    \end{align*}
\end{proof}

\paragraph{Proof of Lemma~\ref{lem:algo2-noise-energy}}\label{app:algo2-noise-energy-proof}

\begin{lemma}\label{lem:algo2-noise-energy}
Fix a triple $(i,A,j)$ and consider the first $n_0$ completed
$A$-$iiji$ blocks. Let
$\boldsymbol{\xi}=(\xi_1,\ldots,\xi_{n_0})^\top$ be the corresponding
noise-contrast vector from~\eqref{eq:algo2-vecform}. For every
$\delta\in(0,1)$, with probability at least $1-\delta$,
\begin{equation}
    2\sigma_i^2 \left( n_0 - 2\sqrt{n_0\log\frac{2}{\delta}} \right) \leq \norm{\boldsymbol{\xi}}^2 \leq 2\sigma_i^2 \left( n_0 + 2\sqrt{n_0\log\frac{2}{\delta}} + 2\log\frac{2}{\delta}
    \right).
\end{equation}
\end{lemma}

\begin{proof}
For each block, $\xi_b$ is the difference of the two fresh
$i$-innovations at the first and fourth update times of that block:
\begin{align}
    \xi_b = \varepsilon_i^{(t_4^{(b)})} - \varepsilon_i^{(t_1^{(b)})}.
\end{align}
These innovation pairs are disjoint across blocks, and each innovation is distributed as $\mathcal N(0,\sigma_i^2)$. Hence $\xi_1,\ldots,\xi_{n_0}$ are iid $\mathcal N(0,2\sigma_i^2)$, and
\begin{align}
        Z
        :=
        \frac{\norm{\boldsymbol{\xi}}^2}{2\sigma_i^2}
        \sim
        \chi^2_{n_0}.
\end{align}
The Laurent--Massart chi-square inequalities state that, for $Z\sim\chi^2_m$ and every $x>0$,
\begin{align*}
    \Pbb\left(
        Z\geq m+2\sqrt{mx}+2x
    \right)
    &\leq e^{-x},\\
    \Pbb\left(
        Z\leq m-2\sqrt{mx}
    \right)
    &\leq e^{-x}.
\end{align*}
Applying these bounds with $m=n_0$ and $x=\log(2/\delta)$ and taking a union bound gives the claim.
\end{proof}

We finish the separation ingredients by treating the two possible values of the edge coefficient. Under the null, the self-normalized cross term is small relative to the noise energy. Under the alternative, the fresh $j$-innovation supplies enough covariate energy for the signal term to dominate the same fluctuations.
\paragraph{Proof of Lemma~\ref{lem:algo2-null}}\label{app:null-control-lem-proof}

\begin{proof}
    Under the null hypothesis $\beta_{ij}=0$, equation~\eqref{eq:algo2-null-stat} gives
    \begin{align*}
        \Phi_{i,A}(j) = \frac{ \frac{(\mathbf{X}^\top \boldsymbol{\xi})^2}{\norm{\mathbf{X}}^2}}{ \norm{\boldsymbol{\xi}}^2 - \frac{(\mathbf{X}^\top \boldsymbol{\xi})^2}{\norm{\mathbf{X}}^2} }.
    \end{align*}
    Applying Lemma~\ref{lem:algo2-snm-bound} with failure probability $\delta/2$ gives, with probability at least $1-\delta/2$,
    \begin{align*}
        \frac{(\mathbf{X}^\top \boldsymbol{\xi})^2}{\norm{\mathbf{X}}^2} \leq 40\sigma_i^2 \log\frac{2e n_0}{\delta}.
    \end{align*}
    Applying the lower-tail part of Lemma~\ref{lem:algo2-noise-energy} with failure probability $\delta/2$ gives, with probability at least $1-\delta/2$,
    \begin{align*}
        \norm{\boldsymbol{\xi}}^2 \geq 2\sigma_i^2 \left( n_0 - 2\sqrt{n_0\log\frac{4}{\delta}} \right)
    \end{align*}
    By a union bound, both events hold with probability at least $1-\delta$. Assume $n_0>80 \log\frac{2e n_0}{\delta}$. Since $\log(4/\delta)\leq \log\frac{2e n_0}{\delta}$ for $n_0\geq 1$, we have
    \begin{align*}
        \sqrt{n_0\log\frac{4}{\delta}} \leq \sqrt{n_0\log\frac{2e n_0}{\delta}} \leq \frac{n_0}{8}.
    \end{align*}
    Therefore, on the same event,
    \begin{align*}
        \norm{\boldsymbol{\xi}}^2 - \frac{(\mathbf{X}^\top \boldsymbol{\xi})^2}{\norm{\mathbf{X}}^2}
        \geq
        \sigma_i^2 \left\{ 2\left(n_0-\frac{n_0}{4}\right) - \frac{n_0}{2} \right\}
        =
        \sigma_i^2n_0.
    \end{align*}
    Hence
    \begin{align*}
        \Phi_{i,A}(j)
        \leq
        \frac{
            40\sigma_i^2\log(2e n_0/\delta)
        }{
            \sigma_i^2n_0
        }
        =
        40\frac{\log(2e n_0/\delta)}{n_0}.
    \end{align*}
\end{proof}

\paragraph{Proof of Lemma~\ref{lem:algo2-excitation}}\label{app:algo2-excitation-proof}

\begin{proof}
    Set $\zeta_b := \sqrt{\Theta_{jj}} \varepsilon_j^{(t_3^{(b)})}$.
    We use the natural pre-update filtration. Since $t_3^{(b)}$ is determined by the update schedule, for $b=1,\ldots,n_0$ define
    \begin{align*}
        \mathcal H_{b-1} := \sigma\left( {I^{(1:\infty)}},\, \{\varepsilon_{I^{(s)}}^{(s)}:s<t_3^{(b)}\} \right).
    \end{align*}
    These sigma-fields are increasing in $b$ and freeze the whole trajectory up to time $t_3^{(b)}-1$. Since the schedule is independent of the update innovations, the stopped innovation $\zeta_b$ is independent of $\mathcal H_{b-1}$ and has distribution $\mathcal N(0,1)$.  We next express $x_b$ in terms of the fresh $j$-innovation at time $t_3^{(b)}$. Since $t_3^{(b)}$ is the first update of coordinate $j$ after $t_2^{(b)}$ {along the $\bar A$-clock}, coordinate $j$ is not updated between $t_2^{(b)}$ and $t_3^{(b)}$. Hence
    \begin{align*}
        X_j^{(t_3^{(b)}-1)} = X_j^{(t_2^{(b)})}.
    \end{align*}
    The Glauber update at time $t_3^{(b)}$ gives
    \begin{align*}
        X_j^{(t_3^{(b)})} = \sum_{k\in S_j}\beta_{jk}X_k^{(t_3^{(b)}-1)} + \varepsilon_j^{(t_3^{(b)})} .
    \end{align*}
    Therefore,
    \begin{align*}
        x_b &= X_j^{(t_3^{(b)})} - X_j^{(t_2^{(b)})} \\
        &= \sum_{k\in S_j}\beta_{jk}X_k^{(t_3^{(b)}-1)} - X_j^{(t_3^{(b)}-1)} + \varepsilon_j^{(t_3^{(b)})} .
    \end{align*}
    Thus, after multiplying by $\sqrt{\Theta_{jj}}$, we may write
    \begin{align*}
        \sqrt{\Theta_{jj}} x_b = a_b+\zeta_b,
    \end{align*}
    where $a_b$ is $\mathcal H_{b-1}$-measurable. Note that $a_b$ may depend on earlier fresh $j$-innovations, but not on $\zeta_b$ itself. Since $(a+Z)^2$ has a noncentral chi-square distribution with one degree of freedom and noncentrality parameter $a^2$, the noncentral chi-square moment-generating function gives, for any deterministic $a\in\mathbb R$, any $\lambda>0$, and $Z\sim\mathcal N(0,1)$,
    \begin{align*}
        \Ebb\left[ \exp\{-\lambda(a+Z)^2\} \right] =
        \frac{1}{\sqrt{1+2\lambda}} \exp\left\{ -\frac{\lambda a^2}{1+2\lambda} \right\} \leq (1+2\lambda)^{-1/2}.
    \end{align*}
    Applying this conditionally with $\lambda=1$ gives
    \begin{align*}
        \Ebb\left[ \exp\{-\Theta_{jj}x_b^2\}  \middle|  \mathcal H_{b-1} \right] \leq 3^{-1/2}.
    \end{align*}
    Since $x_1,\ldots,x_{b-1}$ are $\mathcal H_{b-1}$-measurable, iterating conditional expectations over this increasing sequence yields
    \begin{align*}
        \Ebb\left[ \exp\{-\Theta_{jj}\norm{\mathbf{X}}^2\} \right] \leq 3^{-n_0/2}.
    \end{align*}
    By Markov's inequality applied to the decreasing exponential,
    \begin{align*}
        \Pbb\left( \Theta_{jj}\norm{\mathbf{X}}^2 < \frac{n_0}{4} \right)
        &\leq e^{n_0/4} \Ebb\left[ \exp\{-\Theta_{jj}\norm{\mathbf{X}}^2\} \right] \\
        &\leq \exp\left\{ -\left( \frac{1}{2}\log 3-\frac{1}{4} \right)n_0 \right\} \leq e^{-n_0/4}.
    \end{align*}
    If $n_0\geq 4\log(1/\delta)$, the last bound is at most $\delta$.
\end{proof}

\paragraph{Proof of Lemma~\ref{lem:algo2-alt}}\label{app:algo2-alt-proof}

\begin{proof}
Under the assumption $S_i\setminus\{j\}\subseteq A$, the regression identity applies. Recall from~\eqref{eq:algo2-signal-drop} that
\begin{align*}
    \mathrm{RSS}_0-\mathrm{RSS}_1
    \geq
    \frac{1}{2}\beta_{ij}^2\|x\|^2
    -
    \frac{(\mathbf{X}^\top \boldsymbol{\xi})^2}{\|x\|^2}.
\end{align*}

The sample-size condition is stronger than the excitation requirement $n_0\geq 4\log(3/\delta)$, since $\kappa_{ij}^2\leq 1$. Applying Lemmas~\ref{lem:algo2-excitation}, \ref{lem:algo2-snm-bound}, and \ref{lem:algo2-noise-energy} with failure probability $\delta/3$ each gives the following three events:
\begin{enumerate}[label=(\roman*)]
    \item By Lemma~\ref{lem:algo2-excitation}, with probability at least $1-\delta/3$,
    \begin{align*}
        \|x\|^2
        \geq
        \frac{1}{4}\sigma_j^2 n_0.
    \end{align*}
    \item By Lemma~\ref{lem:algo2-snm-bound}, with probability at least $1-\delta/3$,
    \begin{align*}
        \frac{(\mathbf{X}^\top \boldsymbol{\xi})^2}{\|x\|^2}
        \leq
        40\sigma_i^2\log\frac{3e n_0}{\delta}.
    \end{align*}
    \item By the upper-tail part of Lemma~\ref{lem:algo2-noise-energy}, with probability at least $1-\delta/3$,
    \begin{align*}
        \|\xi\|^2
        \leq
        2\sigma_i^2
        \left(
            n_0
            +
            2\sqrt{n_0\log\frac{6}{\delta}}
            +
            2\log\frac{6}{\delta}
        \right).
    \end{align*}
\end{enumerate}
By a union bound, all three events hold simultaneously with probability at least $1-\delta$.

Set $L:=\log(3e n_0/\delta)$ and $a:=\log(6/\delta)$. Then $a\leq L$ and the sample-size condition gives $L\leq \kappa_{ij}^2 n_0/640\leq n_0/640$. Also,
\begin{align*}
    \beta_{ij}^2\sigma_j^2
    =
    \frac{\Theta_{ij}^2}{\Theta_{ii}^2\Theta_{jj}}
    =
    \kappa_{ij}^2\sigma_i^2.
\end{align*}
Therefore,
\begin{align*}
    \mathrm{RSS}_0-\mathrm{RSS}_1
    &\geq
    \frac{1}{8}\kappa_{ij}^2\sigma_i^2 n_0
    -
    40\sigma_i^2 L
    \\
    &\geq
    \frac{1}{8}\kappa_{ij}^2\sigma_i^2 n_0
    -
    \frac{1}{16}\kappa_{ij}^2\sigma_i^2 n_0
    \\
    &=
    \frac{1}{16}\kappa_{ij}^2\sigma_i^2 n_0 .
\end{align*}

For the denominator, $\mathrm{RSS}_1\leq\|\xi\|^2$. Since $a\leq n_0/640$, we have $2\sqrt{n_0a}\leq n_0/2$ and $2a\leq n_0/8$. Hence
\begin{align*}
    \mathrm{RSS}_1
    &\leq
    2\sigma_i^2
    \left(
        n_0
        +
        2\sqrt{n_0a}
        +
        2a
    \right)
    \\
    &\leq
    2\sigma_i^2
    \left(
        n_0+\frac{n_0}{2}+\frac{n_0}{8}
    \right)
    \\
    &=
    \frac{13}{4}\sigma_i^2 n_0
    \leq
    4\sigma_i^2 n_0.
\end{align*}
If $\mathrm{RSS}_1=0$, then $\Phi_{i,A}(j)=+\infty$ and the claim is immediate. Otherwise,
\begin{align*}
    \Phi_{i,A}(j)
    =
    \frac{\mathrm{RSS}_0-\mathrm{RSS}_1}{\mathrm{RSS}_1}
    \geq
    \frac{\kappa_{ij}^2\sigma_i^2 n_0/16}{4\sigma_i^2 n_0}
    =
    \frac{\kappa_{ij}^2}{64}.
\end{align*}
This completes the proof.
\end{proof}

\subsection{Finite-sample recovery}
The preceding groups establish the two events needed for the global guarantee: every tested triple has a separated statistic, and every statistic is available by the observation horizon. We now combine those events with the deterministic recovery proposition.
\paragraph{Proof of Theorem~\ref{thm:algo2-finite-sample}}
\label{app:algo2-finite-sample-proof}
\begin{proof}
Set $n_0:=n_\star$ and $\delta_0:=\delta/(2|\mathfrak T|)$. We first remove the fixed-point dependence in the separation requirement. Put
\begin{align}
    a:=\frac{5120}{\kappa^2}, \qquad b:=\frac{6e|\mathfrak T|}{\delta}, \qquad L:=\log(2ab).
\end{align}
Then~\eqref{eq:algo2-final-nstar} gives $n_0\geq 2aL$. Since $L\geq 1$, the function $x\mapsto x-a\log(bx)$ is increasing on $[2aL,\infty)$, and
\begin{equation*}
    2aL-a\log(2abL) = a(L-\log L) \geq 0.
\end{equation*}
Therefore
\begin{equation}
    n_0
    \geq
    \frac{5120}{\kappa^2}
    \log\frac{6e|\mathfrak T| n_0}{\delta}.
    \label{eq:algo2-final-n0-implicit}
\end{equation}

For every $(i,A,j)\in\mathfrak T$ satisfying $S_i\setminus\{j\}\subseteq A$, condition~\eqref{eq:algo2-final-n0-implicit} is the hypothesis of Proposition~\ref{prop:algo2-fixed-triple-separation} with failure probability $\delta_0$, since $\log(3e n_0/\delta_0)=\log(6e|\mathfrak T|n_0/\delta)$. A union bound therefore shows that, with probability at least $1-\delta/2$, all relevant triples satisfy the separation inequalities.

For every $(i,A,j)\in\mathfrak T$, $|A|=2d$ and hence $q=|A|+1$. Since $\kappa\leq1$, the definition~\eqref{eq:algo2-final-nstar} gives $n_\star\geq\log(4|\mathfrak T|/\delta)$, so condition~\eqref{eq:algo2-final-horizon} implies $N\geq128pq^{3}\bigl(n_\star+\log(4|\mathfrak T|/\delta)\bigr)$, which is the hypothesis of Lemma~\ref{lem:algo2-fixed-triple-clock} with failure probability $\delta_0$, since $\log(2/\delta_0)=\log(4|\mathfrak T|/\delta)$. A second union bound shows that, with probability at least $1-\delta/2$, every tested triple has accumulated $n_0$ blocks by time $N$.

On the intersection of these two events, the statistics are available on $\mathfrak T$ and satisfy the separation property of Definition~\ref{def:generic-separation}. Proposition~\ref{prop:generic-separation-recovery} then gives $\widehat \Ecal=\Ecal$, and a union bound over the two failures yields $\Pbb(\widehat \Ecal\neq\Ecal)\leq\delta$.
\end{proof}

\section{Large Degree-Constrained Spectral Radius and Slow Mixing}
\label{app:large-rho-slow-mixing}

In this appendix we prove
Proposition~\ref{prop:large-rho-slow-relaxation}, stated in
Section~\ref{sec:algoone}. Its proof constructs, for each
precision matrix $\Theta$ and each sufficiently large potential level,
an initialization at which the Gaussian Gibbs sampler mixes slowly.
The result is not used in either recovery proof. Its role is to make a
matched comparison with mixing-based reductions:
Theorem~\ref{thm:algoone-main} pays only logarithmically in
$\rho_{2d}(\Theta)$ at the constructed initialization, whereas a
mixing horizon that is uniform over the corresponding potential ball
must account for that initialization and hence pays at least linearly
in $\rho_{2d}(\Theta)$. We proceed in three stages. First, we fix the
relevant notions of mixing on the unbounded state space. Next, we
record an elementary two-moment lower bound on total variation distance
(Lemma~\ref{lem:moment-tv-obstruction}). Finally, we verify the lemma's
hypotheses for the random-scan chain and prove the proposition in
quantitative form.

Let $\pi_\Theta=\mathcal N(0,\Theta^{-1})$, and let $P_\Theta$ denote the one-step transition kernel of the discrete-time random-scan Gaussian Glauber dynamics of Section~\ref{sec:glauber}. A sweep is $p$ consecutive global updates, with kernel $P_\Theta^p$. We work with the normalized precision matrix
\begin{equation}
    D:=\diag(\Theta_{11},\ldots,\Theta_{pp}),
    \qquad
    \widetilde\Theta:=D^{-1/2}\Theta D^{-1/2}.
    \label{eq:normalized-precision}
\end{equation}
Thus $\widetilde\Theta$ is the precision matrix of the normalized coordinates $Y=D^{1/2}X$, and its diagonal entries are all equal to one. The degree-constrained spectral radius $\rho_{2d}(\Theta)$ of Definition~\ref{def:rho2d} will enter through these coordinates: each matrix in the maximum defining it is a principal block of $\widetilde\Theta^{-1}=D^{1/2}\Sigma D^{1/2}$.

\paragraph{Mixing time on a continuous state space.}
For probability measures $\mu$ and $\nu$ on $\Rbb^p$, we define the total variation distance by
\begin{equation}
    \norm{\mu-\nu}_{\mathrm{TV}}
    :=
    \sup_{B\in\mathcal B(\Rbb^p)}
    \abs{\mu(B)-\nu(B)}.
    \label{eq:tv-distance-continuous}
\end{equation}
Because the Gaussian state space is unbounded, the unrestricted worst-case quantity
\begin{equation*}
    \sup_{x\in\Rbb^p}
    \norm{P_\Theta^t(x,\cdot)-\pi_\Theta}_{\mathrm{TV}}
\end{equation*}
is not the right object for comparison with our finite-horizon theorems, which start the chain from a deterministic initialization of bounded potential (recall the potential $\mathfrak{L}(x)=\norm{x}_\Theta^2$ from Section~\ref{sec:design-control}). We therefore restrict the supremum to the set of such initializations,
\begin{equation}
    K_{E_0}:=\{x\in\Rbb^p:\norm{x}_\Theta^2\leq E_0\}
    \label{eq:energy-ball}
\end{equation}
and define the corresponding restricted distance and mixing time by
\begin{align}
    d_{E_0}(t) &:= \sup_{x\in K_{E_0}} \norm{P_\Theta^t(x,\cdot)-\pi_\Theta}_{\mathrm{TV}}, \nonumber\\
    t_{\mathrm{mix}}^{(E_0)}(\varepsilon) &:= \inf\{t\geq0:d_{E_0}(t)\leq\varepsilon\}.
    \label{eq:energy-ball-mixing-time}
\end{align}
This is the natural total-variation notion for statements that start the chain from an initialization of potential at most $E_0$ and measure time in the same global-update units as the observed trajectory. The corresponding sweep count is $t_{\mathrm{mix}}^{(E_0)}(\varepsilon)/p$, up to integer rounding.

For a reversible Markov kernel $K$ with stationary law $\pi$, we define the spectral gap and the relaxation time by
\begin{equation}
    \gamma_{\mathrm{rel}}(K) := 1-\sup\left\{ \frac{\norm{K f}_{L^2(\pi)}} {\norm{f}_{L^2(\pi)}}: f\in L^2(\pi),\ \Ebb_{\pi}f=0,\ f\neq0 \right\}, \qquad \tau_{\mathrm{rel}}(K):=\gamma_{\mathrm{rel}}(K)^{-1}.
    \label{eq:relaxation-time-definition}
\end{equation}
The spectral behavior of the Gaussian Gibbs sampler is classical: for a Gaussian target, the relaxation of the Gibbs sampler is attained on linear functions~\citep{AMIT199182,amit1996convergence}, and on linear functions the random-scan kernel acts through the matrix $I-p^{-1}D^{-1}\Theta$. In the notation above, this gives the exact identity
\begin{equation}
    {\gamma_{\mathrm{rel}}(P_\Theta) = \frac{\lambda_{\min}(\widetilde\Theta)}{p}, \qquad \tau_{\mathrm{rel}}(P_\Theta) = \frac{p}{\lambda_{\min}(\widetilde\Theta)}.}
    \label{eq:amit-gaussian-gibbs-gap}
\end{equation}
Since $P_\Theta$ is reversible and positive on $L^2(\pi_\Theta)$ {(it is the average over $i\in[p]$ of the conditional-expectation operators $f\mapsto\Ebb[f\mid X_{-i}]$, each an orthogonal projection in $L^2(\pi_\Theta)$)}, the sweep kernel obeys $1-\gamma_{\mathrm{rel}}(P_\Theta^p)=(1-\gamma_{\mathrm{rel}}(P_\Theta))^p$; in sweep units, therefore,
\begin{equation}
    {\gamma_{\mathrm{rel}}(P_\Theta^p)
    =
    1-\Bigl(1-\frac{\lambda_{\min}(\widetilde\Theta)}{p}\Bigr)^{\!p}
    \asymp
    \lambda_{\min}(\widetilde\Theta),
    \qquad
    \tau_{\mathrm{rel}}(P_\Theta^p)
    \asymp
    \lambda_{\min}(\widetilde\Theta)^{-1},}
    \label{eq:sweep-relaxation-from-update}
\end{equation}
{where the implied constants may be taken to be $1-e^{-1}$ and $1$, using only $\lambda_{\min}(\widetilde\Theta)\leq1$, which holds because $\widetilde\Theta$ has unit diagonal.}
It remains to convert a slow spectral mode into a lower bound on the total variation distance. The next lemma records the elementary two-moment bound we use; Proposition~\ref{prop:large-rho-slow-relaxation} then verifies its hypotheses for the random-scan chain by conditioning on the scan schedule.

\begin{lemma}[Two-moment total-variation lower bound]
\label{lem:moment-tv-obstruction}
Let $\mu$ and $\pi$ be probability measures on a common measurable space, and let $f$ be square-integrable under both laws with $\Ebb_\pi f=0$. If, for some $a>0$,
\begin{equation}
    \Ebb_\mu f\geq 2a,
    \qquad
    \Var_\mu(f)\leq V_\mu,
    \qquad
    \Var_\pi(f)\leq V_\pi,
    \label{eq:moment-tv-conditions}
\end{equation}
then, for $A:=\{x:f(x)\geq a\}$,
\begin{equation}
    \mu(A)\geq 1-\frac{V_\mu}{a^2},
    \qquad
    \pi(A)\leq \frac{V_\pi}{a^2},
    \qquad
    \norm{\mu-\pi}_{\mathrm{TV}}
    \geq
    1-\frac{V_\mu+V_\pi}{a^2}.
    \label{eq:moment-tv-lower}
\end{equation}
In particular, $\norm{\mu-\pi}_{\mathrm{TV}}\geq 3/4$ whenever $V_\mu+V_\pi\leq a^2/4$.
\end{lemma}

\begin{proof}
We begin by observing that, on the event $A^c$, we have $f-\Ebb_\mu f\leq-a$, since $\Ebb_\mu f\geq2a$. Chebyshev's inequality therefore gives
\begin{equation*}
    \mu(A^c)
    \leq
    \frac{V_\mu}{a^2}.
\end{equation*}
Next, since $\Ebb_\pi f=0$, Chebyshev's inequality under $\pi$ gives
\begin{equation*}
    \pi(A)
    \leq
    \frac{V_\pi}{a^2}.
\end{equation*}
Combining the two bounds, we therefore have
\begin{equation*}
    \norm{\mu-\pi}_{\mathrm{TV}}
    \geq
    \mu(A)-\pi(A)
    \geq
    1-\frac{V_\mu+V_\pi}{a^2},
\end{equation*}
which is~\eqref{eq:moment-tv-lower}. The final assertion follows by substituting $V_\mu+V_\pi\leq a^2/4$.
\end{proof}

\begin{proof}[Proof of Proposition~\ref{prop:large-rho-slow-relaxation}]
We prove the proposition in two stages. First, we show that
$\rho_{2d}(\Theta)$ lower bounds the relaxation time. Next, we use the
resulting slow spectral mode to construct an initialization
$x_\star\in K_{E_0}$ that remains far from stationarity. In particular,
we prove that there are universal constants $c>0$ and
$E_\star<\infty$ such that, for every positive definite precision
matrix $\Theta$ and every $E_0\geq E_\star$,
\begin{equation}
    t_{\mathrm{mix}}^{(E_0)}(1/4)
    \geq
    c\,p\,\rho_{2d}(\Theta),
    \qquad
    \tau_{\mathrm{rel}}(P_\Theta)
    \geq
    p\,\rho_{2d}(\Theta).
    \label{eq:energy-ball-mixing-lower}
\end{equation}
The construction below gives
$\norm{x_\star}_\Theta^2=E_0$ and establishes the stronger
pointwise bound~\eqref{eq:slow-start-main} at that single
initialization; taking the supremum over $K_{E_0}$ then yields the
first inequality in~\eqref{eq:energy-ball-mixing-lower}.

For the first stage, we begin by observing that, for every coordinate set $A$,
\begin{equation}
    D_A^{1/2}\Sigma_{AA}D_A^{1/2}
    =
    (\widetilde\Theta^{-1})_{AA}
\end{equation}
where $D_A=\diag(\Theta_{jj}:j\in A)$. Since the largest eigenvalue of a principal submatrix of a positive semidefinite matrix is at most the largest eigenvalue of the full matrix, we therefore have
\begin{equation}
    \rho_{2d}(\Theta)
    =
    \max_{|A|=2d}
    \lambda_{\max}\!\left((\widetilde\Theta^{-1})_{AA}\right)
    \leq
    \lambda_{\max}(\widetilde\Theta^{-1})
    =
    \lambda_{\min}(\widetilde\Theta)^{-1}.
    \label{eq:rho-controls-normalized-gap}
\end{equation}
Combining~\eqref{eq:rho-controls-normalized-gap} with the exact identity~\eqref{eq:amit-gaussian-gibbs-gap} gives $\tau_{\mathrm{rel}}(P_\Theta)=p/\lambda_{\min}(\widetilde\Theta)\geq p\,\rho_{2d}(\Theta)$, the second inequality in~\eqref{eq:energy-ball-mixing-lower}, with constant one. This completes the first stage.

For the second stage, we must connect relaxation to total-variation mixing over the initializations in $K_{E_0}$. The plan is to exhibit a slow linear eigenfunction, start the chain at a bounded-potential point aligned with it, and apply Lemma~\ref{lem:moment-tv-obstruction} to the law of the chain at time $t$. Let $\lambda_\star:=\lambda_{\min}(\widetilde\Theta)$, let $v_\star$ be a unit eigenvector of $\widetilde\Theta$ for $\lambda_\star$, and set
\begin{equation}
    f_\star(x):=v_\star^\top D^{1/2}x,
    \qquad
    \alpha_\star:=1-\frac{\lambda_\star}{p}.
    \label{eq:slow-linear-mode}
\end{equation}
Each update replaces one uniformly chosen coordinate by its conditional mean plus centered noise, so the one-step conditional mean is linear: $\Ebb[X^{(t)}\mid X^{(t-1)}=x]=(I-p^{-1}D^{-1}\Theta)\,x$. Since $D^{1/2}(I-p^{-1}D^{-1}\Theta)=(I-p^{-1}\widetilde\Theta)D^{1/2}$ and $v_\star^\top\widetilde\Theta=\lambda_\star v_\star^\top$, applying $v_\star^\top D^{1/2}$ gives $P_\Theta f_\star=\alpha_\star f_\star$: this is the slow mode behind~\eqref{eq:amit-gaussian-gibbs-gap}, with $\tau_{\mathrm{rel}}(P_\Theta)=(1-\alpha_\star)^{-1}=p/\lambda_\star$. We start the chain at the point $x_\star$ defined by
\begin{equation}
    D^{1/2}x_\star
    =
    \sqrt{\frac{E_0}{\lambda_\star}}\,v_\star .
    \label{eq:slow-energy-ball-start}
\end{equation}
This point satisfies $\mathfrak{L}(x_\star)=\norm{x_\star}_\Theta^2=E_0$, so $x_\star\in K_{E_0}$, and iterating the eigenfunction identity gives
\begin{equation}
    \Ebb_{x_\star} f_\star(X^{(t)})
    =
    \alpha_\star^t f_\star(x_\star).
    \label{eq:slow-eigenfunction-mean}
\end{equation}
Under the stationary law, we have
\begin{equation}
    \Var_{\pi_\Theta}(f_\star)
    =
    v_\star^\top\widetilde\Theta^{-1}v_\star
    =
    \lambda_\star^{-1},
    \qquad
    \frac{f_\star(x_\star)^2}{\Var_{\pi_\Theta}(f_\star)}
    =
    E_0.
    \label{eq:slow-mode-visibility}
\end{equation}
We now bring in the scan schedule. A word on why we condition on it: the covariance-domination bound we are about to use holds for each fixed schedule, but not after averaging over the random labels, because the randomness of the schedule contributes an additional variance term through the conditional mean. We therefore argue conditionally on the schedule throughout. To this end, we write each update in affine form: if $I^{(s)}=i$, then
\begin{equation}
    X^{(s)}=M_iX^{(s-1)}+\Theta_{ii}^{-1/2}e_i\,\zeta^{(s)},
    \qquad
    M_i:=I-\Theta_{ii}^{-1}e_ie_i^\top\Theta,
    \qquad
    \zeta^{(s)}\stackrel{\mathrm{iid}}{\sim}\mathcal N(0,1),
    \label{eq:affine-gibbs-update}
\end{equation}
Fix an update schedule $\sigma=(\sigma_1,\ldots,\sigma_t)\in[p]^t$. For integers $1\leq a\leq b\leq t$, define the partial products and the accumulated noise
\begin{equation}
    \Pi_a^b(\sigma):=M_{\sigma_b}M_{\sigma_{b-1}}\cdots M_{\sigma_a},
    \qquad
    \mathbf B_a^b(\sigma):=\sum_{s=a}^{b}\Pi_{s+1}^{b}(\sigma)\,\Theta_{\sigma_s\sigma_s}^{-1/2}e_{\sigma_s}\,\zeta^{(s)},
    \label{eq:schedule-products}
\end{equation}
with the convention $\Pi_{b+1}^{b}(\sigma):=I$. Iterating~\eqref{eq:affine-gibbs-update} along the schedule gives, conditionally on $\sigma$,
\begin{equation}
    X^{(t)}=M_\sigma x_\star+\mathbf B_\sigma,
    \qquad
    M_\sigma:=\Pi_1^t(\sigma),
    \qquad
    \mathbf B_\sigma:=\mathbf B_1^t(\sigma)\sim\mathcal N\bigl(0,V_\sigma V_\sigma^\top\bigr),
    \label{eq:fixed-schedule-affine}
\end{equation}
where $V_\sigma:=\begin{bmatrix}V_\sigma(1)&\cdots&V_\sigma(t)\end{bmatrix}$ has columns $V_\sigma(s):=\Pi_{s+1}^{t}(\sigma)\,\Theta_{\sigma_s\sigma_s}^{-1/2}e_{\sigma_s}$. Since each fixed coordinate update leaves $\pi_\Theta=\mathcal N(0,\Sigma)$ invariant and $\mathbf B_\sigma$ is independent of the initialization, the stationary covariance decomposes as $\Sigma=M_\sigma\Sigma M_\sigma^\top+V_\sigma V_\sigma^\top$. This yields the covariance domination announced above:
\begin{equation}
    V_\sigma V_\sigma^\top
    =
    \Sigma-M_\sigma\Sigma M_\sigma^\top
    \preceq \Sigma.
    \label{eq:fixed-schedule-covariance-domination}
\end{equation}
Define the noiseless slow-mode value along the schedule $\sigma$ as the conditional mean of $f_\star(X^{(t)})$,
\begin{equation}
    s_\sigma:=v_\star^\top D^{1/2}M_\sigma x_\star .
\end{equation}
Averaging over the random labels and using~\eqref{eq:slow-eigenfunction-mean} gives $\Ebb_\sigma s_\sigma=\alpha_\star^t f_\star(x_\star)$. We next observe that each noiseless coordinate update decreases the potential:
\begin{equation}
    \mathfrak{L}(M_ix)
    =
    \mathfrak{L}(x)-\Theta_{ii}^{-1}\bigl((\Theta x)_i\bigr)^2
    \leq \mathfrak{L}(x) .
    \label{eq:noiseless-energy-contraction}
\end{equation}
Combining this contraction with the Cauchy--Schwarz inequality, $\abs{s_\sigma}\leq\norm{\widetilde\Theta^{-1/2}v_\star}_2\,\norm{M_\sigma x_\star}_{\Theta}=\lambda_\star^{-1/2}\sqrt{\mathfrak{L}(M_\sigma x_\star)}$, we conclude that $\abs{s_\sigma}\leq\lambda_\star^{-1/2}\sqrt{\mathfrak{L}(x_\star)}=f_\star(x_\star)$ for every schedule $\sigma$. We now take $c_0:=1/100$ and restrict to horizons $t\leq c_0p/\lambda_\star$. On this range, Bernoulli's inequality $(1-u)^t\geq1-tu$ gives
$\alpha_\star^t=(1-\lambda_\star/p)^t\geq1-t\lambda_\star/p\geq1-c_0=0.99$, and hence $\Ebb_\sigma s_\sigma\geq0.99\,f_\star(x_\star)$. Consider the set of favorable schedules
\begin{equation}
    \mathfrak S_t
    :=
    \left\{\sigma\in[p]^t:
    s_\sigma\geq \frac12 f_\star(x_\star)
    \right\}.
\end{equation}
Let $q:=\Pbb(\mathfrak S_t^{c})$. The deterministic envelope gives $s_\sigma\leq f_\star(x_\star)$ on $\mathfrak S_t$, while $s_\sigma<\frac12 f_\star(x_\star)$ on $\mathfrak S_t^{c}$ by definition. Splitting $\Ebb_\sigma s_\sigma$ over the two events,
\begin{equation}
    0.99\,f_\star(x_\star)
    \leq
    \Ebb_\sigma s_\sigma
    \leq
    (1-q)\,f_\star(x_\star)+\frac{q}{2}\,f_\star(x_\star)
    =
    \Bigl(1-\frac{q}{2}\Bigr)f_\star(x_\star),
\end{equation}
so $q\leq0.02$; that is, $\Pbb(\mathfrak S_t)\geq0.98$. For every schedule in $\mathfrak S_t$, the conditional law of $f_\star(X^{(t)})$ has mean at least $f_\star(x_\star)/2$ and variance at most $\Var_{\pi_\Theta}(f_\star)$: indeed, by~\eqref{eq:fixed-schedule-covariance-domination}, the conditional variance $v_\star^\top D^{1/2}V_\sigma V_\sigma^\top D^{1/2}v_\star$ is at most $v_\star^\top D^{1/2}\Sigma D^{1/2}v_\star=v_\star^\top\widetilde\Theta^{-1}v_\star$. Consider now the event
\begin{equation}
    A_t:=\left\{x:f_\star(x)\geq \frac14 f_\star(x_\star)\right\}.
\end{equation}
Applying Lemma~\ref{lem:moment-tv-obstruction} conditionally on each $\sigma\in\mathfrak S_t$, with $\mu$ the conditional law of $X^{(t)}$, $\pi=\pi_\Theta$, $f=f_\star$, and $a=f_\star(x_\star)/4$, and using the variance identity~\eqref{eq:slow-mode-visibility}, we obtain
\begin{equation}
    \Pbb_{x_\star}\!\left(X^{(t)}\in A_t\,\middle|\,\sigma\right)
    \geq
    1-\frac{16}{E_0},
    \qquad
    \pi_\Theta(A_t)
    \leq
    \frac{16}{E_0}.
    \label{eq:good-schedule-tv-separation}
\end{equation}
Since $\Pbb_{x_\star}(X^{(t)}\in A_t)\geq\Pbb(\mathfrak S_t)\bigl(1-\tfrac{16}{E_0}\bigr)$, taking $E_\star$ large enough, the last two displays imply that, for all $E_0\geq E_\star$ and all $t\leq c_0p/\lambda_\star$,
\begin{equation}
    \norm{P_\Theta^t(x_\star,\cdot)-\pi_\Theta}_{\mathrm{TV}}
    \geq
    \Pbb_{x_\star}(X^{(t)}\in A_t)-\pi_\Theta(A_t)
    \geq
    \frac34.
    \label{eq:slow-start-main}
\end{equation}
We therefore have $d_{E_0}(t)\geq3/4$ throughout a constant fraction of the relaxation time. Since $\lambda_\star^{-1}\geq\rho_{2d}(\Theta)$ by~\eqref{eq:rho-controls-normalized-gap}, taking $c=c_0$ proves the mixing-time bound in~\eqref{eq:energy-ball-mixing-lower} and completes the second stage.
\end{proof}
\end{document}